\documentclass[12pt]{article}

\usepackage[utf8]{inputenc} 
\usepackage[T1]{fontenc}    
\usepackage{hyperref}       
\usepackage{url}            
\usepackage{booktabs}       
\usepackage{amsfonts}       
\usepackage{nicefrac}       
\usepackage{microtype}      
\usepackage[pdftex]{graphicx}  
\usepackage{float}
\usepackage{amsthm}
\usepackage{mathtools}
\usepackage{amsfonts}
\usepackage{amssymb}
\floatstyle{ruled}
\usepackage{xcolor}
\usepackage{bbm}

\usepackage{adjustbox}
\usepackage{wrapfig}
\usepackage{caption}
\usepackage{subcaption}
\usepackage{cprotect}
\usepackage{etoc}
\usepackage{titlesec}
\usepackage{enumitem}

\usepackage{tikz}
\usetikzlibrary{arrows,automata, calc, positioning}
\usepackage{adjustbox}
\usepackage{wrapfig}

\newcommand{\indep}{\raisebox{0.05em}{\rotatebox[origin=c]{90}{$\models$}}}

\theoremstyle{definition}\newtheorem{theorem}{Theorem}
\newtheorem{assumption}{Assumption}
\newtheorem{algorithm}{Algorithm}
\newtheorem{proposition}{Proposition}
\newtheorem{definition}{Definition}
\newtheorem{remark}{Remark}
\newtheorem{corollary}{Corollary}
\newtheorem{lemma}{Lemma}

\DeclareMathOperator*{\argmin}{\arg\!\min}

\begin{document}

\def\spacingset#1{\renewcommand{\baselinestretch}%
{#1}\small\normalsize} \spacingset{1}


  \title{\bf Kernel Methods for Unobserved Confounding: Negative Controls, Proxies, and Instruments}
  \author{Rahul Singh\thanks{
  Rahul Singh is a PhD Candidate in Economics and Statistics, MIT Department of Economics, Cambridge MA 02143 (e-mail: \url{rahul.singh@mit.edu}). I am particularly grateful to Olivia Foley for sharing medical expertise that is critical for the empirical application. I thank Alberto Abadie, Xiaohong Chen, Victor Chernozhukov, Ben Deaner, Anna Mikusheva, Whitney Newey, and Vasilis Syrgkanis for helpful comments. I thank Douglas Almond, Kenneth Chay, and David Lee for building the data set, as well as Matias Cattaneo for sharing it with their permission.  I am grateful to the Jerry A Hausman Graduate Dissertation Fellowship for financial support.
  }\hspace{.2cm}\\
    MIT Economics}
    \date{Original draft: December 18, 2020. This draft: March 23, 2023.}
  \maketitle

\bigskip
\begin{abstract}
Negative control is a strategy for learning the causal relationship between treatment and outcome in the presence of unmeasured confounding. The treatment effect can nonetheless be identified if two auxiliary variables are available: a negative control treatment (which has no effect on the actual outcome), and a negative control outcome (which is not affected by the actual treatment). These auxiliary variables can also be viewed as proxies for a traditional set of control variables, and they bear resemblance to instrumental variables. I propose a family of algorithms based on kernel ridge regression for learning nonparametric treatment effects with negative controls. Examples include dose response curves, dose response curves with distribution shift, and heterogeneous treatment effects. Data may be discrete or continuous, and low, high, or infinite dimensional. I prove uniform consistency and provide finite sample rates of convergence. I estimate the dose response curve of cigarette smoking on infant birth weight adjusting for unobserved confounding due to household income, using a data set of singleton births in the state of Pennsylvania between 1989 and 1991.
\end{abstract}

\noindent%
\vfill

\newpage

\spacingset{1.5}
\section{Introduction}\label{section:intro}


\textit{Selection on observables} is the popular assumption in causal inference that the assignment of treatment $D$ is as good as random after conditioning on covariates $X$. It is a strong causal assumption which is often violated even in laboratory settings. Negative controls, widely used in laboratory science, guard against unobserved confounding. The idea is to check for spurious relationships that would only be nonzero in the presence of an unobserved confounder $U$--an approach sometimes called \textit{falsification} or \textit{specificity} testing. Consider two auxiliary variables: a negative control treatment $Z$ (which a priori has no effect on the actual outcome $Y$), and a negative control outcome $W$ (which a priori is not affected by the actual treatment $D$). \cite{miao2018confounding} and \cite{deaner2018nonparametric} carefully formalize a learning problem in which negative controls $(Z,W)$ can not only check for the presence of unobserved confounding $U$, but also recover the causal relationship of interest.

As a concrete example, consider the empirical strategy of \cite{lousdal2020negative}. The goal is to measure the effect of mammography screening $D$ on death from breast cancer $Y$. The set of covariates $X$ includes marriage status, number of children, age at first birth, years of education, annual income, and hormone drug use. The authors used negative controls to document that, even after taking into account the covariates $X$, unobserved confounding $U$ drives spurious correlations. Specifically, dental care participation $Z$ decreases the likelihood of death from breast cancer $Y$ in the dataset. Mammography screening $D$ decreases the likelihood of death from other causes $W$ in the dataset. The authors conclude that unobserved confounding contaminates treatment effect estimation in this setting.

In the present work, I propose a family of nonparametric algorithms based on kernel ridge regression that use negative controls to not only detect but also adjust for unobserved confounding. I consider treatment effects of the population, of subpopulations, and of alternative populations with alternative covariate distributions. Moreover, I allow for treatments, covariates, and negative controls that may be discrete or continuous, and low, high, or infinite dimensional.  Due to the intuitive nature of negative controls, I use such terminology throughout the paper. In recent work, \cite{tchetgen2020introduction} refer to negative controls as proxy variables in order to emphasize that they may arise in not only experimental but also observational settings. Due to the formal resemblance between negative controls and instrumental variables, the new statistical results I provide also apply to nonparametric instrumental variable regression (NPIV). Altogether, I provide conceptual, algorithmic, and statistical contributions.


{\bf Conceptual.} I unify a variety of learning problems with unobserved confounding into one general nonparametric learning problem. In semiparametric causal inference, treatment $D$ is restricted to be binary. I consider nonparametric causal inference, allowing the treatment $D$ to be binary, discrete, or continuous. It appears that this is the first work on \textcolor{black}{conditional} dose response curves and heterogeneous treatment effects using negative controls. \textcolor{black}{See Section~\ref{sec:related} for a discussion of related work on related estimands, e.g. \cite{mastouri2021proximal,kallus2021causal,ghassami2021minimax}}. I provide a template for future epidemiology research to estimate dose response curves and heterogeneous treatment effects from medical records despite unobserved confounding.

{\bf Algorithmic.} 
I propose a family of novel estimators with closed form solutions that are straightforward to implement by matrix operations. To do so, I assume that the true causal relationship is a function in a reproducing kernel Hilbert space (RKHS), which is a popular nonparametric setting in machine learning. The hyperparameters are ridge regression penalties and kernel hyperparameters. For the former, I derive the closed form solution for leave-one-out cross validation. The latter have well known heuristics. I evaluate the estimators in simulations against alternative estimators that ignore unobserved confounding.

{\bf Statistical.} I prove uniform ($\sup$ norm) consistency with finite sample rates. A uniform guarantee encodes caution about worst case scenarios when informing policy decisions. The finite sample rates of convergence do not directly depend on the data dimension but rather the smoothness of the true causal relationship. An important intermediate result is finite sample analysis of NPIV in $\sup$ norm. Of independent interest, I relate assumptions required in ill posed inverse problems--existence and completeness--to the RKHS setting. This characterization appears to be absent from previous work on NPIV in the RKHS.

To illustrate how the proposed estimators are useful, I conduct a case study. Estimating the effect of cigarette smoking on infant birth weight is challenging for several reasons. First, pregnant women are classified as a vulnerable population, so they are typically excluded from clinical trials; observational data are the only option. Second, pregnancy induces many physiological changes, so medical knowledge predicts different dose response curves for women who are pregnant compared to women who are not pregnant. Third, medical records exclude an \textit{unobserved confounder} known to be crucial for maternal-fetal health: household income. I argue that medical records include variables that satisfy the properties of \textit{negative controls} for unobserved income, and discuss what issues may arise if there are additional unobserved confounders. I provide preliminary results and outline directions for future work on this important topic.


The structure of the paper is as follows. Section~\ref{sec:related} describes related work. Section~\ref{sec:problem} formalizes the learning problem. Section~\ref{section:algorithm} proposes the new algorithms. Section~\ref{sec:consistency} proves uniform consistency. Section~\ref{section:experiments} conducts simulation experiments and estimates the dose response curve of cigarette smoking on infant birth weight, adjusting for unobserved confounding due to household income. Section~\ref{sec:conclusion} concludes.
\section{Related work}\label{sec:related}

I view dose response curves and heterogeneous treatment effects as reweightings of a structural function defined by an ill posed inverse problem. As such, I extend the \textit{partial means} framework \cite{newey1994kernel}. Existing work on partial means considers consumer surplus \cite{newey1994kernel} and certain causal parameters \cite{singh2020kernel} to be reweightings of a \textit{regression function}. By contrast, I consider causal parameters that are reweightings of a \textit{structural function}. My uniform analysis therefore generalizes uniform analysis in previous work. To express causal parameters in this way, I generalize identification theorems for treatment effects that use negative controls \cite{miao2018identifying,miao2018confounding,deaner2018nonparametric,tchetgen2020introduction}. 

Early work on negative controls emphasized their role in \textit{detection} of unobserved confounding. As early as the 1950s, epidemiologists proposed the principle of causal specificity as a diagnostic tool \cite{berkson1958smoking,yerushalmy1959methodology,ab1965environment}. Subsequent work formalized these concepts \cite{rosenbaum1989role,weiss2002can,lipsitch2010negative}. A more recent literature emphasizes the role of negative controls in \textit{adjustment} for unobserved confounding. Many papers eliminate
the bias from unobserved confounding by imposing additional structure: linearity and normality \cite{gagnon2012using,wang2017confounder}; joint normality \cite{kuroki2014measurement}; rank preservation of individual potential outcomes \cite{tchetgen2014control}; or monotonicity of confounding effects \cite{sofer2016negative}. I generalize identification results that relax such additional structure.

In econometrics, closely related strategies adjust for unobserved confounding in dynamic settings: \textit{difference-in-difference} \cite{card1990impact,meyer1995natural,abadie2005semiparametric}, and \textit{panel proxy control} \cite{deaner2018nonparametric}. Traditional difference-in-difference analysis requires strong assumptions such as linearity and additive separability of confounding. \cite{athey2006identification} present a more general approach, called \textit{changes-in-changes}, articulated in terms of a nonseparable, nonlinear structural model. A key assumption is monotonicity of confounding effects. Importantly, the model I present allows nonlinearity \textit{without} requiring monotonicity of confounding effects. The panel proxy control approach is also articulated in terms of a nonseparable, nonlinear structural model, and its static special case closely resembles the negative control model. \cite{deaner2018nonparametric} presents a series estimator as well as innovative strategies to handle ill posedness and completeness for both static and dynamic settings. It is straightforward to use the techniques developed in this paper to derive an RKHS estimator for the panel proxy setting. See \cite{sofer2016negative} and \cite{deaner2018nonparametric} for explicit comparisons of negative control with difference-in-difference and panel proxy control, respectively.

As previewed above, the causal parameters studied in this work are reweightings of a structural function called a confounding bridge, which closely resembles a nonparametric instrumental variable regression (NPIV). NPIV has a rich literature, including the seminal works of \cite{newey2003instrumental,hall2005nonparametric,blundell2007semi,darolles2011nonparametric,chen2011rate,chen2012estimation}, among others. The kernel ridge regression approach in this work employs RKHS-norm Tikhonov regularization over an infinite dimensional RKHS with a low effective dimension. See e.g. \cite{darolles2011nonparametric,hall2005nonparametric,horowitz2005nonparametric,carrasco2007linear, chen2012estimation}, and references therein, for a rich variety NPIV estimators that employ various types of Tikhonov regularizations over various infinite dimensional function spaces. In Appendix~\ref{sec:source}, I compare my approximation assumptions to the approximation assumptions employed in this literature, building on the discussion of \cite{chen2011rate}.

I contribute to a growing literature that adapts RKHS methods to treatment effect estimation. \cite{nie2017quasi} propose an RKHS estimator of heterogeneous treatment effects under selection on observables and prove mean square error rates. I pursue a more general definition of heterogeneous treatment effects conditional on some interpretable subvector $V\subset X$ \cite{abrevaya2015estimating} and allow for unobserved confounding. \cite{singh2019kernel} present an RKHS approach for nonparametric instrumental variable regression and prove projected mean square error rates in the sense of \cite{ai2003efficient}. \cite{singh2020kernel} present an RKHS approach for treatment effects identified by selection on observables and prove uniform rates. I unify the RKHS constructions in both works in order to handle both ill posedness and reweighting. My work is complementary in that I consider a new causal setting. My uniform analysis applies to not only negative control treatment effects but also nonparametric instrumental variable regression, providing alternative results under the same assumptions as \cite{singh2019kernel}. I build on fundamental statistical contributions from \cite{smale2005shannon,smale2007learning,fischer2017sobolev}.

\textcolor{black}{%
This draft subsumes \cite{singh2020kernel_original}. 
Several other works have proposed alternative RKHS estimators for the negative control setting. Independently and contemporaneously to \cite{singh2020kernel_original},
\cite{mastouri2021proximal} propose estimators for the dose response curve. \cite{mastouri2021proximal} formulate a kernel two stage regression approach and a kernel moment restriction approach, and formalize connections between them. \cite{mastouri2021proximal} analyze excess risk of a surrogate loss for the former, and consistency for the latter. Excess risk of a surrogate loss corresponds to projected mean square error for the confounding bridge. 
See Sections~\ref{section:algorithm} and~\ref{sec:consistency} for further comparisons. \cite{kallus2021causal,ghassami2021minimax} study the semiparametric problem rather than the nonparametric problem considered here. Both works propose doubly robust estimators that combine nuisance functions estimated by a minimax procedure. \cite{chernozhukov2021simple} provide abstract conditions to translate learning theory rates into semiparametric inference when treatment is binary. I summarize the connection to semiparametrics in Appendix~\ref{sec:semi}.} 

\textcolor{black}{The emphasis of this work is uniform consistency for the nonparametric case.
The main theoretical results of this paper are (i) uniform consistency of the confounding bridge and (ii) uniform consistency of causal functions estimated using a kernel two stage regression approach.
Uniform nonparametric inference remains an open question for future research.}
\section{Learning problem}\label{sec:problem}

\subsection{Treatment effects}

Treatment effects are statements about counterfactual outcomes given hypothetical interventions. Though we observe outcome $Y$, we seek to infer means of counterfactual outcomes $\{Y^{(d)}\}$, where $Y^{(d)}$ is the potential outcome given the hypothetical intervention $D=d$. The treatment effect literature aims to measure a rich variety of treatment effects, which I quote from \cite[Definition 3.1]{singh2020kernel}.

\begin{definition}[Treatment effects]\label{def:TE}
I define the following treatment effects.
\begin{enumerate}
    \item Dose response: $\theta_0^{ATE}(d):=\mathbb{E}[Y^{(d)}]$ is the counterfactual mean outcome given intervention $D=d$ for the entire population.
     \item Dose response with distribution shift: $ \theta_0^{DS}(d,\tilde{\mathbb{P}}):=\mathbb{E}_{\tilde{\mathbb{P}}}[Y^{(d)}]$ is the counterfactual mean outcome given intervention $D=d$ for an alternative population with data distribution $\tilde{\mathbb{P}}$ (elaborated in Assumption~\ref{assumption:covariate}).
    \item Conditional dose response: $ \theta_0^{ATT}(d,d'):=\mathbb{E}[Y^{(d')}|D=d]$ is the counterfactual mean outcome given intervention $D=d'$ for the subpopulation who actually received treatment $D=d$.
     \item Heterogeneous treatment effect: $\theta_0^{CATE}(d,v):=\mathbb{E}[Y^{(d)}|V=v]$ is the counterfactual mean outcome given intervention $D=d$ for the subpopulation with covariate value $V=v$.
\end{enumerate}
\end{definition}

The superscipt of each nonparametic treatment effect corresponds to its semiparametric analogue. If treatment is binary, then average treatment effect (ATE) is $\mathbb{E}[Y^{(1)}-Y^{(0)}]$; average treatment effect with distribution shift (DS) is $\mathbb{E}_{\tilde{\mathbb{P}}}[Y^{(1)}-Y^{(0)}]$; average treatment on the treated (ATT) is $\mathbb{E}[Y^{(1)}-Y^{(0)}|D=1]$; and conditional average treatment effect (CATE) is $\mathbb{E}[Y^{(1)}-Y^{(0)}|V=v]$. Rather than differences of potential outcomes indexed by binary treatment, I analyze potential outcomes indexed by discrete or continuous treatment.

$\theta_0^{ATE}(d)$ has many names: dose response curve, continuous treatment effect, and average structural function. If treatment is binary, then  $\theta_0^{ATE}(d)$ is a vector in $\mathbb{R}^2$ and the learning problem is semiparametric. If the treatment is discrete or continuous, then $\theta_0^{ATE}(d)$ is a function and the learning problem is nonparametric. $\theta_0^{DS}(d,\tilde{\mathbb{P}})$ is a closely related variant that handles the scenario where the covariate distribution has shifted. This variant may be called distribution shift, covariate shift, policy effect, or transfer learning.

Both $\theta_0^{ATT}(d,d')$ and $\theta_0^{CATE}(d,v)$  involve conditioning on a particular subpopulation. If treatment is binary, then $\theta_0^{ATT}(d,d')$ is a matrix in $\mathbb{R}^{2\times 2}$ and the learning problem is semiparametric. If the treatment is discrete or continuous, then $\theta_0^{ATT}(d,d')$ is a surface and the learning problem is nonparametric.  Likewise for $\theta_0^{CATE}(d,v)$. $\theta_0^{ATT}(d,d')$ is called the conditional dose response, and  $\theta_0^{CATE}(d,v)$ is called the heterogeneous treatment effect. The possibility for $D$ to be discrete or continuous and for $V$ to be a particular covariate, rather than the full set of covariates required for identification, is more general than the typical heterogeneous treatment effect \cite{nie2017quasi}. For $\theta_0^{CATE}$, I slightly abuse notation by denoting the complete set of identifying covariates as $(V,X)$.

\subsection{Negative control identification}

In pioneering work, \cite{tchetgen2020introduction} propose a potential outcome model in which treatment effects can be measured from outcomes $Y$, treatments $D$, and covariates $(V,X)$ despite unobserved confounding $U$. The technique involves two auxiliary variables: negative control treatment $Z$, and negative control outcome $W$. In this model, potential outcomes $\{Y^{(d,z)}\}$ and potential negative control outcomes $\{W^{(d,z)}\}$ are initially indexed by both the treatment value $D=d$ and the negative control treatment value $Z=z$. The identification strategy requires prior knowledge of how the unobserved confounder, which may be a vector, relates to the observed variables. The validity of negative controls as articulated in Assumptions~\ref{assumption:negative} and~\ref{assumption:solution} is \textit{relative} to a conjectured unobserved confounder. 

\begin{assumption}[Negative controls]\label{assumption:negative}
Assume
\begin{enumerate}
    \item No interference: if $D=d$ and $Z=z$ then $Y=Y^{(d,z)}$ and $W=W^{(d,z)}$.
    \item Latent exchangeability: $\{Y^{(d,z)}\},\{W^{(d,z)}\} \indep D,Z| U,X $.
    \item Overlap: if $f(u,x)>0$ then $f(d,z|u,x)>0$, where $f(u,x)$ and $f(d,z|u,x)$ are densities.
    \item Negative control treatment and outcome: $Y^{(d,z)}=Y^{(d)}$ and $W^{(d,z)}=W$.
\end{enumerate}
For $\theta_0^{CATE}$, replace $X$ with $(V,X)$.
\end{assumption}

No interference is also called consistency or the stable unit treatment value assumption in causal inference, and it rules out network effects. Latent exchangeability states that conditional on covariates $X$ and unobserved confounder $U$, treatment assignment and negative control treatment assignment are as good as random. Latent exchangeability relaxes the classic assumption of conditional exchangeability in which $U=\varnothing$, i.e. in which there is no unobserved confounder. Overlap ensures that there is no confounder-covariate stratum such that treatment and negative control treatment have a restricted support; for any stratum, any value of treatment or negative control treatment can occur. 

In a graphical causal model, there could be many sets of observed variables that could serve as covariates $X$ and many sets of unobserved variables that could serve as the unobserved confounder $U$ based on these initial criteria. The subsequent criteria provide guidance in how to choose $(X,U)$. We will see that the set of covariates $X$ should be chosen to block as much unobserved confounding as possible, because the variation in unobserved confounding that remains must be tied to variation in negative controls.

The negative control treatment condition imposes that the negative control treatment $Z$ only affects the outcome $Y$ via actual treatment $D$. It is identical to the exclusion restriction assumed for instrumental variables \cite{angrist1996identification}. The negative control outcome condition imposes that the negative control outcome $W$ is unaffected by the treatment $D$ and negative control treatment $Z$. It is an even stronger exclusion restriction. Altogether, Assumption~\ref{assumption:negative} formalizes the intuition that \textit{if there are spurious correlations then there is unobserved confounding}. It also implies $Y\indep Z | D,U,X$ and $W\indep D,Z | U,X$, which are weaker conditions used in the identification argument \cite{miao2018identifying}.

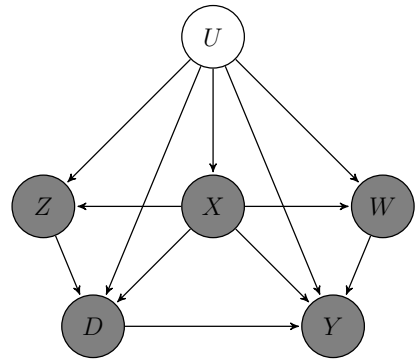
\begin{wrapfigure}{R}{0.33\textwidth}
\vspace{-10pt}
\begin{center}
\begin{adjustbox}{width=.33\textwidth}
\begin{tikzpicture}[->,>=stealth',shorten >=1pt,auto,node distance=2.8cm,
                    semithick]
  \tikzstyle{every state}=[draw=black,text=black]

  \node[state]         (x) [fill=gray]                   {$X$};
  \node[state]         (w) [right of=x, fill=gray]       {$W$};
  \node[state]         (z) [left of=x, fill=gray]       {$Z$};
  \node[state]         (d) [below left of=x, fill=gray]       {$D$};
    \node[state]         (y) [below right of=x, fill=gray]       {$Y$};
  \node[state]         (u) [above of=x]                  {$U$};

  \path (u) edge              node {$ $} (z)
             edge           node {$ $} (x)
             edge           node {$ $} (w)
              edge           node {$ $} (d)
             edge           node {$ $} (y)
        (x) edge              node {$ $} (w)
            edge            node {$ $} (z)
             edge           node {$ $} (d)
             edge           node {$ $} (y)
             (z) edge              node {$ $} (d)
            (w) edge              node {$ $} (y)
            (d) edge              node {$ $} (y);;
\end{tikzpicture}
\end{adjustbox}
\vspace{-15pt}
\caption{Negative control DAG}
\label{dag:nc}
\end{center}
\vspace{-15pt}
\end{wrapfigure}

Figure~\ref{dag:nc} visualizes a representative directed acyclic graph (DAG). Despite access to covariates $X$, unobserved confounding $U$ has unblocked paths to treatment $D$ and outcome $Y$. In the DAG, we also see the proxy interpretation of this learning problem. Covariates $X$, negative control treatment $Z$, and negative control outcome $W$ are all imperfect proxies for a set of control variables that would block unobserved confounding. Covariates $X$ are proxies that induce treatment and outcome; negative control treatment $Z$ is a proxy that induces treatment only; and negative control outcome $W$ is a  proxy that induces outcome only.

Next, I quote a high level technical condition, which I will later verify for the RKHS setting. Define the regression $\gamma_0(d,x,z):=\mathbb{E}[Y|D=d,X=x,Z=z]$.
\begin{assumption}[Confounding bridge]\label{assumption:solution}
Assume
\begin{enumerate}
    \item Existence: there exists a solution $h_0$ to the operator equation $$\gamma_0(d,x,z)=\mathbb{E}[h(D,X,W)|D=d,X=x,Z=z].$$
    \item Completeness: for any function $f$,
    $$
    \mathbb{E}[f(U)|D=d,X=x,Z=z]=0\quad \forall(d,x,z) \iff f(U)=0.
    $$
\end{enumerate}
\end{assumption}

I call $h_0$ the \textit{confounding bridge}, following \cite{miao2018confounding}. Here, we see the formal resemblance to the nonparametric instrumental variable regression problem (NPIV) \cite{newey2003instrumental}. In the language of NPIV, the LHS $\gamma_0(d,x,z)$ is the \textit{reduced form}, while the RHS is a composition of a \textit{stage 1} conditional expectation operator
and \textit{stage 2} structural function $h_0$. In the language of functional analysis, the operator equation is a Fredholm integral equation of the first kind. Solving this operator equation for $h_0$ involves inverting a linear operator with infinite dimensional domain; it is an ill posed problem. Indeed, existence will require conditions on the spectrum of the conditional expectation operator formalized in Appendix~\ref{section:existence}. Completeness is a technical condition from the NPIV literature. Taking $f=f_1-f_2$, it states that the observed variables $(D,X,Z)$ have sufficiently rich variation in the sense that different functions of unobserved confounding $f_1(U)$ and $f_2(U)$ lead to different projections onto $(D,X,Z)$; if they lead to the same projections, then $f_1=f_2$.

There is a subtle yet fundamental connection between the existence of the confounding bridge and the relevance of negative controls.
\begin{proposition}[Relevance]\label{prop:relevance}
Suppose Assumption~\ref{assumption:negative} holds and $\gamma_0$ varies in $z$, i.e. there exist $(z,z')$ such that $\gamma_0(d,x,z)\neq \gamma_0(d,x,z')$. 
\begin{enumerate}
    \item If the negative control treatment is irrelevant to the unobserved confounder in the sense that $Z\indep U |D,X$, then no confounding bridge exists.
    \item If the negative control outcome is irrelevant to the unobserved confounder in the sense that $W\indep U |X$, then no confounding bridge exists.
\end{enumerate}
\end{proposition}

See Appendix~\ref{section:existence} for the proof, as well as further discussion of the cases in which $(U,Z,W)$ are discrete or continuous. Assumption~\ref{assumption:solution} formalizes the converse intuition of Assumption~\ref{assumption:negative}: \textit{if there is unobserved confounding, then there are spurious correlations}. In order to use $(Z,W)$ to adjust for unobserved confounding $U$, it must be the case that $(Z,W)$ can detect $U$ well enough.
In practice, an analyst must collaborate with domain experts in order to assess (i) what are the sources of unobserved confounding, (ii) whether the negative control exclusion restrictions hold, and (iii) whether the negative controls are relevant. In summary, the key assumptions of negative control identification apply to settings where \textit{there are spurious correlations if and only if there is unobserved confounding}. When domain knowledge is insufficient to verify these assumptions, then the strategy of negative controls is inappropriate.

To handle $\theta_0^{DS}$, I generalize a standard assumption in transfer learning.
\begin{assumption}[Distribution shift]\label{assumption:covariate}
Assume
\begin{enumerate}
    \item The difference in population distributions $\mathbb{P}$ and $\tilde{\mathbb{P}}$ is only in the marginal distribution of treatments, negative control treatments, and covariates:
$$
\tilde{\mathbb{P}}(Y,W,D,X,Z)=\mathbb{P}(Y,W|D,X,Z)\tilde{\mathbb{P}}(D,X,Z).
$$
    \item $\tilde{\mathbb{P}}(D,X,Z)$ is absolutely continuous with respect to $\mathbb{P}(D,X,Z)$.
\end{enumerate}
\end{assumption}

\begin{proposition}[Invariance of confounding bridge]\label{prop:covariate}
Under Assumptions~\ref{assumption:solution} and~\ref{assumption:covariate}, the confounding bridge $h_0$ remains the same across the different populations $\mathbb{P}$ and $\tilde{\mathbb{P}}$.
\end{proposition}

See Appendix~\ref{sec:id} for the proof. It appears that Assumption~\ref{assumption:covariate} and Proposition~\ref{prop:covariate} are the first formalization of distribution shift in negative control and NPIV settings. 

In summary, I place three assumptions: availability of negative controls (Assumption~\ref{assumption:negative}); existence and completeness of the confounding bridge (Assumption~\ref{assumption:solution});  and invariance of the confounding bridge for transfer learning (Assumption~\ref{assumption:covariate}). Formally, the theorem that uses these assumptions to express treatment effects in terms of data is known as an identification result. I present the main identification result below, extending the powerful insights of \cite{miao2018identifying,miao2018confounding,deaner2018nonparametric,tchetgen2020introduction} to additional treatment effects beyond $\theta_0^{ATE}(d)$.

\begin{theorem}[Identification of treatment effects]\label{theorem:id_treatment}
If Assumptions~\ref{assumption:negative} and~\ref{assumption:solution} hold then
\begin{enumerate}
    \item $\theta_0^{ATE}(d)=\int h_0(d,x,w)\mathrm{d}\mathbb{P}(x,w)$.
    \item If in addition Assumption~\ref{assumption:covariate} holds, then $\theta_0^{DS}(d,\tilde{\mathbb{P}})=\int h_0(d,x,w)\mathrm{d}\tilde{\mathbb{P}}(x,w)$.
    \item $\theta_0^{ATT}(d,d')=\int h_0(d',x,w)\mathrm{d}\mathbb{P}(x,w|d)$.
    \item $\theta_0^{CATE}(d,v)=\int h_0(d,v,x,w)\mathrm{d}\mathbb{P}(x,w|v)$.
\end{enumerate}
\end{theorem}
See Appendix~\ref{sec:id} for the proof. In Theorem~\ref{theorem:id_treatment}, we see how negative controls $(Z,W)$ allow us to adjust for unobserved confounding $U$ and to thereby recover the treatment effect of interest. Each treatment effect is a reweighting of the confounding bridge $h_0$ defined in Assumption~\ref{assumption:solution} with respect to some distribution $\mathbb{Q}$. In this sense, each treatment effect is an example of the same general nonparametric learning problem: $\int h_0(d,x,w)\mathrm{d}\mathbb{Q}$,
where
$\mathbb{Q}$ may be an unconditional distribution such as $\mathbb{P}(x,w)$ or a conditional distribution such as $\mathbb{P}(x,w|d)$.

\subsection{RKHS background}

Until this point, I have placed only causal assumptions formalized in Assumptions~\ref{assumption:negative},~\ref{assumption:solution}, and~\ref{assumption:covariate}. For computational and analytical tractability, I now place additional structure on the learning problem: I assume key quantities are elements of a reproducing kernel Hilbert space (RKHS). The RKHS is a canonical setting in machine learning, and it is a space of smooth functions that generalizes the Sobolev space. For a broad statistical audience, I organize ideas from RKHS learning theory that underpin the algorithm derivation (Section~\ref{section:algorithm}) and consistency guarantee (Section~\ref{sec:consistency}) to follow.

\textbf{Kernel and feature map.} I begin with basic kernel and feature map notation. Consider the RKHS $\mathcal{H}$ which consist of functions of the form $f:\mathcal{A}\rightarrow\mathbb{R}$, where $\mathcal{A}$ is a Polish space (defined formally below). An RKHS $\mathcal{H}$ is characterized by its feature map $\phi(a)$, which can be interpreted as the dictionary of basis functions for the RKHS in the sense that, for any $f\in\mathcal{H}$, $f(a)=\langle f,\phi(a)\rangle_{\mathcal{H}}$. The kernel $k:\mathcal{A}\times \mathcal{A}\rightarrow \mathbb{R}$ is a positive definite, symmetric, and continuous function such that $k(a,a')=\langle \phi(a),\phi(a') \rangle_{\mathcal{H}}$; it is the inner product of features, so it encodes the geometry of the RKHS. Alternatively, one may define the kernel first, then define the feature map as $\phi:\mathcal{A}\rightarrow \mathcal{H}$, $a\mapsto k(a,\cdot)$. The feature map perspective is helpful for theory, but the kernel perspective is helpful for practice, since $k(a,a')$ is a scalar that can be computed. Ultimately, I will reduce the algorithm to kernel evaluations.

\textbf{Kernel mean embedding.} We have seen how the feature map helps to evaluate a function. A related object, called the kernel mean embedding, helps to take the expectation of a function. Suppose we wish to calculate $\mathbb{E}[f(A)]$. The idea of kernel mean embedding is to write
$$
\mathbb{E}[f(A)]=\int f(a)\mathrm{d}\mathbb{P}(a) = \int \langle f,\phi(a)\rangle_{\mathcal{H}} \mathrm{d}\mathbb{P}(a) =  \left\langle f,\int \phi(a)\mathrm{d}\mathbb{P}(a) \right\rangle_{\mathcal{H}} =\langle f,\mu \rangle_{\mathcal{H}}
$$
where the exchange of expectation and inner product requires weak regularity conditions (defined formally below). The object $\mu:=\int \phi(a)\mathrm{d}\mathbb{P}(a)$ is called the \textit{mean embedding} of the distribution $\mathbb{P}(a)$. A kernel is \textit{characteristic} when the mapping $\mathbb{P}(a)\mapsto \mu$ is injective. The geometry of the RKHS implies that, to calculate the expectation of a function, it suffices to take the product of the function and the mean embedding of the corresponding distribution. This idea extends to conditional distributions over a subset of the arguments of the function. I will use this idea extensively when deriving the algorithm.

\textbf{Tensor product.} What if a function is defined over multiple variables, e.g. $f:\mathcal{A}\times \mathcal{B} \rightarrow \mathbb{R}$? A natural approach is to define an RKHS $\mathcal{H}$ for such functions as the combination of RKHSs $\mathcal{H}_{\mathcal{A}}$ and $\mathcal{H}_{\mathcal{B}}$ that contain functions of the form $f_1:\mathcal{A}\rightarrow \mathbb{R}$ and $f_2:\mathcal{B}\rightarrow \mathbb{R}$, respectively. Denote the individual feature maps by $\phi_{\mathcal{A}}(a)$ and $\phi_{\mathcal{B}}(b)$, then define the tensor product feature map $\phi(a,b)=\phi_{\mathcal{A}}(a) \otimes \phi_{\mathcal{B}}(b)$ for the tensor product RKHS $\mathcal{H}$. The tensor product is a generalization of the outer product; formally, $[a\otimes b]c=a \langle b,c \rangle$. Then, for any $f\in\mathcal{H}$, $f(a,b)=\langle f,\phi_{\mathcal{A}}(a)\otimes \phi_{\mathcal{B}}(b)\rangle_{\mathcal{H}}$. It turns out that the kernel of this RKHS is simply the product of the kernels of the individual RKHSs: $k(a,b;a',b')=k_{\mathcal{A}}(a,a')\cdot k_{\mathcal{B}}(b,b')$. As such, $k(a,b;a',b')$ is a scalar that can be computed. In this work, I will extensively use tensor product constructions. For this reason, the algorithm statements will have the symbol $\odot$ for the elementwise product of objects that contain kernel evaluations.

\textbf{RKHS for operators.} So far, I have defined RKHSs for functions of one or more variables. RKHSs also exist for operators. I denote by $\mathcal{L}_2(\mathcal{H}_{\mathcal{A}},\mathcal{H}_{\mathcal{B}})$ the space of Hilbert-Schmidt operators of the form $E:\mathcal{H}_{\mathcal{A}}\rightarrow \mathcal{H}_{\mathcal{B}}$. It turns out that this space is an RKHS in its own right. The operators of interest are conditional expectation operators, which correspond to conditional mean embeddings. For example, consider the goal of calculating $\mathbb{E}[f(a)|B=b]$. As before, one can express
$$
\int f(a)\mathrm{d}\mathbb{P}(a|b) = \int \langle f,\phi_{\mathcal{A}}(a)\rangle_{\mathcal{H}_{\mathcal{A}}} \mathrm{d}\mathbb{P}(a|b) =  \left\langle f,\int \phi_{\mathcal{A}}(a)\mathrm{d}\mathbb{P}(a|b) \right\rangle_{\mathcal{H}_{\mathcal{A}}} =\langle f,\mu_{a}(b) \rangle_{\mathcal{H}_{\mathcal{A}}}
$$
where $\mu_{a}(b):=\int \phi_{\mathcal{A}}(a)\mathrm{d}\mathbb{P}(a|b)$ is called the \textit{conditional mean embedding} of the distribution $\mathbb{P}(a|b)$. Observe that
$$
\mu_{a}(b)=\int \phi(a)\mathrm{d}\mathbb{P}(a|b)=[E\phi(\cdot)](b)=[E^* \phi(b)](\cdot)
$$
where $E^*:\mathcal{H}_{\mathcal{B}}\rightarrow \mathcal{H}_{\mathcal{A}}$ is the adjoint of $E: \mathcal{H}_{\mathcal{A}}\rightarrow \mathcal{H}_{\mathcal{B}}$, and both are conditional expectation operators. Formally, the operators $E:f(\cdot)\mapsto \mathbb{E}[f(A)|B=\cdot]$ and $E^*:g(\cdot)\mapsto \mathbb{E}[g(B)|A=\cdot]$ encode the same information as the conditional mean embedding $\mu_{a}(b)$. This relationship will facilitate estimation and analysis.

\textbf{Closed form solution.} The final piece of RKHS machinery necessary for the algorithm derivation (Section~\ref{section:algorithm}) is the so-called \textit{kernel trick}. We have seen how a kernel evaluation is, in the end, simply a scalar. So, if an analyst can express an algorithm exclusively in terms of kernel evaluations, then the algorithm can be easily computed. A virtue of kernel methods is that they tend to have closed form solutions in terms of kernel evaluations.  Conceptually, an RKHS algorithm $\hat{f}(a)=\langle \hat{f},\phi(a)\rangle_{\mathcal{H}}$ involves a possibly nonlinear feature map $\phi(\cdot)$ applied to the data, so such an algorithm maintains computational simplicity while allowing for rich nonlinearity. Consider, for example, the kernel ridge regression
\begin{equation}\label{eq:loss}
    \hat{f}=\argmin_{f\in \mathcal{H}} \frac{1}{n}\sum_{i=1}^n \{y_i-\langle f,\phi(a_i) \rangle_{\mathcal{H}} \}^2+\lambda\|f\|^2_{\mathcal{H}}
\end{equation}
with regularization hyperparameter $\lambda>0$. Its closed form solution is
\begin{equation}\label{eq:sol}
    \hat{f}(a)=\langle \hat{f},\phi(a)\rangle_{\mathcal{H}}=\mathbf{Y}^{\top} (\mathbf{K}_{AA}+n\lambda \mathbf{I})^{-1}\mathbf{K}_{Aa}
\end{equation}
where $\mathbf{Y}\in\mathbb{R}^n$ is the vector of outcomes with $i$-th entry $y_i$, $\mathbf{K}_{AA}\in\mathbb{R}^{n\times n}$ is the kernel matrix with $(i,j)$-th entry $k(a_i,a_j)$, and $\mathbf{K}_{Aa}\in\mathbb{R}^{n}$ is the evaluation vector with $i$-th entry $k(a_i,a)$. The algorithms I propose generalize kernel ridge regression. Sometimes, instead of regressing the outcome $Y$ on features $\phi(A)$, I will regress the outcome $Y$ on mean embeddings $\mu_a(B)$. At other times, I will regress one collection of features $\phi(A)$ on another collection of features $\phi(B)$. These generalizations haves losses and closed form solutions that generalize those of kernel ridge regression. As such, they are simple combinations of kernel matrices and evaluation vectors despite being nonparametric. 

\textbf{Spectral view.} The statistical guarantees of Section~\ref{sec:consistency} require the \textit{spectral} view of the RKHS $\mathcal{H}$. The spectral view is more challenging, but it is the only way to articular RKHS learning theory. Let $\mathbb{L}_2$ denote the space of square integrable functions mapping from $\mathcal{A}$ to $\mathbb{R}$ with respect to measure $\mathbb{P}$. For a fixed kernel $k$, define the convolution operator $L:\mathbb{L}_2\rightarrow \mathbb{L}_2$, $f\mapsto \int k(a,\cdot)f(a)\mathrm{d}\mathbb{P}(a)$. By the spectral theorem, we can express the operator $L$ in terms of its countable eigenvalues $\{\eta_j\}$ and eigenfunctions $\{\varphi_j\}$: $Lf=\sum_{j=1}^{\infty} \eta_j \langle f, \varphi_j \rangle \cdot \varphi_j$. Without loss of generality, $\{\eta_j\}$ is a weakly decreasing sequence and $\{\varphi_j\}$ forms an orthonormal basis of $\mathbb{L}_2$. With this spectral notation, we are ready to formalize the sense in which the RKHS $\mathcal{H}$ is a smooth subset of $\mathbb{L}_2$. Since $\{\varphi_j\}$ forms an orthonormal basis of $\mathbb{L}_2$, any $f,g\in \mathbb{L}_2$ can be expressed as $f=\sum_{j=1}^{\infty}f_j\varphi_j$ and $g=\sum_{j=1}^{\infty}g_j\varphi_j$. By \cite[Theorem 4]{cucker2002mathematical}, $\mathbb{L}_2$ and the RKHS $\mathcal{H}$ can be explicitly represented as
\begin{align*}
    \mathbb{L}_2&=\left\{f=\sum_{j=1}^{\infty}f_j\varphi_j:\; \sum_{j=1}^{\infty}f_j^2<\infty\right\},\quad \langle f,g \rangle_{\mathbb{L}_2}=\sum_{j=1}^{\infty} f_jg_j \\
    \mathcal{H}&=\left\{f=\sum_{j=1}^{\infty}f_j\varphi_j:\;\sum_{j=1}^{\infty} \frac{f_j^2}{\eta_j}<\infty\right\},\quad \langle f,g \rangle_{\mathcal{H}}=\sum_{j=1}^{\infty} \frac{f_jg_j}{\eta_j}.
\end{align*}
The RKHS $\mathcal{H}$ is the subset of $\mathbb{L}_2$ for which higher order terms in the series $\{\varphi_j\}$ have a smaller contribution. In the RKHS, there is a penalty on higher order coefficients, and the magnitude of the penalty corresponds to how small the eigenvalue is.

\textbf{Main assumptions.} Finally, I articulate the main approximation assumptions of this paper. Formally, to analyze bias, I assume that a statistical target $f_0$ satisfies
\begin{equation}\label{eq:prior}
    f_0\in \mathcal{H}^c:=\left\{f=\sum_{j=1}^{\infty}f_j\varphi_j:\;\sum_{j=1}^{\infty} \frac{f_j^2}{\eta^c_j}<\infty\right\},\quad c\in(1,2].
\end{equation}
For $c=1$, we see that $\mathcal{H}^1=\mathcal{H}$; I am simply assuming that $f_0$ is correctly specified as an element of the RKHS. For $c>1$, I am assuming that $f_0$ is in the interior of the RKHS. This assumption is called the source condition in statistical learning theory and econometrics \cite{smale2007learning,caponnetto2007optimal,carrasco2007linear}. As we will see, a larger value of $c$ corresponds to a smoother target $f_0$ and a faster uniform rate. I allow $c$ to be as large as $c=2$, which is the highest degree of smoothness to which kernel ridge estimators can adapt. In Appendix~\ref{sec:source}, I compare this main approximation assumption with alternative approximation assumptions in the negative control and NPIV literatures.

The second main approximation is a spectral decay assumption called the effective dimension of the basis $\{\varphi_j\}$. I quantify the effective dimension as the rate at which the eigenvalues $\{\eta_j\}$ decay. Formally, to analyze variance, I assume that there exists some constant $C$ such that each $\eta_j$ satisfies 
\begin{equation}\label{eq:prior2}
\eta_j\leq C j^{-b},\quad b\geq1.
\end{equation}
\cite[Lemma 10]{fischer2017sobolev} shows that a bounded kernel $k$ satisfies this condition with $b$ that is at least one. A higher value of $b$ corresponds to a lower effective dimension, better control of the variance, and a faster rate. The limit $b\rightarrow \infty$ corresponds to an RKHS with finite dimension \cite{caponnetto2007optimal}.

\textbf{Special case: Sobolev space.} The abstract approximation conditions are easy to interpret in the context of Sobolev spaces. Denote by $\mathbb{H}_2^s$ the Sobolev space of functions of the form $f:\mathcal{A}\rightarrow \mathbb{R}$ with $\mathcal{A}\subset \mathbb{R}^p$. The parameter $s$ denotes how many derivatives of $f$ are square integrable. The Sobolev space $\mathbb{H}_2^s$ is an RKHS if and only if $s>p/2$ \cite[Theorem 132]{berlinet2011reproducing}. Its kernel is known as the \textit{Mat\'ern} kernel. Suppose we take $\mathcal{H}=\mathbb{H}_2^s$ with $s>p/2$ as the RKHS for estimation. If the true target $f_0$ is in $\mathbb{H}_2^{s_0}$, then $c=s_0/s$ \cite{fischer2017sobolev}. In the notation of~\eqref{eq:prior}, $\mathbb{H}_2^{s_0}=[\mathbb{H}_2^s]^c$. Clearly $c>1$ means that the target $f_0$ is in the interior of $\mathbb{H}_2^s$. Moreover, the effective dimension of the RKHS $\mathbb{H}_2^s$ is quantified by $b=2s/p$ \cite{fischer2017sobolev}. Rates in terms of $(b,c)$ adapt to the smoothness of $f_0$ and the effective dimension of the RKHS $\mathcal{H}$. They are invariant to dimension as long as $s_0>s>p/2$.
\section{Algorithm}\label{section:algorithm}

\subsection{RKHS construction}

I provide a new RKHS construction for negative control treatment effect estimation, generalizing and unifying the constructions in \cite{singh2019kernel,singh2020kernel}. In my construction, I define RKHSs for treatment $D$, negative controls $(Z,W)$, and covariates $(V,X)$. For example, for treatment $D$ define the RKHS $\mathcal{H}_{\mathcal{D}}$ with feature map $
\phi_{\mathcal{D}}(d)
$ and kernel $k_{\mathcal{D}}(d,d')=\langle \phi_{\mathcal{D}}(d),\phi_{\mathcal{D}}(d') \rangle_{\mathcal{H}_{\mathcal{D}}}$. Formally,
$
\phi_{\mathcal{D}}:\mathcal{D}\rightarrow \mathcal{H}_{\mathcal{D}}
$ and $k_{\mathcal{D}}:\mathcal{D} \times \mathcal{D} \rightarrow \mathbb{R}$.
To lighten notation, I suppress subscripts when arguments are provided, e.g. I write $\phi(d)=\phi_{\mathcal{D}}(d)$.

From these individual RKHSs, I construct a tensor product RKHS $\mathcal{H}$ for the confounding bridge $h_0$. For clarity of exposition, I initially focus on the case without $V$, i.e. excluding $\theta_0^{CATE}$. I assume the confounding bridge $h_0$ is an element of the RKHS with tensor product feature map $\phi(d,x,w):=\phi(d)\otimes \phi(x)\otimes \phi(w)$, i.e. $h_0\in \mathcal{H}:=\mathcal{H}_{\mathcal{D}}\otimes \mathcal{H}_{\mathcal{X}}\otimes \mathcal{H}_{\mathcal{W}}$. As before, the feature map can be interpreted as the dictionary of basis functions since
$
    h_0(d,x,w)=\langle h_0, \phi(d,x,w) \rangle_{\mathcal{H}}=\langle h_0, \phi(d)\otimes \phi(x)\otimes \phi(w)\rangle_{\mathcal{H}}
$. In Appendix~\ref{section:existence}, I discuss how an analogous assumption for the regression $\gamma_0$ relates to Assumption~\ref{assumption:solution}.

This tensor product RKHS construction plays a central role in deriving simple representations of treatment effects and hence deriving simple estimators, under weak regularity conditions. The RKHS aspect allows for the technique of kernel mean embedding: an analyst can reweight a function $h_0(d,x,w)$ according to some counterfactual distribution by reweighting its feature map $\phi(d,x,w)$ by that distribution. The tensor product aspect ensures separability of the different variables in the feature map; since $\phi(d,x,w)=\phi(d)\otimes \phi(x)\otimes \phi(w)$, the reweighting can apply to the specific variables $\phi(x)\otimes \phi(w)$. Using these properties, I represent the causal quantities defined in Definition~\ref{def:TE} and identified in Theorem~\ref{theorem:id_treatment} in a more tractable form. To begin, I state the regularity conditions. 

\begin{assumption}[RKHS regularity conditions]\label{assumption:RKHS}
Assume 
\begin{enumerate}
    \item $k_{\mathcal{D}}$, $k_{\mathcal{X}}$, $k_{\mathcal{W}}$, and $k_{\mathcal{Z}}$ are continuous and bounded:
    $$
    \sup_{d\in\mathcal{D}}\|\phi(d)\|_{\mathcal{H}_{\mathcal{D}}}\leq \kappa_d,\;  \sup_{x\in\mathcal{X}}\|\phi(x)\|_{\mathcal{H}_{\mathcal{X}}}\leq \kappa_x,\;  \sup_{w\in\mathcal{W}}\|\phi(w)\|_{\mathcal{H}_{\mathcal{W}}}\leq \kappa_w,\;   \sup_{z\in\mathcal{Z}}\|\phi(z)\|_{\mathcal{H}_{\mathcal{Z}}}\leq \kappa_z;
    $$
    \item $\phi(d)$, $\phi(x)$, $\phi(w)$, and $\phi(z)$ are measurable;
    \item $k_{\mathcal{X}}$ and $k_{\mathcal{W}}$ are characteristic.
\end{enumerate}
For $\theta_0^{CATE}$, extend the stated assumptions from $X$ to $(V,X)$.
\end{assumption}
Continuity, boundedness, and measurability are weak conditions satisfied by commonly used kernels. The characteristic property is a regularity condition for embedding a distribution in an RKHS, and it is satisfied by commonly used kernels as well \cite{sriperumbudur2010relation}. Formally, $k_{\mathcal{X}}$ is characteristic if and only if, for all Borel probability measures $\mathbb{Q}$, the mapping
    $\mathbb{Q}\mapsto \int \phi(x)\mathrm{d}\mathbb{Q}$ is injective. I explain the role of the characteristic property below.

\begin{theorem}[Representation via kernel mean embedding]\label{theorem:representation_treatment}
Suppose the conditions of Theorem~\ref{theorem:id_treatment} hold. Further suppose Assumption~\ref{assumption:RKHS} holds and $h_0\in\mathcal{H}$. Then
$$
\gamma_0(d,x,z)=\langle h_0,\phi(d)\otimes \phi(x)\otimes \mu_w(d,x,z) \rangle_{\mathcal{H}} \text{ where } \mu_w(d,x,z):=\int \phi(w) \mathrm{d}\mathbb{P}(w|d,x,z).
$$
Moreover
\begin{enumerate}
    \item $\theta_0^{ATE}(d)=\langle h_0, \phi(d)\otimes \mu \rangle_{\mathcal{H}} $ where $\mu:=\int[\phi(x)\otimes \phi(w)]\mathrm{d}\mathbb{P}(x,w)$;
    \item $\theta_0^{DS}(d,\tilde{\mathbb{P}})=\langle h_0, \phi(d)\otimes \nu \rangle_{\mathcal{H}} $ where $\nu:=\int[\phi(x)\otimes \phi(w)]\mathrm{d}\tilde{\mathbb{P}}(x,w)$;
    \item $\theta_0^{ATT}(d,d')=\langle h_0, \phi(d')\otimes \mu(d)\rangle_{\mathcal{H}} $ where $\mu(d):=\int[\phi(x)\otimes \phi(w)] \mathrm{d}\mathbb{P}(x,w|d)$;
    \item $\theta_0^{CATE}(d,v)=\langle h_0, \phi(d)\otimes \phi(v)\otimes \mu(v)\rangle_{\mathcal{H}} $ where $\mu(v):= \int [\phi(x)\otimes \phi(w)] \mathrm{d}\mathbb{P}(x,w|v)$.
\end{enumerate}
\end{theorem}


I present the proof in Appendix~\ref{section:alg_deriv}. Whereas the expressions in Theorem~\ref{theorem:id_treatment} are reweightings of the confounding bridge $h_0$, the expressions in Theorem~\ref{theorem:representation_treatment} are inner products of $h_0$. The quantity $\mu_w(d,x,z)$ encodes the conditional distribution $\mathbb{P}(w|d,x,z)$ from the integral equation that defines the confounding bridge $h_0$. The quantities $\mu,\nu,\mu(d),\mu(v)$ embed various reweighting distributions: $\mathbb{P}(x,w)$, $\tilde{\mathbb{P}}(x,w)$, $\mathbb{P}(x,w|d)$, and $\mathbb{P}(x,w|v)$, respectively. In general, the quantity $\int[\phi(x)\otimes \phi(w)]\mathrm{d}\mathbb{Q}$ encodes the distribution $\mathbb{Q}$ as a function in $\mathcal{H}_{\mathcal{X}}\otimes \mathcal{H}_{\mathcal{W}}$. The characteristic property for $k_{\mathcal{X}}$ and $k_{\mathcal{W}}$ ensures that the mapping $\mathbb{Q}\mapsto \int [\phi(x)\otimes \phi(w)]\mathrm{d}\mathbb{Q}$ is injective, so that the RKHS representation of the reweighting distribution $\mathbb{Q}$ is unique.

These representations abstractly suggest estimators. For example, for $\theta_0^{ATE}(d)$ the estimator should be of the form $\hat{\theta}^{ATE}(d)=\langle \hat{h}, \phi(d)\otimes \hat{\mu}\rangle_{\mathcal{H}}$, where $\hat{h}$ is an estimator of the confounding bridge $h_0$ and $\hat{\mu}$ is an estimator of the mean embedding $\mu$. I propose a regularized kernel estimator of the confounding bridge function in the spirit of two stage least squares (2SLS): first project $\phi(W)$ onto $\phi(D)\otimes \phi(X)\otimes \phi(Z)$ to obtain $\hat{\mu}_w(D,X,Z)$, then project $Y$ onto $\phi(D)\otimes \phi(X)\otimes \hat{\mu}_w(D,X,Z)$ to obtain $\hat{h}$. I estimate unconditional mean embeddings $\hat{\mu},\hat{\nu}$ with simple averages, and I estimate conditional mean embeddings $\hat{\mu}(d),\hat{\mu}(v)$ with projections similar to $\hat{\mu}_w(d,x,z)$.\footnote{When $V$ is discrete, a simple average suffices; see the discussion in Appendix~\ref{section:application_details}.} 

\subsection{Generalized regression loss}

It is not obvious that $\hat{\theta}^{ATE}(d)=\langle \hat{h}, \phi(d)\otimes \hat{\mu}\rangle_{\mathcal{H}}$ has a closed form expression in terms of kernel matrices. In this section, I state the generalized regression losses that define the estimator. These losses generalize the loss in~\eqref{eq:loss}. Along the way, I discuss the connection to 2SLS and provide intuition for these techniques. In the next section, I formally prove that a closed form expression exists and then solve for it, generalizing the expression in~\eqref{eq:sol}. 

Similar to 2SLS, I estimate the confounding bridge $\hat{h}$ in two stages. In the first stage, I estimate the conditional mean embedding $\hat{\mu}_w(d,x,z)$. Let $n$ be the number of observations of $(d_i,x_i,w_i,z_i)$ used to estimate the conditional mean embedding $\hat{\mu}_w(d,x,z)$ with regularization parameter $\lambda$. The generalized regression loss for the regression of $\phi(W)$ on $\phi(D,X,Z)$ is
$$
\hat{E}=\argmin_{E \in \mathcal{L}_2(\mathcal{H}_{\mathcal{W}},\mathcal{H}_{\mathcal{D}}\otimes \mathcal{H}_{\mathcal{X}} \otimes \mathcal{H}_{\mathcal{Z}})} \frac{1}{n}\sum_{i=1}^n \|\phi(w_i)-E^* \phi(d_i,x_i,z_i) \|^2_{\mathcal{H}_{\mathcal{W}}} + \lambda \|E\|^2_{\mathcal{L}_2(\mathcal{H}_{\mathcal{W}},\mathcal{H}_{\mathcal{D}}\otimes \mathcal{H}_{\mathcal{X}} \otimes \mathcal{H}_{\mathcal{Z}})} 
$$
so that $\hat{\mu}_w(d,x,z)=\hat{E}^* \phi(d,x,z)$. In this notation, $E^*$ is the adjoint operator of $E$, and $\mathcal{L}_2(\mathcal{H}_{\mathcal{W}},\mathcal{H}_{\mathcal{D}}\otimes \mathcal{H}_{\mathcal{X}} \otimes \mathcal{H}_{\mathcal{Z}})$ is an RKHS whose elements are operators of the form $E:\mathcal{H}_{\mathcal{W}} \rightarrow \mathcal{H}_{\mathcal{D}}\otimes \mathcal{H}_{\mathcal{X}} \otimes \mathcal{H}_{\mathcal{Z}}$.

In the second stage, I estimate the confounding bridge $\hat{h}$. Let $m$ be the number of observations of $(\dot{y}_i,\dot{d}_i,\dot{x}_i,\dot{z}_i)$ used to estimate the confounding bridge $\hat{h}$ with regularization parameter $\xi$. This notation allows the analyst to use different quantities of observations $(n,m)$ to estimate $\hat{\mu}_w(d,x,z)$ and $\hat{h}$, or to reuse the same observations. To estimate $\hat{h}$, I regress $Y$ on $\hat{\mu}(D,X,Z):=\phi(D)\otimes \phi(X)\otimes \hat{\mu}_w(D,X,Z)$. The generalized regression loss for the regression of $Y$ on $\hat{\mu}(D,X,Z)$ is
$$
\hat{h}=\argmin_{h\in\mathcal{H}}\frac{1}{m}\sum_{i=1}^{m}\left\{\dot{y}_i-\langle h, \hat{\mu}(\dot{d}_i,\dot{x}_i,\dot{z}_i)\rangle_{\mathcal{H}} \right\}^2+\xi\|h\|^2_{\mathcal{H}}.
$$

\textcolor{black}{\cite[eqs. 5, 6]{mastouri2021proximal} propose the same generalized regression losses for the confounding bridge. The original draft of this paper \cite{singh2020kernel_original} misquoted the regression loss for $\hat{E}$ from \cite[Section 4.1]{singh2019kernel}.\footnote{By misplacing the parentheses as $\mathbb{E}[\phi(D)\otimes \phi(X) \otimes \phi(W)|D,X,Z]$ instead of $\phi(D)\otimes \phi(X) \otimes \mathbb{E}[\phi(W)|D,X,Z]$, $\hat{E}$ was undefined.} This correction was pointed out by \cite{mastouri2021proximal} and an anonymous referee prior to the current draft. However, the closed form \cite[Proposition 2]{mastouri2021proximal} differs as detailed below.}

Finally, when estimating $\hat{\theta}^{ATT}$ and $\hat{\theta}^{CATE}$, one must also estimate the conditional mean embeddings $\hat{\mu}(d)$ and $\hat{\mu}(v)$. The losses for these conditional mean embeddings mirror the loss for the conditional mean embedding $\hat{\mu}_w(d,x,z)$.

\subsection{Closed form}

I present a closed form solution for the confounding bridge estimator, generalizing kernel instrumental variable regression \cite[Algorithm 1]{singh2019kernel} to my extended RKHS construction. 

\begin{algorithm}[Estimation of confounding bridge]\label{algorithm:bridge}
Let $\odot$ mean elementwise product. Then
\begin{align*}
  \mathbf{A}&=\mathbf{K}_{DD}\odot \mathbf{K}_{XX}\odot \mathbf{K}_{ZZ}\in\mathbb{R}^{n\times n},\quad 
   \dot{\mathbf{A}}=\mathbf{K}_{D\dot{D}}\odot \mathbf{K}_{X\dot{X}}\odot \mathbf{K}_{Z\dot{Z}}\in\mathbb{R}^{n\times m}, \\
  \mathbf{B}&=(\mathbf{A}+n\lambda \mathbf{I})^{-1} \dot{\mathbf{A}}\in\mathbb{R}^{n\times m},\quad \quad \quad 
    \mathbf{M}=\mathbf{K}_{\dot{D}\dot{D}}\odot \mathbf{K}_{\dot{X}\dot{X}}\odot \{\mathbf{B}^{\top}\mathbf{K}_{WW} \mathbf{B}\}  \in\mathbb{R}^{m\times m}, \\
    \boldsymbol{\hat{\alpha}}&=(\mathbf{M}\mathbf{M}^{\top}+m\xi \mathbf{M})^{-1}\mathbf{M}\dot{\mathbf{Y}} \in \mathbb{R}^{m} ,\quad 
    \hat{h}(d,x,w)=\boldsymbol{\hat{\alpha}}^{\top}[\mathbf{K}_{\dot{D}d}\odot \mathbf{K}_{\dot{X}x}\odot \{\mathbf{B}^{\top} \mathbf{K}_{Ww}\}]\in\mathbb{R}
\end{align*}
where $(\lambda,\xi)$ are ridge penalty hyperparameters.
\end{algorithm}
See Appendix~\ref{section:alg_deriv} for the derivation, which begins with an original proof that such an $\boldsymbol{\hat{\alpha}} \in\mathbb{R}^m$ even exists. \textcolor{black}{\cite[Proposition 2]{mastouri2021proximal} show that a matrix representation $\boldsymbol{\hat{\alpha}}\in \mathbb{R}^{n\times m}$ exists, rather than the vector representation $\boldsymbol{\hat{\alpha}} \in\mathbb{R}^m$ in Algorithm~\ref{algorithm:bridge}. The vector of representation of Algorithm~\ref{algorithm:bridge} is similar to kernel ridge regression.
} The elementwise products arise because tensor product RKHSs correspond to product kernels. For example, the kernel of $\mathcal{H}_{\mathcal{D}}\otimes \mathcal{H}_{\mathcal{X}}\otimes \mathcal{H}_{\mathcal{Z}}$ is $k(d,x,z;d',x',z')=k_{\mathcal{D}}(d,d')k_{\mathcal{X}}(x,x')k_{\mathcal{Z}}(z,z')$ so its kernel matrix is $\mathbf{K}_{DD}\odot \mathbf{K}_{XX}\odot \mathbf{K}_{ZZ}$. $\mathbf{M}$ is essentially the kernel matrix for the conditional mean embedding $\hat{\mu}(d,x,z)=\phi(d)\otimes \phi(x)\otimes \hat{\mu}_w(d,x,z)$. Interpreting these expressions, $\boldsymbol{\hat{\alpha}}$ is clearly the regularized empirical projection of $Y$ onto $\phi(D)\otimes \phi(X)\otimes \hat{\mu}_w(D,X,Z)$ in the spirit of 2SLS. 

Next, I present closed form solutions for treatment effects, e.g. $\hat{\theta}^{ATE}(d)=\langle \hat{h}, \phi(d)\otimes \hat{\mu}\rangle_{\mathcal{H}}$, building on $\hat{h}$ from Algorithm~\ref{algorithm:bridge}. Whereas \cite[Algorithm 3.1]{singh2020kernel} estimate treatment effects assuming selection on observables, Algorithm~\ref{algorithm:treatment} estimates treatment effects assuming access to negative controls. For $\theta_0^{DS}$, let $\tilde{n}$ be the number of observations of $(\tilde{x}_i,\tilde{w}_i)$ drawn from population $\tilde{\mathbb{P}}$.

\begin{algorithm}[Estimation of treatment effects]\label{algorithm:treatment}
Treatment effect estimators have the closed form solutions
\begin{enumerate}
    \item $\hat{\theta}^{ATE}(d)=n^{-1}\sum_{i=1}^n \boldsymbol{\hat{\alpha}}^{\top}[\mathbf{K}_{\dot{D}d}\odot \mathbf{K}_{\dot{X}x_i}\odot \{\mathbf{B}^{\top} \mathbf{K}_{Ww_i}\}]$
     \item $\hat{\theta}^{DS}(d,\tilde{\mathbb{P}})=\tilde{n}^{-1}\sum_{i=1}^{\tilde{n}} \boldsymbol{\hat{\alpha}}^{\top}[\mathbf{K}_{\dot{D}d}\odot \mathbf{K}_{\dot{X}\tilde{x}_i}\odot \{\mathbf{B}^{\top} \mathbf{K}_{W\tilde{w}_i}\}]$
    \item $\hat{\theta}^{ATT}(d,d')=\boldsymbol{\hat{\alpha}}^{\top}[\mathbf{K}_{\dot{D}d'}\odot \{[\mathbf{K}_{\dot{X}X}\odot \{\mathbf{B}^{\top}\mathbf{K}_{WW}\}](\mathbf{K}_{DD}+n\lambda_1\mathbf{I})^{-1}\mathbf{K}_{Dd}\}]$
    \item $\hat{\theta}^{CATE}(d,v)=\boldsymbol{\hat{\alpha}}^{\top}[\mathbf{K}_{\dot{D}d}\odot \mathbf{K}_{\dot{V}v}\odot \{[\mathbf{K}_{\dot{X}X}\odot \{\mathbf{B}^{\top}\mathbf{K}_{WW}\}](\mathbf{K}_{VV}+n\lambda_2\mathbf{I})^{-1}\mathbf{K}_{Vv}\}] $
\end{enumerate}
where $(\lambda_1,\lambda_2)$ are ridge regression penalty hyperparameters. In $\hat{\theta}^{CATE}(d,v)$, $\boldsymbol{\hat{\alpha}}$ is the coefficient for the confounding bridge that includes $V$.
\end{algorithm}

See Appendix~\ref{section:alg_deriv} for the derivation. I give theoretical values for the regularization parameters that balance bias and variance in Section~\ref{sec:consistency} below. In particular, I specify $(\lambda,\xi)$ in Theorem~\ref{theorem:consistency_bridge} and $(\lambda_1,\lambda_2)$ in Theorem~\ref{theorem:consistency_treatment}. In Appendix~\ref{section:tuning}, I propose a practical tuning procedure based on the closed form solution of leave-one-out cross validation (LOOCV) to empirically balance bias and variance, and I discuss the time complexity.

\subsection{Summary}

To fix ideas, I summarize the end-to-end procedure for $\hat{\theta}^{ATE}(d)$. For simplicity, I suppose that the analyst re-uses observations in the two stages of confounding bridge estimation. I provide additional discussion in Appendix~\ref{section:alg_deriv} for researchers who are new to kernel methods.

\begin{algorithm}[End-to-end details]\label{algorithm:concrete_main}
Given $n$ observations of outcome $Y$, treatment $D$, covariates $X$, negative control outcome $W$, and negative control treatment $Z$,
\begin{enumerate}
 \item Specify the kernels $k_{\mathcal{D}},k_{\mathcal{X}},k_{\mathcal{W}},k_{\mathcal{Z}}$.
 \begin{enumerate}
     \item For multivariate objects, e.g. $X=(X_1,...,X_p)$, use the product of scalar kernels
     $$
     k_{\mathcal{X}}(x,x')=\prod_{j=1}^p k_{\mathcal{X}_1}(x_1,x_1')\cdot ...\cdot k_{\mathcal{X}_p}(x_p,x_p').
     $$
     \item Tune the scalar kernel hyperparameters. For example, if the treatment kernel $k_{D}$ is chosen as the Gaussian kernel, a standard heuristic is to use the median interpoint distance among observed treatment values.
     \item Compute the kernel matrices, e.g. $\mathbf{K}_{DD}\in\mathbb{R}^{n\times n}$ with $(i,j)$-th entry $k_{\mathcal{D}}(d_i,d_j)$.
 \end{enumerate}
    \item Specify the regularization hyperparameters $(\lambda,\xi)$.
    \begin{enumerate}
        \item For $\lambda$, I derive the closed form solution of LOOCV in Appendix~\ref{section:tuning}.
        \item The same procedure applies to $\xi$, plugging in the chosen value of $\lambda$.
    \end{enumerate}
    \item Estimate the confounding bridge $\hat{h}$ in two stages, using $(\lambda,\xi)$.
    \begin{enumerate}
        \item Estimate the distribution in the integral equation $\hat{\mathbb{P}}(w|d,x,z)$ via its mean embedding $\hat{\mu}_w(d,x,z)$ with regularization $\lambda$ as
 $$
[\hat{\mu}_w(d,x,z)](w)=\mathbf{K}_{wW} (\mathbf{K}_{DD}\odot \mathbf{K}_{XX}\odot \mathbf{K}_{ZZ}+n\lambda \mathbf{I})^{-1}[\mathbf{K}_{Dd}\odot \mathbf{K}_{Xx} \odot \mathbf{K}_{Zz}].
$$
        \item Regress $Y$ onto $\phi(D)\otimes \phi(X)\otimes \hat{\mu}_w(D,X,Z)$ with regularization $\xi$. It turns out that $\hat{h}(d,x,w)=\boldsymbol{\hat{\alpha}}^{\top}[\mathbf{K}_{Dd}\odot \mathbf{K}_{Xx}\odot \{\mathbf{B}^{\top} \mathbf{K}_{Ww}\}]\in\mathbb{R}$ where
        \begin{align*}
  \mathbf{A}&=\mathbf{K}_{DD}\odot \mathbf{K}_{XX}\odot \mathbf{K}_{ZZ}\in\mathbb{R}^{n\times n},\quad 
  \mathbf{B}=(\mathbf{A}+n\lambda \mathbf{I})^{-1} \mathbf{A}\in\mathbb{R}^{n\times n}, \\
    \mathbf{M}&=\mathbf{K}_{DD}\odot \mathbf{K}_{XX}\odot \{\mathbf{B}^{\top}\mathbf{K}_{WW} \mathbf{B}\}  \in\mathbb{R}^{n\times n},\quad  
    \boldsymbol{\hat{\alpha}}=(\mathbf{M}\mathbf{M}^{\top}+n\xi \mathbf{M})^{-1}\mathbf{M}\mathbf{Y} \in \mathbb{R}^{n}.
\end{align*}
    \end{enumerate}
    \item Estimate the counterfactual distribution $\hat{\mathbb{P}}(x,w)$ via its mean embedding $\hat{\mu}$ as
    $$
    [\hat{\mu}](x,w)=\frac{1}{n}\sum_{i=1}^n k_{\mathcal{X}}(x_i,x)k_{\mathcal{W}}(w_i,w)\in\mathbb{R}.
    $$
    \item Estimate the dose response $\hat{\theta}^{ATE}(d)$ by combining $\hat{h}$ and $\hat{\mu}$ according to $
\hat{\theta}^{ATE}(d)=\langle \hat{h}, \phi(d)\otimes \hat{\mu} \rangle_{\mathcal{H}}
$. To do so, match the common arguments $(x,w)$ of $\hat{h}$ and $\hat{\mu}$. In summary,
$$
\hat{\theta}^{ATE}(d)=\frac{1}{n}\sum_{i=1}^n \boldsymbol{\hat{\alpha}}^{\top}[\mathbf{K}_{Dd}\odot \mathbf{K}_{Xx_i}\odot \{\mathbf{B}^{\top} \mathbf{K}_{Ww_i}\}]\in\mathbb{R}.
$$
\end{enumerate}
\end{algorithm}
\section{Consistency}\label{sec:consistency}

To define the learning problem in Section~\ref{sec:problem}, I placed three assumptions: availability of negative controls (Assumption~\ref{assumption:negative}); existence and completeness of the confounding bridge (Assumption~\ref{assumption:solution});  and invariance of the confounding bridge for transfer learning (Assumption~\ref{assumption:covariate}). To construct an algorithm in Section~\ref{section:algorithm}, I assumed RKHS regularity (Assumption~\ref{assumption:RKHS}). To guarantee uniform consistency in this section, I place three final assumptions: original space regularity (Assumption~\ref{assumption:original}); smoothness and effective dimension of conditional expectation operators (Assumption~\ref{assumption:smooth_op}); and smoothness and effective dimension of the confounding bridge (Assumption~\ref{assumption:smooth_bridge}). I first prove uniform consistency of the confounding bridge, then uniform consistency of treatment effects. As before, for $\theta_0^{CATE}$ I extend the stated assumptions from $X$ to $(V,X)$. 

\textcolor{black}{\cite[Theorem 2]{mastouri2021proximal} analyze excess risk of a surrogate loss for a kernel two stage regression estimator of the confounding bridge, with finite sample rates, under smoothness and effective dimension assumptions. Excess risk of a surrogate loss corresponds to projected mean square error for the confounding bridge. 
}  

\subsection{Confounding bridge}

I require weak regularity conditions on the original spaces of the outcome $Y$, treatment $D$, covariates $(V,X)$, and negative controls $(W,Z)$.

\begin{assumption}[Original space regularity conditions]\label{assumption:original}
Assume
\begin{enumerate}
\item $Y \in\mathcal{Y}\subset \mathbb{R}$ is bounded, i.e. there exists $C<\infty$ such that $|Y|\leq C$ almost surely.
    \item $\mathcal{D}$, $\mathcal{X}$, $\mathcal{W}$, and $\mathcal{Z}$ are Polish spaces.
\end{enumerate}
\end{assumption}

To simplify notation and analysis, I require that the outcome $Y\in\mathbb{R}$ is a bounded scalar. More generally, $\mathcal{Y}$ could be a separable Hilbert space. I preserve generality for treatment, covariates, and negative controls. A Polish space is a separable and completely metrizable topological space. Random variables with support in a Polish space may be discrete or continuous and low, high, or infinite dimensional.  As such, I allow for treatment, covariates, and negative controls that could even be texts, graphs, or images. 

Next, I place a smoothness and effective dimension conditions in the sense of~\eqref{eq:prior} and~\eqref{eq:prior2} for the conditional mean embedding $\mu_{w}(d,x,z)$. In anticipation of later analysis, I articulate this assumption abstractly. Consider the abstract conditional mean embedding $\mu_{a}(b):=\int \phi(a)\mathrm{d}\mathbb{P}(a|b)$ where $a\in\mathcal{A}_{\ell}$ and $b\in\mathcal{B}_{\ell}$. I will ultimately consider the three different conditional mean embeddings $\mu_w(d,x,z)$, $\mu(d)$, and $\mu(v)$ indexed by $\ell \in\{0,1,2\}$. As previewed in Section~\ref{sec:problem}, the conditional expectation operator $E_{\ell}:\mathcal{H}_{\mathcal{A}_{\ell}}\rightarrow\mathcal{H}_{\mathcal{B}_{\ell}}$, $f(\cdot)\mapsto \mathbb{E}[f(A_{\ell})|B_{\ell}=\cdot]$ encodes the same information as $\mu_{a}(b)$. In particular,
$$
\mu_{a}(b)=\int \phi(a)\mathrm{d}\mathbb{P}(a|b) =[E_{\ell}\phi(\cdot)](b) =[E_{\ell}^* \phi(b)](\cdot),\quad a\in\mathcal{A}_{\ell},\quad  b\in\mathcal{B}_{\ell}
$$
where $E_{\ell}^*:\mathcal{H}_{\mathcal{B}_{\ell}} \rightarrow \mathcal{H}_{\mathcal{A}_{\ell}}$, $g(\cdot)\mapsto  \mathbb{E}[g(B_{\ell})|A_{\ell}=\cdot]$ is the adjoint of $E_{\ell}$. I denote the space of Hilbert-Schmidt operators between $\mathcal{H}_{\mathcal{A}_{\ell}}$ and $\mathcal{H}_{\mathcal{B}_{\ell}}$ by $\mathcal{L}_2(\mathcal{H}_{\mathcal{A}_{\ell}},\mathcal{H}_{\mathcal{B}_{\ell}})$, which is an RKHS in its own right.  

\begin{assumption}[Smoothness and spectral decay for conditional expectation operator]\label{assumption:smooth_op}
Assume $E_{\ell}\in [\mathcal{L}_2(\mathcal{H}_{\mathcal{A}_{\ell}},\mathcal{H}_{\mathcal{B}_{\ell}})]^{^{c_{\ell}}}$ and $\eta_j(\mathcal{H}_{\mathcal{B}_{\ell}})\leq C j^{-b_{\ell}}$.
\end{assumption}

To specialize the assumption, all one has to do is specify $\mathcal{A}_{\ell}$ and $\mathcal{B}_{\ell}$. For example, for $\mu_{w}(d,x,z)$, $\mathcal{A}_0=\mathcal{W}$ and $\mathcal{B}_0=\mathcal{D}\times \mathcal{X}\times \mathcal{Z}$. By assuming smoothness and effective dimension of $E_0$, I assume smoothness and effective dimension of $\mu_w(d,x,z)$. I explicitly specialize the assumption in Appendices~\ref{section:consistency_proof1} and~\ref{section:consistency_proof2}.

The final assumption that I place in order to prove uniform consistency of the confounding bridge estimator $\hat{h}$ is that the confounding bridge $h_0$ is smooth in the sense of~\eqref{eq:prior} with low effective dimension in the sense of~\eqref{eq:prior2}. Recall that the features for the RKHS $\mathcal{H}$ are $\phi(d,x,w)=\phi(d)\otimes \phi(x)\otimes \phi(w)$. Recall from Theorem~\ref{theorem:representation_treatment} that we must solve the integral equation $\gamma_0(d,x,z)=\langle h_0, \mu(d,x,z) \rangle_{\mathcal{H}}$ where $\mu(d,x,z)=\phi(d)\otimes \phi(x)\otimes \mu_w(d,x,z)$ is a mean embedding. By construction,
$$
\langle \mu(d,x,z), \mu(d',x',z') \rangle_{\mathcal{H}}=k_{\mathcal{D}}(d,d')k_{\mathcal{X}}(x,x')\int k_{\mathcal{W}}(w,w')\mathrm{d}\mathbb{P}(w|d,x,z)\mathrm{d}\mathbb{P}(w'|d',x',z'),
$$
which we may interpret as the inner product of a space $\mathcal{H}_{\mu}\subset \mathcal{H}$. In words, the kernel of $\mathcal{H}_{\mu}$ is simply the kernel of $\mathcal{H}$ after integrating out $(w,w')$ according to the conditional distribution $\mathbb{P}(w|d,x,z)$ from the integral equation. Formally, $\mathcal{H}_{\mu}$ can be viewed as an RKHS of functions evaluated on mean embeddings instead of features \cite{szabo2016learning}.

\begin{assumption}[Smoothness and spectral decay for confounding bridge]\label{assumption:smooth_bridge}
Assume $h_0\in \mathcal{H}_{\mu}^c$ and $\eta_j(\mathcal{H}_{\mu})\leq C j^{-b}$.
\end{assumption}

As a technical aside, an analyst may introduce additional nonlinearity by enriching the model and enriching this assumption. A richer model would instead allow $\gamma_0(d,x,z)=H_0\mu(d,x,z)$ where $H_0:\mathcal{H}\rightarrow \mathbb{R}$ is a nonlinear mapping in a richer RKHS. In such case, the smoothness and effective dimension assumptions would be placed on $H_0$ rather than $h_0$ and a notion of H\"older continuity is required \cite[Table 1]{szabo2016learning}. For clarity, I omit this complexity. 

Under these conditions, I arrive at the first main result: uniform consistency of the confounding bridge. This result appears to be the first finite sample uniform analysis of nonparametric instrumental variable regression in the RKHS. By allowing $n\neq m$, I allow
the possibility of asymmetric sample splitting and the use of observations from different
data sets. The finite sample rates are expressed in terms of $(n,m)$. The parameter $a>0$ characterizes the ratio between the sample sizes.

\begin{theorem}[Consistency of confounding bridge]\label{theorem:consistency_bridge}
Suppose Assumptions~\ref{assumption:solution}, \ref{assumption:RKHS}, \ref{assumption:original}, \ref{assumption:smooth_op} with $\mathcal{A}_0=\mathcal{W}$ and $\mathcal{B}_0=\mathcal{D}\times \mathcal{X}\times \mathcal{Z}$, and~\ref{assumption:smooth_bridge} hold. Set $\lambda=n^{-\frac{1}{c_0+1/b_0}}$ and $n=m^{\frac{a(c_0+1/b_0)}{c_0-1}}$ where $a>0$.
\begin{enumerate}
    \item If $a\leq (c+3)/(c+1/b)$ then 
    $\|\hat{h}-h_0\|_{\infty}=O_p(m^{-\frac{1}{2}\frac{a(c-1)}{c+3}})$ with $\xi=m^{-\frac{a}{c+3}}$.
    \item If $a\geq (c+3)/(c+1/b)$ then $\|\hat{h}-h_0\|_{\infty}=O_p(m^{-\frac{1}{2}\frac{c-1}{c+1/b}})$ with $\xi=m^{-\frac{1}{c+1/b}}$.
\end{enumerate}
\end{theorem}
See Appendix~\ref{section:consistency_proof1} for exact finite sample rates and intermediate results in RKHS norm. At $a=(c+3)/(c+1/b)$, the convergence rate $m^{-\frac{1}{2}\frac{c-1}{c+1/b}}$ attains the rate of single stage kernel ridge regression with respect to $m$ \cite{fischer2017sobolev}. This rate is calibrated by $c$, the smoothness of confounding bridge and $b$, the effective dimension of its RKHS.

The rate of Theorem~\ref{theorem:consistency_bridge} requires the ratio between sample sizes to be $n=m^{\frac{c+3}{c+1/b}\cdot \frac{(c_0+1/b_0)}{c_0-1}}$, implying $n\gg m$. In practice, the analyst often uses the same observations to estimate $\hat{\mu}_w(d,x,z)$ and $\hat{h}$.

\begin{corollary}[Reusing samples]\label{corollary:reuse}
If samples are reused to estimate $\hat{\mu}_w(d,x,z)$ and $\hat{h}$, then $n=m$, $a=(c_0-1)/(c_0+1/b_0)$, $\xi=n^{-\frac{c_0-1}{(c_0+1/b_0)(c+3)}}$, and $\|\hat{h}-h_0\|_{\infty}=O_p(n^{-\frac{1}{2}\frac{c_0-1}{c_0+1/b_0}\frac{c-1}{c+3}})$.
\end{corollary}
This rate adapts to the smoothness $c_0$ of the conditional distribution as well as the smoothness $c$ of the confounding bridge. The slow rate reflects the challenge of a uniform norm guarantee in an ill posed inverse problem. A faster rate could be possible under further assumptions, which I leave to future work.

\subsection{Treatment effects}

Recall from Theorem~\ref{theorem:representation_treatment} that $\theta_0^{ATT}$ and $\theta_0^{CATE}$ contain conditional mean embeddings $\mu(d)$ and $\mu(v)$, respectively. I estimate these conditional mean embeddings by regularized projections in Algorithm~\ref{algorithm:treatment}. To control bias and variance, I place smoothness and effective dimension conditions in the sense of~\eqref{eq:prior} and~\eqref{eq:prior2} for $\mu(d)$ and $\mu(v)$ as well. As previewed in the discussion about $\mu_w(d,x,z)$ and $E_0$, a conditional mean embedding corresponds to a conditional expectation operator. As before, all one has to do is specify $\mathcal{A}_{\ell}$ and $\mathcal{B}_{\ell}$ to specialize the assumption. For $\mu(d)$, $\mathcal{A}_1=\mathcal{X}\times \mathcal{W}$ and $\mathcal{B}_1=\mathcal{D}$; for $\mu(v)$, $\mathcal{A}_2=\mathcal{X}\times \mathcal{W}$ and $\mathcal{B}_2=\mathcal{V}$. 

Under these conditions, I arrive at the second main result: uniform consistency of the treatment effect estimators. For simplicity, I specialize to the scenario with the fastest rates. Recall $\tilde{n}$ is the number of observations drawn from the alternative population $\tilde{\mathbb{P}}$.

\begin{theorem}[Consistency of treatment effects]\label{theorem:consistency_treatment}
Suppose Assumption~\ref{assumption:negative} holds, as well as the conditions of Theorem~\ref{theorem:consistency_bridge}. Set $(\lambda,\lambda_1,\lambda_2)=(n^{-\frac{1}{c_0+1/b_0}},n^{-\frac{1}{c_1+1/b_1}},n^{-\frac{1}{c_2+1/b_2}})$, $\xi=m^{-\frac{1}{c+1/b}}$, and $n=m^{\frac{c+3}{c+1/b}\cdot \frac{(c_0+1/b_0)}{c_0-1}}$.
\begin{enumerate}
    \item Then
    $$
    \|\hat{\theta}^{ATE}-\theta_0^{ATE}\|_{\infty}=O_p\left(m^{-\frac{1}{2}\frac{c-1}{c+1/b}}+n^{-\frac{1}{2}}\right).
    $$
    \item If in addition Assumption~\ref{assumption:covariate} holds, then 
      $$
    \|\hat{\theta}^{DS}(\cdot,\tilde{\mathbb{P}})-\theta_0^{DS}(\cdot,\tilde{\mathbb{P}})\|_{\infty}=O_p\left( m^{-\frac{1}{2}\frac{c-1}{c+1/b}}+\tilde{n}^{-\frac{1}{2}}\right).
    $$
    \item If in addition Assumption~\ref{assumption:smooth_op} holds with $\mathcal{A}_1=\mathcal{X}\times \mathcal{W}$ and $\mathcal{B}_1=\mathcal{D}$, then
      $$
    \|\hat{\theta}^{ATT}-\theta_0^{ATT}\|_{\infty}=O_p\left(m^{-\frac{1}{2}\frac{c-1}{c+1/b}}+n^{-\frac{1}{2}\frac{c_1-1}{c_1+1/b_1}}\right).
    $$
    \item If in addition Assumption~\ref{assumption:smooth_op} holds with $\mathcal{A}_2=\mathcal{X}\times \mathcal{W}$ and $\mathcal{B}_2=\mathcal{V}$, then
      $$
    \|\hat{\theta}^{CATE}-\theta_0^{CATE}\|_{\infty}=O_p\left(m^{-\frac{1}{2}\frac{c-1}{c+1/b}}+n^{-\frac{1}{2}\frac{c_2-1}{c_2+1/b_2}}\right).
    $$
\end{enumerate}
\end{theorem}
See Appendix~\ref{section:consistency_proof2} for exact finite sample rates. Inspecting the rates, we see that each one is a sum of the rate for the confounding bridge from Theorem~\ref{theorem:consistency_bridge} with the rate for the appropriate mean embedding estimation procedure. The rates adapt to the smoothness parameters $(c,c_0,c_1,c_2)$ and effective dimension parameters $(b,b_0,b_1,b_2)$ of the confounding bridge $h_0$ and the conditional expectation operators $E_0$, $E_1$, and $E_2$. Equivalently, the rates adapt to the smoothness and effective dimension of the confounding bridge $h_0$ and the conditional distributions $\mathbb{P}(w|d,x,z)$, $\mathbb{P}(x,w|d)$, and $\mathbb{P}(x,w|v)$.

The goal of this project is to propose dose response and heterogeneous treatment effect estimators to ultimately inform policy and medical decisions. For this reason, I prove a uniform guarantee that strictly controls error for \textit{any} level of treatment, rather than a mean square guarantee that controls error for the \textit{average} level of treatment. Uniform guarantees come at the cost of slower rates. In negative control treatment effect estimation, the ill posedness of the confounding bridge learning problem compounds this phenomenon. Theorem~\ref{theorem:consistency_treatment} appears to be the first finite sample analysis for nonparametric negative control treatment effects, and it holds under weak assumptions. Obtaining faster rates, perhaps via further assumptions, is an important direction for future work.
\section{Simulation and application}\label{section:experiments}

\subsection{Simulations}

I evaluate the empirical performance of the new estimators. I focus on dose response with negative controls, and consider various designs with varying sample sizes. Specifically, I compare the new algorithm that uses negative controls (\verb|N.C.|) with an existing RKHS algorithm for nonparametric treatment effects (\verb|T.E.|) \cite{singh2020kernel} that ignores unobserved confounding and instead classifies negative controls as additional covariates. Whereas the new algorithm involves reweighting a confounding bridge, the previous algorithm involves reweighting a regression. For each design, sample size, and algorithm, I implement 100 simulations and calculate mean square error (MSE) with respect to the true counterfactual function. 

\begin{figure}[ht]
\begin{centering}
     \begin{subfigure}[b]{0.30\textwidth}
         \centering
         \includegraphics[width=\textwidth]{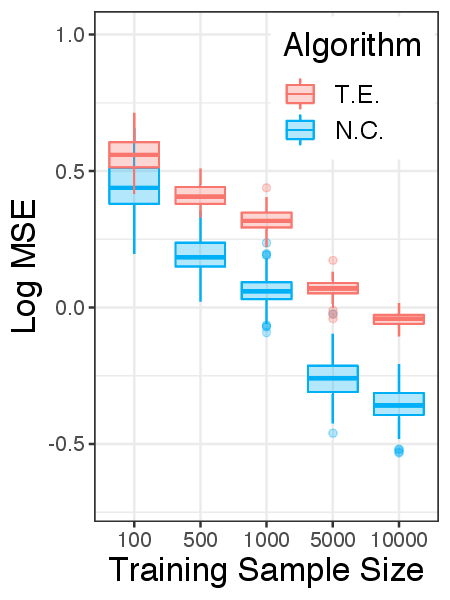}
         \vspace{-15pt}
         \caption{Quadratic}
     \end{subfigure}
     \hfill
     \begin{subfigure}[b]{0.30\textwidth}
         \centering
         \includegraphics[width=\textwidth]{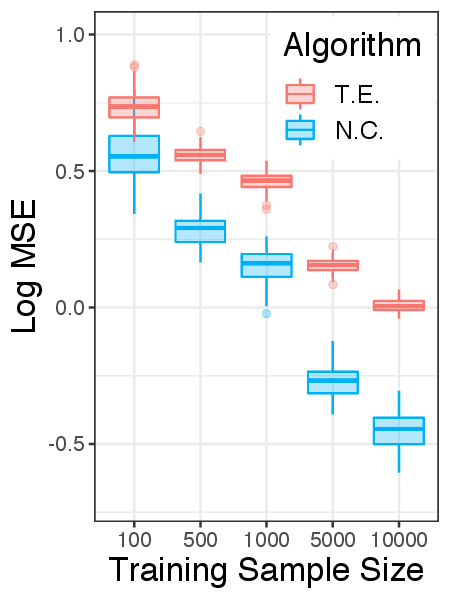}
         \vspace{-15pt}
         \caption{Sigmoid}
     \end{subfigure}
      \hfill
     \begin{subfigure}[b]{0.30\textwidth}
         \centering
         \includegraphics[width=\textwidth]{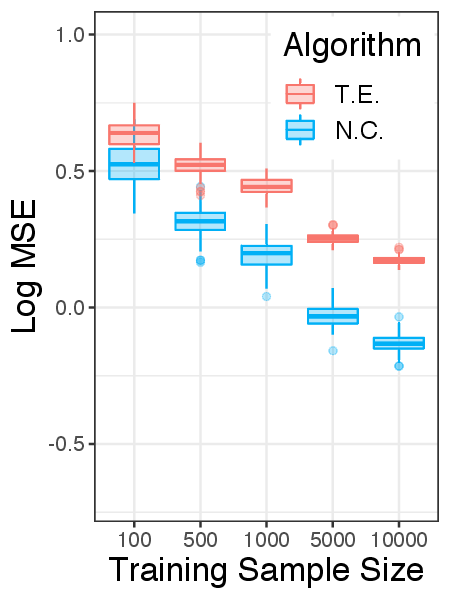}
         \vspace{-15pt}
         \caption{Peaked}
     \end{subfigure}
\par
\vspace{-5pt}
\caption{\label{fig:sim_design}
Simulation results for various designs}
\end{centering}
\end{figure}

Specifically, I adapt the continuous treatment effect design proposed by \cite{colangelo2020double}. Whereas the original setting studied by \cite{colangelo2020double} has no unobserved confounding, my modification does have unobserved confounding. The goal is to learn the counterfactual function $\theta_0^{ATE}(d)$, which may be a quadratic, sigmoid, or peaked function. A single observation consists of the tuple $(Y,W,D,Z,X)$ for outcome, negative control outcome, treatment, negative control treatment, and covariates. $(Y,D)$ are continuous scalars. In the baseline experiment, $X\in\mathbb{R}^5$ and $(Z,W)\in\mathbb{R}$. 

To explore the role of sample size, I consider $n\in\{100,500,1000,5000,10000\}$. To explore the role of dimension, I focus on the quadratic design, fix sample size at $n=1000$, and then vary $dim(X)\in\{1,5,10,50,100\}$, $dim(Z)\in\{1,5,10\}$, or $dim(W)\in\{1,5,10\}$. This range of sample sizes and dimensions is common in epidemiology research. Figures~\ref{fig:sim_design} and~\ref{fig:sim_dimension} visualize results. Across designs, sample sizes, and dimensions, the use of negative controls to adjust for unobserved confounding improves performance. The improvement is generally increasing in $n$ and $dim(Z)$ but decreasing in $dim(X)$ and $dim(W)$. Intuitively, $(X,W)$ are the variables used in the reweighting step, which is common across the two estimators \verb|N.C.| and \verb|T.E.|; as this step becomes relatively more important, the estimators become more similar. See Appendix~\ref{section:simulation_details} for implementation details as well as additional simulations that confirm: (i) robustness to tuning; (ii) improvement when treatment is discrete; and (iii) inefficiency in the absence of unobserved confounding.

\begin{figure}[ht]
\begin{centering}
     \begin{subfigure}[b]{0.3\textwidth}
         \centering
         \includegraphics[width=\textwidth]{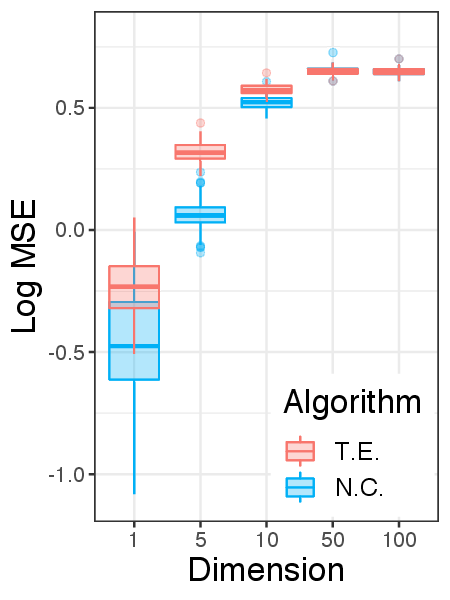}
         \vspace{-15pt}
         \caption{Covariate}
     \end{subfigure}
     \hfill
     \begin{subfigure}[b]{0.3\textwidth}
         \centering
         \includegraphics[width=\textwidth]{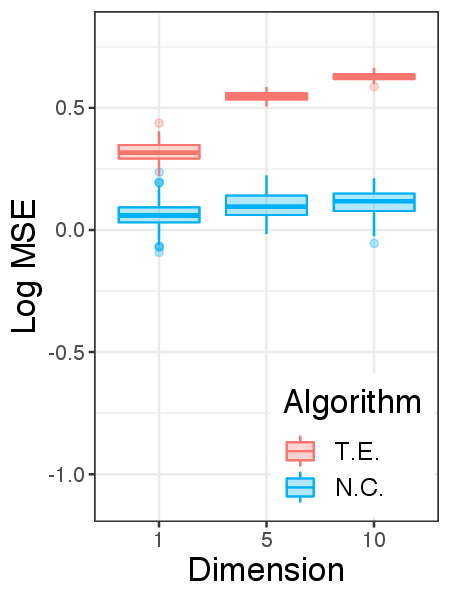}
         \vspace{-15pt}
         \caption{N.C. treatment}
     \end{subfigure}
      \hfill
     \begin{subfigure}[b]{0.3\textwidth}
         \centering
         \includegraphics[width=\textwidth]{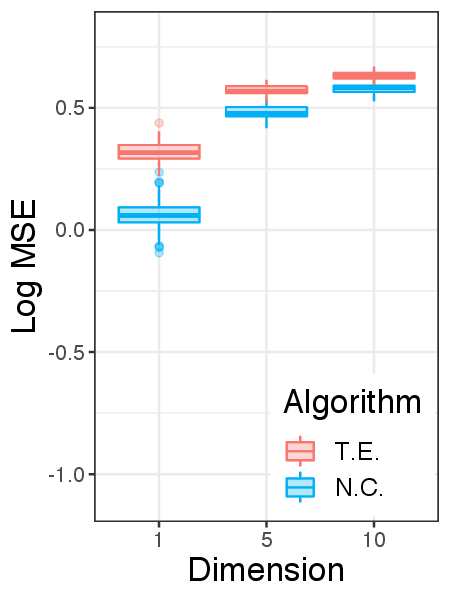}
         \vspace{-15pt}
         \caption{N.C. outcome}
     \end{subfigure}
\par
\vspace{-5pt}
\caption{\label{fig:sim_dimension}
Simulation results for various dimensions}
\end{centering}
\end{figure}

\subsection{Dose response of cigarette smoking}

Estimating the effect of cigarette smoking on infant birth weight is challenging for several reasons. First, pregnant women are classified as a vulnerable population, so they are typically excluded from clinical trials of any kind. When the treatment of interest causes harm, ethical considerations preclude randomization. Therefore \textit{observational} data are the only option. Second, pregnancy induces many physiological changes, so medical knowledge predicts different dose response curves for women who are pregnant compared to women who are not pregnant. For example, plasma volume increases 35\%, cardiac output increases 40\%, and glomerular filtration rate (a measure of kidney function) increases 50\% during pregnancy \cite{cunningham2014williams}. Therefore the shape of the dose response curve for pregnant women is an unknown, \textit{nonparametric} quantity. Third, medical records exclude an \textit{unobserved confounder} known to be crucial for maternal-fetal health: household income \cite{joseph2007socioeconomic}. 

In this section, I argue that medical records include variables that satisfy the properties of \textit{negative controls} for unobserved income. I provide preliminary results and outline directions for future work on this topic. Finally, I discuss what issues may arise if there are additional unobserved confounders. The purpose of this case study is to illustrate how the proposed estimators may be useful in epidemiology research, though the findings are not conclusive.

I estimate the dose response curve of cigarette smoking on infant birth weight using a data set of singleton births in the state of Pennsylvania between 1989 and 1991 assembled by \cite{almond2005costs} and subsequently analyzed by \cite{cattaneo2010efficient}. I focus on Pennsylvania because smoking data are available for over 95\% of mothers. I focus on singleton births because multiple gestations reflect a variety of factors and result in different fetal growth trajectories. 21\% of women report smoking during pregnancy, and I subset to this sample. I consider the subpopulations of (a) nonhispanic white women who smoke ($n=73,834$), (b) nonhispanic black women who smoke ($n=17,625$), and (c) hispanic women who smoke ($n=2,152$). Formally, I estimate $\theta_0^{CATE}(d,v)$ where $D\in\mathbb{R}$ is the number of cigarettes smoked per day, and $V$ concatenates mother's race $V_1 \in\{\text{white, black, hispanic}\}$ and mother's smoking status $V_2 \in\{0,1\}$. See Appendix~\ref{section:application_details} for further discussion.  

The classification of variables extensively relies on domain knowledge, so I sought the expertise of physicians from the Department of Obstetrics, Gynecology \& Reproductive Biology at Harvard Medical School. Together, we arrived at the classification given in Appendix~\ref{section:application_details}, based on a canonical textbook \cite{cunningham2014williams}. Figure~\ref{dag:nc_smoking} illustrates the model. Demographics, alcohol consumption, prenatal care, existing medical conditions, county, and year serve as covariates $X$ since they may be associated with both smoking $D$ and birth weight $Y$.

\begin{wrapfigure}{H}{0.45\textwidth}
\vspace{-30pt}
\begin{center}
\begin{adjustbox}{width=.45\textwidth}
\begin{tikzpicture}[->,>=stealth',shorten >=1pt,auto,node distance=2.8cm,
                    semithick]
  \tikzstyle{every state}=[draw=black,text=black]

  \node[state]         (x) [fill=gray]                   {$X$};
  \node[state]         (w) [right of=x, fill=gray, label={[align=center]above: birth order,\\[-10pt]sex,\\[-10pt]Rh}]       {$W$};
  \node[state]         (z) [left of=x, fill=gray, label={[align=center]above:edu-\\[-10pt]cation}]       {$Z$};
   \node[state]         (d) [below left of=x, fill=gray, label={below:smoking}]       {$D$};
    \node[state]         (y) [below right of=x, fill=gray, label={below:birth weight}]       {$Y$};
   \node[state]         (u) [above of=x, label={income}]                  {$U$};

  \path (u) edge              node {$ $} (z)
             edge           node {$ $} (x)
             edge           node {$ $} (w)
              edge           node {$ $} (d)
             edge           node {$ $} (y)
        (x) edge              node {$ $} (w)
            edge            node {$ $} (z)
             edge           node {$ $} (d)
             edge           node {$ $} (y)
             (z) edge              node {$ $} (d)
            (w) edge              node {$ $} (y)
            (d) edge              node {$ $} (y);;
\end{tikzpicture}
\end{adjustbox}
\vspace{-25pt}
\caption{Smoking DAG}
\label{dag:nc_smoking}
\end{center}
\vspace{-20pt}
\end{wrapfigure}
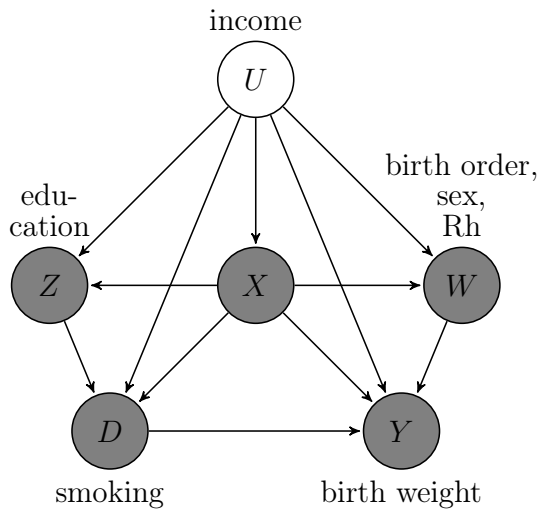

Education serves as a negative control treatment $Z$ because it reflects unobserved confounding due to household income $U$ but has no direct medical effect on birth weight $Y$. Formally, we require $Z\indep Y | D,U,X$: education is independent of birth weight after conditioning on smoking, income, and observed covariates. Prenatal care and weight gain are the observed covariates that, along with smoking and income, justify the conditional independence between education and birth weight. 

Infant birth order and sex serve as a negative control outcomes $W$ because family size reflects household income $U$ but is not directly caused by smoking $D$ or education $Z$. Formally, we require $W\indep D,Z | U,X$: family size is independent of smoking and education after conditioning on income and observed covariates. Age and marriage status are the observed covariates that, along with income, justify the conditional independence between education and family size. We also include Rh sensitization as a negative control outcome because it is one of the few medical conditions not affected by smoking (it is caused by blood type).

\begin{figure}[ht]
\begin{centering}
     \begin{subfigure}[b]{0.45\textwidth}
         \centering
         \includegraphics[width=\textwidth]{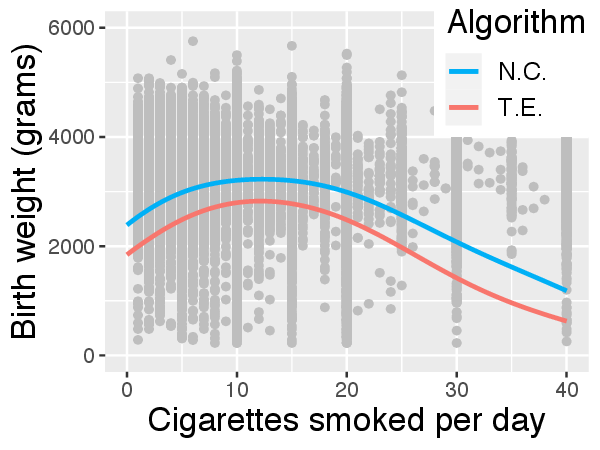}
\vspace{-25pt}
         \caption{White smoking mothers}
     \end{subfigure}
     \hfill
     \begin{subfigure}[b]{0.45\textwidth}
         \centering
         \includegraphics[width=\textwidth]{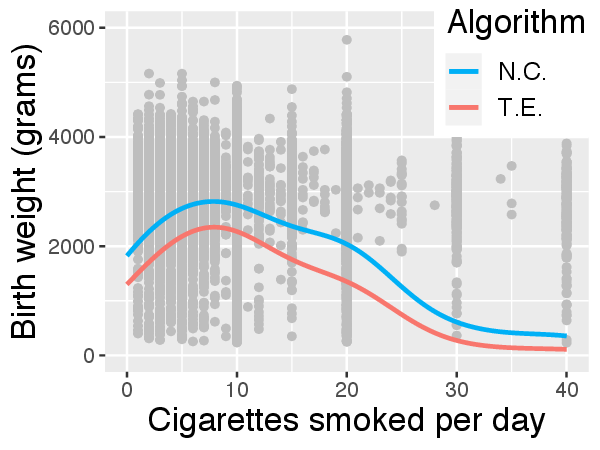}
\vspace{-25pt}
         \caption{Black smoking mothers}
     \end{subfigure}
\par
\vspace{-10pt}
\caption{\label{fig:smoke}
Effect of cigarette smoking on birth weight for different subpopulations; $\theta_0^{CATE}(d,v)$ where $D\in\mathbb{R}$ is the number of cigarettes smoked per day, and $V$ concatenates mother's race $V_1 \in\{\text{white, black, hispanic}\}$ and mother's smoking status $V_2 \in\{0,1\}$.
}
\end{centering}
\end{figure}

\begin{wrapfigure}{R}{0.45\textwidth}
\vspace{-10pt}
\begin{center}
\begin{subfigure}[b]{0.45\textwidth}
\setcounter{subfigure}{2}
         \centering
         \includegraphics[width=\textwidth]{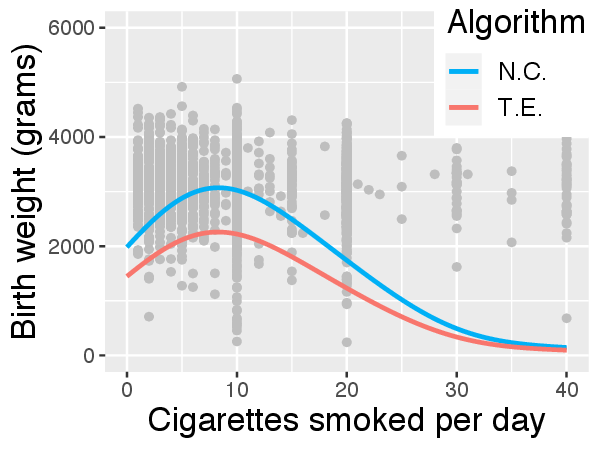}
\vspace{-25pt}
         \caption{Hispanic smoking mothers}
     \end{subfigure}
\vspace{-25pt}
\caption{Effect of cigarette smoking on birth weight for different subpopulations; $\theta_0^{CATE}(d,v)$ where $D\in\mathbb{R}$ is the number of cigarettes smoked per day, and $V$ concatenates mother's race $V_1 \in\{\text{white, black, hispanic}\}$ and mother's smoking status $V_2 \in\{0,1\}$.}
\label{fig:smoke2}
\end{center}
\vspace{-20pt}
\end{wrapfigure}

I implement both the new algorithm (\verb|N.C.|) and an existing RKHS algorithm for continuous treatment effect (\verb|T.E|) \cite{singh2020kernel} that ignores unobserved confounding. For the method that ignores unobserved confounding, I classify negative controls as additional covariates. Figures~\ref{fig:smoke} and~\ref{fig:smoke2} visualize results for white, black, and hispanic smoking mothers. The effect of cigarettes smoked per day $D$ on birth weight in grams $Y$ is generally negative with similar shapes across subpopulations. The counterfactual birth weights for black and hispanic mothers are lower than for white mothers when the number of cigarettes is high. The main finding is that using negative controls leads to higher dose response curves. Under the stated causal assumptions, the gap between \verb|N.C.| and \verb|T.E.| is the magnitude of unobserved confounding due to income. These preliminary results support the clinical hypothesis that poverty is an unmeasured confounder that affects infant birth weight. Unobserved poverty may substantially mislead observational studies that fail to account for it. See Appendix~\ref{section:application_details} for implementation details.

An unanticipated result is that the dose response curves appear nonmonotonic; estimated counterfactual birth weight increases before it decreases. This phenomenon prevails across subpopulations, and it can be seen in not only \verb|N.C.| but also \verb|T.E.| and the raw data. We propose two conjectures based on the data and domain knowledge. Both of these conjectures are ways in which the data generating process may violate the causal assumptions in Assumption~\ref{assumption:negative}.

First, it may be that measurement error contaminates observations. In the raw data, it appears that when the number of cigarettes was between one and 10 it may have been rounded up to 10. Indeed, Figures~\ref{fig:smoke} and~\ref{fig:smoke2} document substantial point masses at multiples of 10. This phenomenon would violate our causal model, since it would mean that when the true treatment value $d$ was less than 10, we observe $D=10$, $Z=z$, yet $Y=Y^{(d,z)}$. In such case, estimates of the dose response for $d<10$ may be unreliable. How to account for measurement error in negative control estimation remains an open question.

Second, it could be that another unobserved confounder exists, is not detected by the negative controls, and disproportionately affects women who reported smoking less than 10 cigarettes. Previous studies suggest that rural-urban classification, poverty, and psychosocial stress are possible confounders \cite{hobel2008psychosocial}. In our analysis, we account for rural-urban classification as an observed covariate and we account for poverty via negative controls, but we did not find plausible negative controls for stress in this data set; see Appendix~\ref{section:application_details} for further discussion. Indeed, psychosocial stress is notoriously difficult to measure, and it may cause both smoking and low birth weight. 
We pose for future work a further analysis that adjusts for unobserved confounding due to both income and stress.

\section{Conclusion}\label{sec:conclusion}

I propose a new family of nonparametric algorithms for learning treatment effects with negative controls. The estimators are easily implemented and uniformly consistent. As a contribution to the negative control literature, I propose methods to estimate dose response curves and heterogeneous treatment effects under the assumption that treatment effects are smooth. As a contribution to the kernel methods literature, I show how the RKHS is well suited to causal inference in the presence of unobserved confounding. As a contribution to maternal-fetal medicine, I propose a toolkit for estimating dose response curves for pregnant women from medical records despite unobserved confounding. The results suggest that RKHS methods may be an effective bridge between epidemiology and machine learning.

\bibliographystyle{apalike}

\newpage

\appendix

\section{Semiparametric inference}\label{sec:semi}

While the main contribution of this paper is uniform consistency in nonparametric settings, in this section I present complementary results for semiparametric settings. In particular, I articulate conditions that imply Gaussian approximation and valid confidence intervals for the setting with binary treatment, appealing to the semiparametric theory of \cite{cui2020semiparametric,kallus2021causal,ghassami2021minimax,chernozhukov2021simple}.

To lighten notation, fix the treatment value $d'\in\{0,1\}$ and denote $\theta_0=\theta_0^{ATE}(d')$. Observe that for each treatment value,
$$
\theta_0=\int h_0(d',x,w)\mathrm{d}\mathbb{P}(x,w)=\int h_0(d,x,w) \tau_0(d,x,w) \mathrm{d}\mathbb{P}(d,x,w)
$$
where by standard propensity score arguments
$$
\tau_0(d,x,w)=\frac{1\{d=d'\}}{\pi_0(x,w)},\quad \pi_0(x,w)=\mathbb{P}(D=d'|X=x,W=w).
$$
Just as the primary confounding bridge $h_0$ is defined as the solution to the operator equation
$$\gamma_0(d,x,z)=\mathbb{E}[h(D,X,W)|D=d,X=x,Z=z]$$
one may define a secondary confounding bridge $\alpha_0$ as the solution to the operator equation
$$\tau_0(d,x,w)=\mathbb{E}[\alpha(D,X,Z)|D=d,X=x,W=w].$$

\begin{proposition}[Secondary confounding bridge]\label{prop:secondary}
Suppose Assumptions~\ref{assumption:negative} and~\ref{assumption:solution} hold. In addition, suppose there exists a solution $\alpha_0$ to the operator equation
$$
\tau_0(d,x,w)=\mathbb{E}[\alpha(D,X,Z)|D=d,X=x,W=w].
$$
Then
$$
\theta_0=\int y \alpha_0(d,x,z)\mathrm{d}\mathbb{P}(y,d,x,z).
$$
\end{proposition}

\begin{proof}
The result follows from the law of iterated expectations and the definitions of the primary and secondary confounding bridges. Write
\begin{align*}
    \theta_0
    &=\int h_0(d',x,w)\mathrm{d}\mathbb{P}(x,w)\\
    &=\int h_0(d,x,w) \tau_0(d,x,w) \mathrm{d}\mathbb{P}(d,x,w) \\
    &=\int h_0(d,x,w) \mathbb{E}[\alpha_0(D,X,Z)|D=d,X=x,W=w] \mathrm{d}\mathbb{P}(d,x,w) \\
    &=\int h_0(d,x,w) \alpha_0(d,x,z) \mathrm{d}\mathbb{P}(d,z,x,w) \\
    &=\int \mathbb{E}[h_0(D,X,W)|D=d,X=x,Z=z] \alpha_0(d,x,z) \mathrm{d}\mathbb{P}(d,x,z) \\
    &=\int \gamma_0(d,x,z) \alpha_0(d,x,z) \mathrm{d}\mathbb{P}(d,x,z) \\
    &=\int y \alpha_0(d,x,z) \mathrm{d}\mathbb{P}(y,d,z,x).
\end{align*}
\end{proof}

The secondary confounding bridge formulation in Proposition~\ref{prop:secondary} generalizes the familiar inverse propensity weight formulation for treatment effects. Instead of using the propensity score $\pi_0$ encoded by $\tau_0$, the formulation uses a secondary confounding bridge $\alpha_0$ defined as the solution to an ill posed inverse problem that involves the propensity score. Analogously, the primary confounding bridge formulation in Theorem~\ref{theorem:id_treatment} generalizes the familiar g-formula for treatment effects, using the confounding bridge $h_0$ rather than the regression $\gamma_0$.

Both the primary and secondary confounding bridges $(h_0,\alpha_0)$ appear in the semiparametrically efficient asymptotic variance, so it is natural to incorporate them into semiparametric estimation and inference \cite{cui2020semiparametric}. The targeted and debiased machine learning literatures provide a meta algorithm to do so, as well as rate conditions that suffice for Gaussian approximation. I quote the meta algorithm and rate conditions, adapting them to the treatment effect $\theta_0$ identified by negative controls.

\begin{algorithm}[Debiased machine learning]\label{algorithm:DML}
Given a sample $(Y_i,W_i,D_i,X_i,Z_i)$ $(i=1,...,n)$, partition the sample into folds $(I_{\ell})$ $(\ell=1,...,L)$. Denote by $I_{\ell}^c$ the complement of $I_{\ell}$.
\begin{enumerate}
    \item For each fold $\ell$, estimate $\hat{h}_{\ell}$ and $\hat{\alpha}_{\ell}$ from observations in $I_{\ell}^c$.
    \item Estimate
    $
  \hat{\theta}=n^{-1}\sum_{\ell=1}^L\sum_{i\in I_{\ell}} [\hat{h}_{\ell}(d,X_i,W_i)+\hat{\alpha}_{\ell}(D_i,X_i,Z_i)\{Y_i-\hat{h}_{\ell}(D_i,X_i,W_i)\}]
    $.
    \item Estimate its $95$\% confidence interval as
    $
    \hat{\theta}\pm 1.96\hat{\sigma} n^{-1/2}$, where $$\hat{\sigma}^2=n^{-1}\sum_{\ell=1}^L\sum_{i\in I_{\ell}} [\hat{h}_{\ell}(d,X_i,W_i)+\hat{\alpha}_{\ell}(D_i,X_i,Z_i)\{Y_i-\hat{h}_{\ell}(D_i,X_i,W_i)\}-\hat{\theta}]^2.
    $$
\end{enumerate}
\end{algorithm}

Towards a formal statement of the inference result, define the following moments:
$$
\sigma^2=\mathbb{E}[\psi_0(Y,W,D,Z,X)^2],\quad \chi^3=\mathbb{E}[|\psi_0(Y,W,D,Z,X)|^3],\quad \omega^4=\mathbb{E}[\psi_0(Y,W,D,Z,X)^4],
$$
where
$$
\psi_0(Y,W,D,Z,X)=h_0(d,X,W)+\alpha_0(D,X,Z)\{Y-h_0(D,X,W)\}-\theta_0
$$
is the asymptotic influence of each observation. Next, I define mean square error.

\begin{definition}[Mean square error]
Write the mean square error $\mathcal{R}(\hat{h}_{\ell})$ and the projected mean square error $\mathcal{P}(\hat{h}_{\ell})$ of $\hat{h}_{\ell}$ trained on observations indexed by $I^c_{\ell}$ as \begin{align*}
      \mathcal{R}(\hat{h}_{\ell})&=\mathbb{E}[\{\hat{h}_{\ell}(D,X,W)-h_0(D,X,W)\}^2\mid I^c_{\ell}]\\
     \mathcal{P}(\hat{h}_{\ell})&=\mathbb{E}([ \mathbb{E}\{\hat{h}_{\ell}(D,X,W)-h_0(D,X,W)\mid D,X,Z, I^c_{\ell}\} ]^2\mid I^c_{\ell}).
\end{align*}
Likewise define $\mathcal{R}(\hat{\alpha}_{\ell})$ and $\mathcal{P}(\hat{\alpha}_{\ell})$:
\begin{align*}
      \mathcal{R}(\hat{\alpha}_{\ell})&=\mathbb{E}[\{\hat{\alpha}_{\ell}(D,X,Z)-\alpha_0(D,X,Z)\}^2\mid I^c_{\ell}]\\
     \mathcal{P}(\hat{\alpha}_{\ell})&=\mathbb{E}([ \mathbb{E}\{\hat{\alpha}_{\ell}(D,X,Z)-\alpha_0(D,X,Z)\mid D,X,W, I^c_{\ell}\} ]^2\mid I^c_{\ell}).
\end{align*}
\end{definition}
Sufficiently fast rates of mean square error and projected mean square error imply Gaussian approximation. The following lemma summarizes the rate conditions.

\begin{lemma}[Semiparametric inference; Corollary 5.1 of \cite{chernozhukov2021simple}]\label{lemma:inference}
Assume the propensity score $\pi_0(x,w)$ is bounded away from zero and one and the following regularity conditions hold for some absolute constant $C<\infty$:
$$
\mathbb{E}[\{Y-h_0(D,X,W)\}^2 \mid D,X,W]\leq C,\quad \|\alpha_0\|_{\infty}\leq C,\quad  \|\hat{\alpha}_{\ell}\|_{\infty}\leq C.
$$  
Further assume the following learning rate conditions, as $n\rightarrow \infty$
\begin{enumerate}
\item $\left\{\left(\chi/\sigma\right)^3+\omega^2\right\}n^{-1/2}=o(1)$;
    \item $\{\mathcal{R}(\hat{h}_{\ell})\}^{1/2}=o_p(1)$ and $\{\mathcal{R}(\hat{\alpha}_{\ell})\}^{1/2}=o_p(1)$;
    \item $\{n \mathcal{R}(\hat{h}_{\ell}) \mathcal{R}(\hat{\alpha}_{\ell})\}^{1/2} \wedge \{n\mathcal{P}(\hat{h}_{\ell})\mathcal{R}(\hat{\alpha}_{\ell})\}^{1/2} \wedge \{n\mathcal{R}(\hat{h}_{\ell})\mathcal{P}(\hat{\alpha}_{\ell})\}^{1/2} =o_p(1)$.
\end{enumerate}
Then the estimator $\hat{\theta}$ in Algorithm~\ref{algorithm:DML} is consistent and asymptotically Gaussian, and the confidence interval in Algorithm~\ref{algorithm:DML} includes $\theta_0$ with probability approaching the nominal level. Formally,
$$
\hat{\theta}\overset{p}{\rightarrow}\theta_0,\quad \sigma^{-1}n^{1/2}(\hat{\theta}-\theta_0)\overset{d}{\rightarrow}\mathcal{N}(0,1),\quad \mathbb{P} \left\{\theta_0 \in  \left(\hat{\theta}\pm 1.96\hat{\sigma} n^{-1/2} \right)\right\}\rightarrow 0.95.
$$
\end{lemma}

In summary, the algorithmic techniques developed in this paper may be combined with semiparametric theory as long as the rate conditions in Lemma~\ref{lemma:inference} are satisfied. Theorem~\ref{theorem:consistency_bridge} in the main text provides a $\sup$ norm guarantee for the primary confounding bridge $\hat{h}$, which implies slow rates of $\mathcal{P}(\hat{h}_{\ell})$ and $\mathcal{R}(\hat{h}_{\ell})$. \cite[Theorem 4]{singh2019kernel} directly provides fast rates of $\mathcal{P}(\hat{h}_{\ell})$. \cite[Theorems 4 and 8]{kallus2021causal} and \cite[Theorem 5 and Lemma 1]{ghassami2021minimax} provide fast rates of $\mathcal{P}(\hat{\alpha}_{\ell})$ and slow rates of $\mathcal{R}(\hat{\alpha}_{\ell})$ for minimax kernel estimators of the secondary confounding bridge. An interesting direction for future work would be to develop a kernel ridge regression estimator for the secondary confounding bridge.
\section{Relevance and existence}\label{section:existence}

In this appendix, I revisit the high level conditions in Assumption~\ref{assumption:solution}. These conditions are standard in the negative control and instrumental variable literatures. I begin by illustrating how irrelevance of negative controls violates existence. Then I characterize existence and completeness in the RKHS. This characterization appears to be absent from previous work on nonparametric instrumental variable regression in the RKHS.

\subsection{Relevance}

Existence of the confounding bridge is fundamentally connected to the relevance of negative controls, as articulated in Proposition~\ref{prop:relevance}. The heart of the argument is as follows.

\begin{lemma}[Rephrasing existence]\label{lemma:exist}
Suppose Assumption~\ref{assumption:negative} holds. The existence condition in Assumption~\ref{assumption:solution} holds if and only if there exists a solution $h_0$ to the operator equation
$$
\gamma_0(d,x,z)=\int h_0(d,x,w)\mathrm{d}\mathbb{P}(w|u,x)\mathrm{d}\mathbb{P}(u|d,x,z).
$$
\end{lemma}

\begin{proof}
As in \cite[Theorem 1]{miao2018identifying}, Assumption~\ref{assumption:negative} implies
\begin{align*}
    &W\indep D,Z |U,X.
\end{align*}
Hence
\begin{align*}
    \mathbb{P}(w|d,x,z)&= \int \mathbb{P}(w,u|d,x,z)\mathrm{d}u \\
    &=\int \mathbb{P}(w|u,d,x,z)\mathrm{d}\mathbb{P}(u|d,x,z) \\
    &=\int \mathbb{P}(w|u,x)\mathrm{d}\mathbb{P}(u|d,x,z).
\end{align*}
Therefore
\begin{align*}
    \int h_0(d,x,w)\mathrm{d}\mathbb{P}(w|d,x,z) 
    &=\int h_0(d,x,w)\mathrm{d}\mathbb{P}(w|u,x)\mathrm{d}\mathbb{P}(u|d,x,z).
\end{align*}
\end{proof}

For interpretation, consider the special case where $(U,Z,W)$ are discrete with finite supports $(\mathcal{U},\mathcal{Z},\mathcal{W})$, respectively. Then for any fixed $(\bar{d},\bar{x})$, the following representations are possible \cite{miao2018identifying,shi2020multiply}:
$$
\gamma_0(\bar{d},\bar{x},z)\in\mathbb{R}^{1\times |\mathcal{Z}|},\quad h_0(\bar{d},\bar{x},w) \in \mathbb{R}^{1\times |\mathcal{W}|},\quad \mathbb{P}(w|u,\bar{x}) \in \mathbb{R}^{|\mathcal{W}|\times |\mathcal{U}|},\quad \mathbb{P}(u|\bar{d},\bar{x},z) \in \mathbb{R}^{|\mathcal{U}|\times |\mathcal{Z}|}.
$$
For this special case, it is clear from the expression in Lemma~\ref{lemma:exist}
$$
\gamma_0(\bar{d},\bar{x},z)=\int h_0(\bar{d},\bar{x},w)\mathrm{d}\mathbb{P}(w|u,\bar{x})\mathrm{d}\mathbb{P}(u|\bar{d},\bar{x},z)
$$
that there exists a solution $h_0$ when $|\mathcal{W}|\geq |\mathcal{U}|$, $|\mathcal{Z}|\geq |\mathcal{U}|$, and both matrices $\mathbb{P}(w|u,\bar{x})$ and $\mathbb{P}(u|\bar{d},\bar{x},z)$ are full rank for any fixed values $(\bar{d},\bar{x})$ \cite[Lemma 1]{shi2020multiply}. $|\mathcal{W}|\geq |\mathcal{U}|$ and $|\mathcal{Z}|\geq |\mathcal{U}|$ mean that the negative controls are expressive enough relative to the unobserved confounder. Full rank $\mathbb{P}(w|u,\bar{x})$ and $\mathbb{P}(u|\bar{d},\bar{x},z)$ mean that the conditional distributions are non-degenerate; variation in the negative controls is relevant for recovering variation in the unobserved confounder. When $|\mathcal{W}|> |\mathcal{U}|$ or $|\mathcal{Z}|>|\mathcal{U}|$, the solution $h_0$ is not unique. Nonetheless, the completeness condition ensures that treatment effects are point identified.\footnote{In particular, completeness is deployed in the proof of Theorem~\ref{theorem:id_treatment} to argue that, despite the fact that $h_0$ may be non-unique, the partial mean $\int h_0(d,x,w)\mathrm{d}\mathbb{P}(w|u,x)$ is unique.} 

Next, I interpret the special case in which $(U,Z,W)$ are continuous with supports $(\mathcal{U},\mathcal{Z},\mathcal{W})$, respectively. I uncover the sense in which the existence assumption may be quite stringent. As before, fix $(\bar{d},\bar{x})$. Now, fix $(z,z')$ such that $\gamma_0(\bar{d},\bar{x},z)\neq \gamma_0(\bar{d},\bar{x},z')$. If, for these choices, $\mathbb{P}(u|\bar{d},\bar{x},z)=\mathbb{P}(u|\bar{d},\bar{x},z')$, then by Lemma~\ref{lemma:exist} the confounding bridge does not exist. In other words, to violate Assumption~\ref{assumption:solution}, there simply needs to exist some $(\bar{d},\bar{x})$ stratum such that $\gamma_0$ takes on different values at $z$ versus $z'$ yet the conditional densities of unobserved confounding coincide. 

Finally, I prove the result given in the main text.
\begin{proof}[Proof of Proposition~\ref{prop:relevance}]
With Lemma~\ref{lemma:exist}, I prove each claim separately.
\begin{enumerate}
    \item $Z\indep U |D,X$ means that $\mathbb{P}(u|d,x,z)=\mathbb{P}(u|d,x)$. Towards a contradiction, suppose $h_0$ exists. Then by Lemma~\ref{lemma:exist},
    $$
    \gamma_0(d,x,z)
    =\int h_0(d,x,w)\mathrm{d}\mathbb{P}(w|u,x)\mathrm{d}\mathbb{P}(u|d,x).
    $$
    The RHS does not depend on $z$, implying $\gamma_0(d,x,z)=\gamma_0(d,x,z')$, which violates the hypothesis that $\gamma_0$ varies in $z$.
     \item $W\indep U | X$ means that $\mathbb{P}(w|u,x)=\mathbb{P}(w|x)$. Towards a contradiction, suppose $h_0$ exists. Then by Lemma~\ref{lemma:exist},
    $$
    \gamma_0(d,x,z)
    =\int h_0(d,x,w)\mathrm{d}\mathbb{P}(w|x)\mathrm{d}\mathbb{P}(u|d,x,z)=\int h_0(d,x,w)\mathrm{d}\mathbb{P}(w|x).
    $$
    The RHS does not depend on $z$, implying $\gamma_0(d,x,z)=\gamma_0(d,x,z')$, which violates the hypothesis that $\gamma_0$ varies in $z$.
\end{enumerate}
\end{proof}

\subsection{Existence}

Next I revisit the existence condition. In particular, I present (i) a lemma to characterize existence in ill posed inverse problems; (ii) a proposition verifying existence in the nonparametric negative control problem; and (iii) a proposition verifying existence in the RKHS negative control problem.

\subsubsection{Picard's criterion}

\begin{lemma}[Picard's criterion; Theorem 15.18 of \cite{kress1989linear}]\label{lemma:picard}
Let $K:H_1\rightarrow H_2$ be a compact operator from the Hilbert space $H_1$ to the Hilbert space $H_2$, with singular value decomposition $(e^1_k,\eta_k,e_k^2)^{\infty}_{k=1}$. Given the function $g\in H_2$, the equation $K f=g$ has a solution if and only if
\begin{enumerate}
    \item Inclusion: $g\in \mathcal{N}(K^*)^{\perp}$ i.e. $g$ is an element of the orthogonal complement to the null space of the adjoint operator $K^*$;
    \item Penalized square summability: $\sum_{k=1}^{\infty} \eta_k^{-2}\langle g, e_k^2 \rangle_{H_2}<\infty$.
\end{enumerate}
\end{lemma}
Note that $e_k^1\in H_1$, $e_k^2 \in H_2$, and $K^*:H_2\rightarrow H_1$ by construction. Various works use this technical lemma to prove existence of the confounding bridge \cite{miao2018identifying,deaner2018nonparametric}. Inclusion essentially means that $g$ is in the appropriate row space. Penalized square summability means that $K^-g$ has a finite norm since $g$ is not too aligned with the right singular functions of $K$ relative to the spectral decay.

In what follows, I provide conditions in $\mathbb{L}_2$ and subsequently conditions in the RKHS that verify the abstract conditions in Lemma~\ref{lemma:picard}. The RKHS conditions are stronger versions of the $\mathbb{L}_2$ conditions, analogous to how the RKHS is a subset of $\mathbb{L}_2$ as we have seen in Section~\ref{sec:problem}.

\subsubsection{Previous work}

I quote a proposition verifying existence of the confounding bridge in the negative control problem. To begin, fix the values $(\bar{d},\bar{x})$. Given $(\bar{d},\bar{x})$, denote by $\mathbb{L}_2(\mathbb{P}(w|\bar{d},\bar{x}))$ the space of functions $g:\mathcal{W}\rightarrow \mathbb{R}$ that are square integrable with respect to the conditional distribution $\mathbb{P}(w|\bar{d},\bar{x})$, i.e. $\int g(w)^2\mathrm{d}\mathbb{P}(w|\bar{d},\bar{x})<\infty$. Likewise for $\mathbb{L}_2(\mathbb{P}(z|\bar{d},\bar{x}))$. Then define the operator
$$
E_{\bar{d},\bar{x}}:\mathbb{L}_2(\mathbb{P}(w|\bar{d},\bar{x}))\rightarrow \mathbb{L}_2(\mathbb{P}(z|\bar{d},\bar{x})),\quad f(\cdot)\mapsto \mathbb{E}[f(W)|D=\bar{d},X=\bar{x},Z=\cdot].
$$
Matching symbols with Lemma~\ref{lemma:picard}, $H_1=\mathbb{L}_2(\mathbb{P}(w|\bar{d},\bar{x}))$, $H_2=\mathbb{L}_2(\mathbb{P}(z|\bar{d},\bar{x}))$, and the adjoint operator is
$$
E^*_{\bar{d},\bar{x}}:\mathbb{L}_2(\mathbb{P}(z|\bar{d},\bar{x}))\rightarrow \mathbb{L}_2(\mathbb{P}(w|\bar{d},\bar{x})),\quad g(\cdot)\mapsto \mathbb{E}[g(Z)|D=\bar{d},X=\bar{x},W=\cdot].
$$

\begin{proposition}[Existence in $\mathbb{L}_2$ for fixed $(\bar{d},\bar{x})$; Proposition 1 of \cite{miao2018identifying}]\label{prop:exist_L2}
Fix $(\bar{d},\bar{x})$. Denote by $f(w|\bar{d},\bar{x},z)$ and $f(z|\bar{d},\bar{x},w)$ the densities of $\mathbb{P}(w|\bar{d},\bar{x},z)$ and $\mathbb{P}(z|\bar{d},\bar{x},w)$. Suppose
\begin{enumerate}
    \item Regularity: $\int f(w|\bar{d},\bar{x},z)f(z|\bar{d},\bar{x},w)\mathrm{d}w\mathrm{d}z<\infty$;
    \item Completeness: for any function $g \in \mathbb{L}_2(\mathbb{P}(z|\bar{d},\bar{x}))$,
    $$
    \mathbb{E}[g(Z)|D=\bar{d},X=\bar{x},W=w]=0\quad \forall w \iff g(Z)=0;
    $$
    \item Correct specification: $\gamma_0(\bar{d},\bar{x},\cdot)\in\mathbb{L}_2(\mathbb{P}(z|\bar{d},\bar{x}))$;
    \item Penalized square summability: $\sum_{k=1}^{\infty}\eta_k^{-2}\langle \gamma_0(\bar{d},\bar{x},\cdot), e^z_k \rangle_{\mathbb{L}_2(\mathbb{P}(z|\bar{d},\bar{x}))} <\infty$, where the singular value decomposition exists due to the regularity condition.
\end{enumerate}
Then for any fixed $(\bar{d},\bar{x})$, there exists some $f_{\bar{d},\bar{x}}\in \mathbb{L}_2(\mathbb{P}(w|\bar{d},\bar{x}))$ such that
$$
\gamma_0(\bar{d},\bar{x},z)=[E_{\bar{d},\bar{x}}f_{\bar{d},\bar{x}}](z)=\mathbb{E}[f_{\bar{d},\bar{x}}(W)|D=\bar{d},X=\bar{x},Z=z].
$$
\end{proposition}

Regularity pertains to the conditional distribution $\mathbb{P}(w|d,x,z)$ in the integral operator equation. It ensures compactness of the corresponding conditional expectation operator in order to appeal to Lemma~\ref{lemma:picard}. Completeness is a technical condition from the NPIV literature which, together with correct specification of $\gamma_0$, implies the inclusion condition of Lemma~\ref{lemma:picard}.  Penalized square summability is identical to Lemma~\ref{lemma:picard}. 

\begin{proof}
I verify the conditions of Lemma~\ref{lemma:picard}. Rewriting the regularity condition,
\begin{align*}
    \int f(w|\bar{d},\bar{x},z)f(z|\bar{d},\bar{x},w)\mathrm{d}w\mathrm{d}z
    &=\int \frac{f(z,w|\bar{d},\bar{x})^2}{f(w|\bar{d},\bar{x})f(z|\bar{d},\bar{x})}\mathrm{d}w\mathrm{d}z
\end{align*}
which is a sufficient condition for compactness of $E_{\bar{d},\bar{x}}$  \cite[Example 2.3]{carrasco2007linear}. By compactness, the singular value decomposition $(e_k^w,\eta_k,e_k^z)_{k=1}^{\infty}$ exists, where $e_k^w\in \mathbb{L}_2(\mathbb{P}(w|\bar{d},\bar{x}))$ and $e_k^z\in \mathbb{L}_2(\mathbb{P}(z|\bar{d},\bar{x}))$.

Next, I verify inclusion by appealing to completeness and correct specification. Towards this end, I first argue that $\mathcal{N}(E_{\bar{d},\bar{x}}^*)^{\perp}=\mathbb{L}_2(\mathbb{P}(z|\bar{d},\bar{x}))$. It suffices to show $\mathcal{N}(E_{\bar{d},\bar{x}}^*)=0$, i.e. the null space only consists of the zero function. Consider any $g\in \mathcal{N}(E_{\bar{d},\bar{x}}^*)$. By definition of the null space,
\begin{align*}
    0(\cdot)&=[E_{\bar{d},\bar{x}}^* g](\cdot) \\
    &=\mathbb{E}[g(Z)|D=\bar{d},X=\bar{x},W=\cdot].
\end{align*}
By completeness, I conclude that $\mathcal{N}(E_{\bar{d},\bar{x}}^*)=0$. Therefore $\mathcal{N}(E_{\bar{d},\bar{x}}^*)^{\perp}=\mathbb{L}_2(\mathbb{P}(z|\bar{d},\bar{x}))$. Finally, correct specification implies $\gamma_0(\bar{d},\bar{x},\cdot)\in \mathbb{L}_2(\mathbb{P}(z|\bar{d},\bar{x}))$.

Penalized square summability is immediate.
\end{proof}

\begin{corollary}[Existence in $\mathbb{L}_2$; Proposition 1 of \cite{miao2018identifying}]
If the conditions of Proposition~\ref{prop:exist_L2} hold for each $(\bar{d},\bar{x})$, then there exists a confounding bridge $h_0$ such that
$$
\gamma_0(d,x,z)=\mathbb{E}[h_0(D,X,W)|D=d,X=x,Z=z].
$$
\end{corollary}

\begin{proof}
For each $(\bar{d},\bar{x})$, set $h_0(\bar{d},\bar{x},w)=f_{\bar{d},\bar{x}}(w)$.
\end{proof}

\subsubsection{RKHS}

In $\mathbb{L}_2$, I introduced the function spaces $\mathbb{L}_2(\mathbb{P}(w|\bar{d},\bar{x}))$ and $\mathbb{L}_2(\mathbb{P}(z|\bar{d},\bar{x}))$ which are induced by conditioning on $(\bar{d},\bar{x})$. Now, I introduce analogous function spaces $\mathcal{H}^{\bar{d},\bar{x}}_{\mathcal{W}}$ and $\mathcal{H}^{\bar{d},\bar{x}}_{\mathcal{Z}}$. In particular, I demonstrate that these function spaces are RKHSs and characterize their kernels. For now, I simply state their kernels. Later in this section, I will demonstrate how these kernels relate to the tensor product RKHS construction.

Recall that the kernels for $\mathcal{H}_{\mathcal{W}}$ and $\mathcal{H}_{\mathcal{Z}}$ are $k(w,w')$ and $k(z,z')$, respectively. The kernels for $\mathcal{H}^{\bar{d},\bar{x}}_{\mathcal{W}}$ and $\mathcal{H}^{\bar{d},\bar{x}}_{\mathcal{Z}}$ are
$$
k_{\bar{d},\bar{x}}(w,w')=k(\bar{d},\bar{d})k(\bar{x},\bar{x})k(w,w'),\quad k_{\bar{d},\bar{x}}(z,z')=k(\bar{d},\bar{d})k(\bar{x},\bar{x})k(z,z').
$$
Define the scalar $c_{\bar{d},\bar{x}}=k(\bar{d},\bar{d})k(\bar{x},\bar{x})$. Then
$$
k_{\bar{d},\bar{x}}(w,w')=c_{\bar{d},\bar{x}}\cdot k(w,w'),\quad k_{\bar{d},\bar{x}}(z,z')=c_{\bar{d},\bar{x}}\cdot k(z,z')
$$
and it is clear that these induced kernels are simply rescaled versions of the original kernels $k(w,w')$ and $k(z,z')$ according to the conditioned value $(\bar{d},\bar{x})$, so they remain positive definite and hence valid. Then define the operator
$$
E_{\bar{d},\bar{x}}:\mathcal{H}^{\bar{d},\bar{x}}_{\mathcal{W}} \rightarrow \mathcal{H}^{\bar{d},\bar{x}}_{\mathcal{Z}} ,\quad f(\cdot)\mapsto \mathbb{E}[f(W)|D=\bar{d},X=\bar{x},Z=\cdot].
$$
Matching symbols with Lemma~\ref{lemma:picard}, $H_1=\mathcal{H}^{\bar{d},\bar{x}}_{\mathcal{W}}$, $H_2=\mathcal{H}^{\bar{d},\bar{x}}_{\mathcal{Z}} $, and the adjoint operator is
$$
E^*_{\bar{d},\bar{x}}:\mathcal{H}^{\bar{d},\bar{x}}_{\mathcal{Z}} \rightarrow \mathcal{H}^{\bar{d},\bar{x}}_{\mathcal{W}} ,\quad g(\cdot)\mapsto \mathbb{E}[g(Z)|D=\bar{d},X=\bar{x},W=\cdot].
$$

\begin{proposition}[Existence in the RKHS for fixed $(\bar{d},\bar{x})$]\label{prop:exist_RKHS}
Suppose Assumption~\ref{assumption:RKHS} holds, as well as
\begin{enumerate}
    \item Regularity: $E_{\bar{d},\bar{x}}\in\mathcal{L}_2(\mathcal{H}^{\bar{d},\bar{x}}_{\mathcal{W}},\mathcal{H}^{\bar{d},\bar{x}}_{\mathcal{Z}})$, i.e. the space of Hilbert-Schmidt operators from $\mathcal{H}^{\bar{d},\bar{x}}_{\mathcal{W}}$ to $\mathcal{H}^{\bar{d},\bar{x}}_{\mathcal{Z}}$;
    \item Completeness: $closure(span\{\mu^{\bar{d},\bar{x}}_z(\bar{d},\bar{x},w)\}_{w\in\mathcal{W}})=\mathcal{H}^{\bar{d},\bar{x}}_{\mathcal{Z}}$;
    \item Correct specification: $\gamma_0(\bar{d},\bar{x},\cdot)\in\mathcal{H}^{\bar{d},\bar{x}}_{\mathcal{Z}}$;
    \item Penalized square summability: $\sum_{k=1}^{\infty} \eta_k^{-2} \langle \gamma_0(\bar{d},\bar{x},\cdot),e_k^z \rangle^2_{\mathcal{H}^{\bar{d},\bar{x}}_{\mathcal{Z}}}<\infty$.
\end{enumerate}
Then for any fixed $(\bar{d},\bar{x})$, there exists some $f_{\bar{d},\bar{x}}\in \mathcal{H}^{\bar{d},\bar{x}}_{\mathcal{W}}$ such that
$$
\gamma_0(\bar{d},\bar{x},z)=[E_{\bar{d},\bar{x}}f_{\bar{d},\bar{x}}](z)=\mathbb{E}[f_{\bar{d},\bar{x}}(W)|D=\bar{d},X=\bar{x},Z=z].
$$
\end{proposition}

I name the conditions of Proposition~\ref{prop:exist_RKHS} to match the conditions of Proposition~\ref{prop:exist_L2}, which are standard in the literature. Regularity is again a condition on the smoothness of the distribution $\mathbb{P}(w|d,x,z)$. Completeness becomes a condition on how well the encodings of conditional distributions of $Z$ can recover the full space. Correct specification is with respect to the RKHS rather than $\mathbb{L}_2$. Penalized square summability is as before, though using the appropriate inner product.

\begin{proof}
Hilbert-Schmidt operators are compact \cite[Theorem 2.32]{carrasco2007linear}, verifying the compactness requirement. By compactness, the singular value decomposition $(e_k^w,\eta_k,e_k^z)_{k=1}^{\infty}$ exists, where $e_k^w\in \mathcal{H}^{\bar{d},\bar{x}}_{\mathcal{W}}$ and $e_k^z\in \mathcal{H}^{\bar{d},\bar{x}}_{\mathcal{Z}}$.

Next, I verify inclusion by appealing to completeness and correct specification. Towards this end, I first argue that $\mathcal{N}(E_{\bar{d},\bar{x}}^*)^{\perp}=\mathcal{H}^{\bar{d},\bar{x}}_{\mathcal{Z}}$. It suffices to show $\mathcal{N}(E_{\bar{d},\bar{x}}^*)=0$, i.e. the null space is the zero function. Consider any $g\in \mathcal{N}(E_{\bar{d},\bar{x}}^*)$. By definition of the null space,
\begin{align*}
    0(\cdot)&=[E_{\bar{d},\bar{x}}^* g](\cdot) \\
    &=\mathbb{E}[g(Z)|D=\bar{d},X=\bar{x},W=\cdot].
\end{align*}
In Assumption~\ref{assumption:RKHS}, I impose that the kernels are bounded. This assumption has several implications. First, the feature maps are Bochner integrable \cite[Definition A.5.20]{steinwart2008support}. Bochner integrability permits the exchange of expectation and inner product. Second, the mean embeddings exist. Third, the induced kernel $k_{\bar{d},\bar{x}}(z,z')$ is also bounded and hence the induced RKHS $\mathcal{H}^{\bar{d},\bar{x}}_{\mathcal{Z}}$ inherits these favorable properties. Therefore
\begin{align*}
   \mathbb{E}[g(Z)|D=\bar{d},X=\bar{x},W=\cdot]&=
    \int g(z) \mathrm{d}\mathbb{P}(z|\bar{d},\bar{x},\cdot) \\
    &= \int \langle g,\phi^{\bar{d},\bar{x}}(z) \rangle_{\mathcal{H}^{\bar{d},\bar{x}}_{\mathcal{Z}}}\mathrm{d}\mathbb{P}(z|\bar{d},\bar{x},\cdot) \\
    &=\left \langle g,\int  \phi^{\bar{d},\bar{x}}(z)\mathrm{d}\mathbb{P}(z|\bar{d},\bar{x},\cdot) \right \rangle_{\mathcal{H}^{\bar{d},\bar{x}}_{\mathcal{Z}}} \\
    &=\langle g,\mu^{\bar{d},\bar{x}}_z(\bar{d},\bar{x},\cdot) \rangle_{\mathcal{H}^{\bar{d},\bar{x}}_{\mathcal{Z}}}
\end{align*}
where $\phi^{\bar{d},\bar{x}}(z)$ is the feature map of $\mathcal{H}^{\bar{d},\bar{x}}_{\mathcal{Z}}$ and $\mu^{\bar{d},\bar{x}}_z(\bar{d},\bar{x},w)=\int  \phi^{\bar{d},\bar{x}}(w)\mathrm{d}\mathbb{P}(z|\bar{d},\bar{x},w)$ is the conditional mean embedding of $\mathbb{P}(z|\bar{d},\bar{x},w)$. In summary,
$$
 0(\cdot)=\langle g,\mu^{\bar{d},\bar{x}}_z(\bar{d},\bar{x},\cdot) \rangle_{\mathcal{H}^{\bar{d},\bar{x}}_{\mathcal{Z}}}.
$$
By hypothesis, $g\in\mathcal{H}^{\bar{d},\bar{x}}_{\mathcal{Z}}=closure(span\{\mu^{\bar{d},\bar{x}}_z(\bar{d},\bar{z},w)\}_{w\in\mathcal{W}})$. Since $\mathcal{H}^{\bar{d},\bar{x}}_{\mathcal{Z}}$ is a Hilbert space,
$$
\langle g,\mu^{\bar{d},\bar{x}}_z(\bar{d},\bar{x},w) \rangle_{\mathcal{H}^{\bar{d},\bar{x}}_{\mathcal{Z}}}=0 \quad \forall w\in\mathcal{W}\iff g=0.
$$ Combining these results, $g\in \mathcal{N}(E_{\bar{d},\bar{x}}^*)$ implies $g=0$, so $\mathcal{N}(E_{\bar{d},\bar{x}}^*)=0$ and $\mathcal{N}(E_{\bar{d},\bar{x}}^*)^{\perp}=\mathcal{H}^{\bar{d},\bar{x}}_{\mathcal{Z}}$. Finally, correct specification implies $\gamma_0(\bar{d},\bar{x},\cdot)\in \mathcal{H}^{\bar{d},\bar{x}}_{\mathcal{Z}}$.

Penalized square summability is immediate.
\end{proof}

\begin{corollary}[Existence in the RKHS]
If the conditions of Proposition~\ref{prop:exist_RKHS} hold for each $(\bar{d},\bar{x})$, then there exists a confounding bridge $h_0$ such that
$$
\gamma_0(d,x,z)=\mathbb{E}[h_0(D,X,W)|D=d,X=x,Z=z].
$$
\end{corollary}

\begin{proof}
For each $(\bar{d},\bar{x})$, set $h_0(\bar{d},\bar{x},w)=f_{\bar{d},\bar{x}}(w)$.
\end{proof}

Finally, I relate the spaces $\mathcal{H}^{\bar{d},\bar{x}}_{\mathcal{W}}$ and $\mathcal{H}_{\mathcal{Z}}^{\bar{d},\bar{x}}$ with the tensor product construction in the main text. In particular, I show how the spaces $\mathcal{H}_{\mathcal{W}}^{\bar{d},\bar{x}}$ and $\mathcal{H}_{\mathcal{Z}}^{\bar{d},\bar{x}}$ can be induced by a tensor product construction.

\begin{proposition}[Relation among kernels]\label{prop:relation}
Suppose Assumption~\ref{assumption:RKHS} holds.
\begin{enumerate}
    \item If $\gamma_0\in\mathcal{H}_{\mathcal{D}}\otimes \mathcal{H}_{\mathcal{X}}\otimes \mathcal{H}_{\mathcal{Z}}$ then $\gamma_0(\bar{d},\bar{x},\cdot)\in \mathcal{H}^{\bar{d},\bar{x}}_{\mathcal{Z}}$.
    \item If $h_0\in\mathcal{H}_{\mathcal{D}}\otimes \mathcal{H}_{\mathcal{X}}\otimes \mathcal{H}_{\mathcal{W}}$ then $h_0(\bar{d},\bar{x},\cdot)\in \mathcal{H}^{\bar{d},\bar{x}}_{\mathcal{W}}$.
    \end{enumerate}
\end{proposition}

The succinct assumption that that the reduced form belongs to a tensor product RKHS, i.e. $\gamma_0\in\mathcal{H}_{\mathcal{D}}\otimes \mathcal{H}_{\mathcal{X}}\otimes \mathcal{H}_{\mathcal{Z}}$, implies the correct specification condition in Proposition~\ref{prop:exist_RKHS} for each value $(\bar{d},\bar{x})$. In the main text, I simply assume that the confounding bridge exists and belongs to a tensor product RKHS, i.e. $h_0\in \mathcal{H}_{\mathcal{D}}\otimes \mathcal{H}_{\mathcal{X}}\otimes \mathcal{H}_{\mathcal{W}}$.  Proposition~\ref{prop:relation} demonstrates that this assumption is coherent with conditions that imply existence in Proposition~\ref{prop:exist_RKHS}. I pose as a question for future research how to characterize the conditions with which one can prove, rather than assume, that the confounding bridge is an element of an RKHS.

\begin{proof}
I prove each result separately.
\begin{enumerate}
    \item If $\gamma_0\in\mathcal{H}_{\mathcal{D}}\otimes \mathcal{H}_{\mathcal{X}}\otimes \mathcal{H}_{\mathcal{Z}}$, then
    \begin{align*}
    \gamma_0(\bar{d},\bar{x},z)
    &=\langle \gamma_0,\phi(\bar{d})\otimes \phi(\bar{x}) \otimes \phi(z)\rangle_{\mathcal{H}_{\mathcal{D}}\otimes \mathcal{H}_{\mathcal{X}}\otimes \mathcal{H}_{\mathcal{Z}}}.
    \end{align*}
   Moreover,
   \begin{align*}
     k_{\bar{d},\bar{x}}(z,z')&=k(\bar{d},\bar{d})k(\bar{x},\bar{x})k(z,z') \\
     &= \langle \phi(\bar{d})\otimes \phi(\bar{x}) \otimes \phi(z),\phi(\bar{d})\otimes \phi(\bar{x}) \otimes \phi(z') \rangle_{\mathcal{H}_{\mathcal{D}}\otimes \mathcal{H}_{\mathcal{X}}\otimes \mathcal{H}_{\mathcal{Z}}} 
   \end{align*}
   so the feature maps coincide.
    \item If $h_0\in\mathcal{H}_{\mathcal{D}}\otimes \mathcal{H}_{\mathcal{X}}\otimes \mathcal{H}_{\mathcal{W}}$, then
    \begin{align*}
    h_0(\bar{d},\bar{x},w)
    &=\langle \gamma_0,\phi(\bar{d})\otimes \phi(\bar{x}) \otimes \phi(w)\rangle_{\mathcal{H}_{\mathcal{D}}\otimes \mathcal{H}_{\mathcal{X}}\otimes \mathcal{H}_{\mathcal{W}}}.
    \end{align*}
   Moreover,
   \begin{align*}
     k_{\bar{d},\bar{x}}(w,w')&=k(\bar{d},\bar{d})k(\bar{x},\bar{x})k(w,w') \\
     &= \langle \phi(\bar{d})\otimes \phi(\bar{x}) \otimes \phi(w),\phi(\bar{d})\otimes \phi(\bar{x}) \otimes \phi(w') \rangle_{\mathcal{H}_{\mathcal{D}}\otimes \mathcal{H}_{\mathcal{X}}\otimes \mathcal{H}_{\mathcal{W}}} 
   \end{align*}
   so the feature maps coincide.
\end{enumerate}
\end{proof}

\section{Identification proof}\label{sec:id}

To lighten notation, I abbreviate conditional expectations and conditional probabilities. For example, I write
$$
\gamma_0(d,x,z)=\mathbb{E}[Y|D=d,X=x,Z=z]=\mathbb{E}[Y|d,x,z].
$$

\begin{proof}[Proof of Proposition~\ref{prop:covariate}]
$h_0$ is defined as a solution to $\gamma_0(d,x,z)=\mathbb{E}[h(D,X,W)|D=d,X=x,Z=z]$, which exists by Assumption~\ref{assumption:solution}. The reduced form $\gamma_0(d,x,z)=\mathbb{E}[Y|D=d,X=x,Z=z]$ is the same across populations since $\mathbb{P}(Y|D,X,Z)$ is invariant in Assumption~\ref{assumption:covariate}. The conditional expectation operator is also the same across populations since $\mathbb{P}(W|D,X,Z)$ is invariant in Assumption~\ref{assumption:covariate}.
\end{proof}

\begin{proposition}\label{prop:part1}
Suppose the conditions of Theorem~\ref{theorem:id_treatment} hold. Then
$$
\mathbb{E}[Y|d,u,x]=\int h_0(d,x,w)\mathrm{d}\mathbb{P}(w|u,x)
$$
For $\theta_0^{CATE}$, replace $x$ with $(v,x)$.
\end{proposition}

\begin{proof}
I generalize \cite[Theorem 1]{miao2018identifying}. By Assumption~\ref{assumption:negative}
\begin{align*}
    &Y\indep Z|U,D,X \\
    &W\indep D,Z |U,X.
\end{align*}
The former implies
\begin{align*}
    \gamma_0(d,x,z)&=\mathbb{E}[Y|d,x,z]  \\
    &=\int \mathbb{E}[Y,u|d,x,z]\mathrm{d}u \\
    &= \int \mathbb{E}[Y|u,d,x,z] \mathrm{d}\mathbb{P}(u|d,x,z) \\
    &=\int \mathbb{E}[Y|u,d,x] \mathrm{d}\mathbb{P}(u|d,x,z).
\end{align*}
The latter implies
\begin{align*}
    \mathbb{P}(w|d,x,z)&= \int \mathbb{P}(w,u|d,x,z)\mathrm{d}u \\
    &=\int \mathbb{P}(w|u,d,x,z)\mathrm{d}\mathbb{P}(u|d,x,z) \\
    &=\int \mathbb{P}(w|u,x)\mathrm{d}\mathbb{P}(u|d,x,z).
\end{align*}

Using these results and existence
\begin{align*}
   \int \mathbb{E}[Y|u,d,x] \mathrm{d}\mathbb{P}(u|d,x,z)
    &=\gamma_0(d,x,z) \\
    &=\int h_0(d,x,w)\mathrm{d}\mathbb{P}(w|d,x,z)  \\
    &=\int h_0(d,x,w)\mathrm{d}\mathbb{P}(w|u,x)\mathrm{d}\mathbb{P}(u|d,x,z).
\end{align*}
Hence by completeness
$$
\mathbb{E}[Y|u,d,x]=\int h_0(d,x,w)\mathrm{d}\mathbb{P}(w|u,x).
$$
\end{proof}

\begin{proof}[Proof of Theorem~\ref{theorem:id_treatment}]
I prove each result, appealing to Assumption~\ref{assumption:negative} and Proposition~\ref{prop:part1}.
\begin{align*}
 \theta_0^{ATE}(d)&=\mathbb{E}[Y^{(d)}]\\
 &=\int\mathbb{E}[Y^{(d)}|u,x]\mathrm{d}\mathbb{P}(u,x) \\
 &=\int\mathbb{E}[Y^{(d)}|d,u,x]\mathrm{d}\mathbb{P}(u,x) \\
 &=\int\mathbb{E}[Y|d,u,x]\mathrm{d}\mathbb{P}(u,x) \\
 &=\int h_0(d,x,w)\mathrm{d}\mathbb{P}(w|u,x)\mathrm{d}\mathbb{P}(u,x) \\
 &=\int h_0(d,x,w)\mathrm{d}\mathbb{P}(w,u,x) \\
 &=\int h_0(d,x,w)\mathrm{d}\mathbb{P}(w,x);
\end{align*}
\begin{align*}
    \theta_0^{DS}(d,\tilde{\mathbb{P}})
    &=\mathbb{E}_{\tilde{\mathbb{P}}}[Y^{(d)}]\\
    &=\int\mathbb{E}_{\tilde{\mathbb{P}}}[Y^{(d)}|u,x]\mathrm{d}\tilde{\mathbb{P}}(u,x) \\
 &=\int\mathbb{E}_{\tilde{\mathbb{P}}}[Y^{(d)}|d,u,x]\mathrm{d}\tilde{\mathbb{P}}(u,x) \\
 &=\int\mathbb{E}_{\tilde{\mathbb{P}}}[Y|d,u,x]\mathrm{d}\tilde{\mathbb{P}}(u,x) \\
 &=\int h_0(d,x,w)\mathrm{d}\tilde{\mathbb{P}}(w|u,x)\mathrm{d}\tilde{\mathbb{P}}(u,x) \\
 &=\int h_0(d,x,w)\mathrm{d}\tilde{\mathbb{P}}(w,u,x) \\
 &=\int h_0(d,x,w)\mathrm{d}\tilde{\mathbb{P}}(w,x);
\end{align*}
\begin{align*}
 \theta_0^{ATT}(d,d')&=\mathbb{E}[Y^{(d')}|D=d]\\
 &=\int\mathbb{E}[Y^{(d')}|d,u,x]\mathrm{d}\mathbb{P}(u,x|d) \\
 &=\int\mathbb{E}[Y^{(d')}|d',u,x]\mathrm{d}\mathbb{P}(u,x|d) \\
 &=\int\mathbb{E}[Y|d',u,x]\mathrm{d}\mathbb{P}(u,x|d) \\
 &=\int h_0(d',x,w)\mathrm{d}\mathbb{P}(w|u,x)\mathrm{d}\mathbb{P}(u,x|d) \\
  &=\int h_0(d',x,w)\mathrm{d}\mathbb{P}(w|d,u,x)\mathrm{d}\mathbb{P}(u,x|d) \\
 &=\int h_0(d',x,w)\mathrm{d}\mathbb{P}(w,u,x|d) \\
 &=\int h_0(d',x,w)\mathrm{d}\mathbb{P}(w,x|d);
\end{align*}
\begin{align*}
 \theta_0^{CATE}(d,v)&=\mathbb{E}[Y^{(d)}|V=v]\\
 &=\int\mathbb{E}[Y^{(d)}|u,v,x]\mathrm{d}\mathbb{P}(u,x|v) \\
 &=\int\mathbb{E}[Y^{(d)}|d,u,v,x]\mathrm{d}\mathbb{P}(u,x|v) \\
 &=\int\mathbb{E}[Y|d,u,v,x]\mathrm{d}\mathbb{P}(u,x|v) \\
 &=\int h_0(d,v,x,w)\mathrm{d}\mathbb{P}(w|u,v,x)\mathrm{d}\mathbb{P}(u,x|v) \\
 &=\int h_0(d,v,x,w)\mathrm{d}\mathbb{P}(w,u,x|v) \\
 &=\int h_0(d,v,x,w)\mathrm{d}\mathbb{P}(w,x|v).
\end{align*}
\end{proof}

\section{Discussion of the source condition}\label{sec:source}

In this appendix, I provide further discussion of the smoothness assumption that drives my analysis: the source condition. The source condition pertains to \textit{estimation} of $h_0$ rather than \textit{existence} of $h_0$. In this appendix, I discuss estimation assumptions; see Appendix~\ref{section:existence} for a thorough discussion of existence assumptions. This discussion is more intuitive than formal. See \cite{chen2011rate} for a formal discussion of source conditions in the NPIV literature.

\subsection{Source conditions in this work}

A source condition is an approximation assumption that helps to control the bias from ridge regularization. I place four source conditions in this work, parametrized by $(c,c_0,c_1,c_2)$. The source condition parametrized by $c$ quantifies the smoothness of the confounding bridge $h_0$. The source conditions parametrized by $(c_0,c_1,c_2)$ quantify the smoothness of the conditional distributions $\mathbb{P}(w|d,x,z)$, $\mathbb{P}(x,w|d)$, and $\mathbb{P}(x,w|v)$. Due to the similarity of $(c_0,c_1,c_2)$, in this discussion I focus on only $(c,c_0)$.

To understand the role of $(c,c_0)$, recall the integral equation
$$
\gamma_0(d,x,z)=\int h_0(d,x,w)\mathrm{d}\mathbb{P}(w|d,x,z).
$$
The source conditions apply to both objects on the RHS. In Section~\ref{section:algorithm}, I provide an RKHS construction to express the integral equation as
$$
\gamma_0(d,x,z)=\langle h_0,\phi(d)\otimes \phi(x)\otimes E_0^*[\phi(d)\otimes \phi(x) \otimes \phi(z)]\rangle_{\mathcal{H}}.
$$
The conditional expectation operator $E_0$ encodes the conditional distribution $\mathbb{P}(w|d,x,z)$, and I quantify its smoothness by $c_0$. I quantify the smoothness of the confounding bridge $h_0$ by $c$. Recall that $\phi(d)$, $\phi(x)$, and $\phi(z)$ can be interpreted as dictionaries of basis functions.

I propose an estimator of the confounding bridge $h_0$ that proceeds in two stages similar to two stage least squares (2SLS): (i) estimate the conditional expectation operator $E_0$ with a generalized kernel ridge regression; (ii) estimate $h_0$ with a generalized kernel ridge regression, using the estimator of $E_0$ from the first stage. See Section~\ref{section:algorithm} for discussion of the estimation procedure. Because both $(E_0,h_0)$ involve ridge regularization, I place the source conditions $(c_0,c)$ to control the regularization bias. See Appendix~\ref{section:consistency_proof1} for the formal arguments by which the source conditions imply bounds on regularization bias.

The source condition $c_0$ placed on $E_0$ means that the conditional expectation operator for $\mathbb{P}(w|d,x,z)$ has rapidly decaying Fourier coefficients as formalized in Section~\ref{sec:problem}. In particular, $c_0$ quantifies how fast the Fourier coefficients of $E_0$ decay relative to the eigenvalues of the kernel of $\mathcal{L}_2(\mathcal{H}_{\mathcal{W}},\mathcal{H}_{\mathcal{D}}\otimes \mathcal{H}_{\mathcal{X}} \otimes \mathcal{H}_{\mathcal{Z}})$. Similarly, the source condition $c$ placed on $h_0$ means that the confounding bridge has rapidly decaying Fourier coefficients as formalized in Section~\ref{sec:problem}. In particular, $c$ quantifies how fast the Fourier coefficients of $h_0$ decay relative to the eigenvalues of the kernel of $\mathcal{H}_{\mu}$.

\subsection{Source conditions in ill posed inverse problems}

As documented in the main text, the confounding bridge learning problem is precisely the nonparametric instrumental variable (NPIV) regression problem where $(D,X)$ are exogenous regressors, $W$ is an endogenous regressor, and $Z$ is the instrument. Several works in the NPIV literature deploy source conditions or similar assumptions including \cite{hall2005nonparametric,darolles2011nonparametric,singh2019kernel}. I demonstrate that the source conditions in this work generalize those of \cite{singh2019kernel} and depart from those of \cite{darolles2011nonparametric,hall2005nonparametric}.

The source conditions in this paper most closely resemble those in \cite{singh2019kernel}, who study the NPIV problem without exogenous covariates. In particular, the integral equation studied in that work is 
$$
\gamma_0(z)=\int h_0(w)\mathrm{d}\mathbb{P}(w|z)=\langle h_0,E_0^*\phi(z)\rangle_{\mathcal{H}}.
$$
As such, the source conditions in \cite{singh2019kernel} are a special case of the source conditions in this work. 

\cite{darolles2011nonparametric} also study the NPIV problem without exogenous covariates, but assume a different source condition. Whereas I assume two source conditions for $E_0$ and $h_0$, \cite{darolles2011nonparametric} assume one source condition for $E_0$ and $h_0$. Consider the NPIV problem without exogenous covariates. If $E_0$ is compact, then its singular value decomposition is $\{e_j^w,\eta_j,e_j^z\}_{j=1}^{\infty}$ where $e_j^w\in\mathbb{L}_2(\mathbb{P}(w))$ and $e_j^z \in \mathbb{L}_2(\mathbb{P}(z))$. The authors assume
$$
h_0\in\mathcal{H}^{\beta}:=\left\{f=\sum_{j=1}^{\infty} f_j e_j^z:\;\sum_{j=1}^{\infty} \frac{f_j ^2}{\eta_j^{2\beta}}<\infty,\; f_j=\langle f,e_j^w \rangle_{\mathbb{L}_2(\mathbb{P}(w))} \right\},\quad \beta\in (0,2].
$$
This source condition involves both the smoothness of $E_0$, encoded by the rate at which the singular values $\{\eta_j\}$ decay, and the smoothness of $h_0$, encoded by the rate at which the Fourier coefficients $\{f_j\}$ decay.

While \cite{hall2005nonparametric} do not explicitly place a source condition, the assumptions in that work are closely related to the source condition in \cite{darolles2011nonparametric}. Specifically, \cite{hall2005nonparametric} directly place assumptions on the rates at which  $\{\eta_j\}$ and $\{f_j\}$ and decay: $\eta_j\sim j^{-s}$ and $f_j \sim j^{-r}$. Note that if these conditions hold and $\beta < s^{-1} (r-1/2)$ then $h_0\in \mathcal{H}^{\beta}$. Moreover, $s$ essentially pins down the degree of smoothness of the joint density of $(W,Z)$.

\subsection{Alternatives to source conditions}

Two recent works prove uniform consistency of estimators in ill posed settings. Whereas this work proposes machine learning estimators with finite sample guarantees for the negative control problem, previous work proposes series estimators with asymptotic guarantees for related problems: NPIV \cite{chen2018optimal} and panel proxy control \cite{deaner2018nonparametric}. I now compare the source conditions with the alternatives assumptions in these related works.

\cite{chen2018optimal} assume that $(W,Z)$ have compact rectangular supports and bounded densities. Denoting by $J$ the dimension of the series and $\tau_J$ the series measure of ill posedness, the authors place restrictions on how $J$ evolves relative to $n$, $\tau_J$, and the norms of inverse covariances applied to the series. The main assumption is a collection of high level conditions about ill posedness, the projection of $E_0$ onto the series, and the projection of $h_0$ onto the series \cite[Assumption 4]{chen2018optimal}. The high level conditions are satisfied if the series are taken to be the singular functions $(e_j^w)$ and $(e_j^z)$ of the conditional expectation operator $E_0$. In such case, the stated conditions amount to joint assumptions on $\{f_j\}$ and $\{\eta_j\}$ as before, which we have seen to be a type of source condition.

The conditions of \cite{deaner2018nonparametric} are again in terms of series analysis. \cite{deaner2018nonparametric} assumes (i) the series approximates H\"older functions well; (ii) joint densities are bounded away from zero and one; (iii) several ratios of joint over marginal densities are H\"older smooth, as is the reduced form $\gamma_0$; (iv) the series dimension $J$ grows at a certain rate relative to $n$, the approximation quality, and the norms of inverse covariances applied to the series. The smoothness of joint over marginal density ratios is akin to assumptions on $\{\eta_j\}$. The restrictions on norms of inverse covariances applied to the series is akin to assumptions on $\{f_j\}$. In this sense, the assumptions of \cite{chen2018optimal,deaner2018nonparametric} resemble the assumptions of \cite{hall2005nonparametric}.

\section{Algorithm derivation}\label{section:alg_deriv}

\subsection{Overview}

I present an overview of the end-to-end procedure, first in principle then in practice. For simplicity, I focus on the dose response curve $\theta_0^{ATE}(d)=\int h_0(d,x,w)\mathrm{d}\mathbb{P}(x,w)$, where $h_0$ is the confounding bridge and $\mathbb{P}(x,w)$ is the counterfactual distribution with which we wish to reweight $h_0$. To further simplify the discussion, I abstract from sample splitting. By Theorem~\ref{theorem:representation_treatment},
$$
\theta_0^{ATE}(d)=\langle h_0, \phi(d)\otimes \mu \rangle_{\mathcal{H}} ,\quad \mu:=\int[\phi(x)\otimes \phi(w)]\mathrm{d}\mathbb{P}(x,w)
$$
where $\mu$ is the mean embedding of the counterfactual distribution $\mathbb{P}(x,w)$. This representation suggests an estimator of the form
$$
\hat{\theta}^{ATE}(d)=\langle \hat{h}, \phi(d)\otimes \hat{\mu} \rangle_{\mathcal{H}} ,\quad \hat{\mu}:=\int[\phi(x)\otimes \phi(w)]\mathrm{d}\hat{\mathbb{P}}(x,w)
$$
combining appropriate estimators $\hat{h}$ and $\hat{\mu}$. Estimation of $\hat{\mu}$ is simply an average: $\hat{\mu}=n^{-1}\sum_{i=1}^n \phi(x_i)\otimes \phi(w_i)$. Estimation of $\hat{h}$ is more involved. 

Towards an estimator $\hat{h}$, I recast $h_0$ as a nonparametric instrumental variable (NPIV) regression. The confounding bridge is defined as the solution to the integral equation
$$
\gamma_0(d,x,z)=\int h_0(d,x,w) \mathrm{d}\mathbb{P}(w|d,x,z),\quad \gamma_0(d,x,z):=\mathbb{E}[Y|D=d,X=x,Z=z].
$$
For expositional purposes, define the residual $\epsilon:=Y-\gamma_0(D,X,Z)$ so that
$$
Y=\int h_0(D,X,w) \mathrm{d}\mathbb{P}(w|D,X,Z)+\epsilon,\quad \mathbb{E}[\epsilon|D,X,Z]=0.
$$
Furthermore,
$$
Y=h_0(D,X,W)+\epsilon,\quad \mathbb{E}[\epsilon|D,X,Z]=0.
$$
This representation is precisely the NPIV problem where $W$ is endogenous and $Z$ is the instrument.

In this light, I propose a two stage estimator for $\hat{h}$ that generalizes two stage least squares (2SLS) and appeals to the RKHS construction. In 2SLS, an analyst would estimate $\hat{h}^{2SLS}$ by (i) projecting $W$ onto $(D,X,Z)$, obtaining an estimator $\hat{W}(D,X,Z)$; then (ii) projecting $Y$ onto $(D,X,\hat{W}(D,X,Z))$. I propose something similar. Define the mean embedding
$$
\mu_w(d,x,z):=\int \phi(w)\mathrm{d}\mathbb{P}(w|d,x,z)
$$
which encodes the distribution $\mathbb{P}(w|d,x,z)$ in the integral equation we wish to solve. I propose estimating $\hat{h}$ by (i) projecting $\phi(W)$ onto $\phi(D)\otimes \phi(X)\otimes \phi(W)$, obtaining an estimator $\hat{\mu}_w(D,X,Z)$; then (ii) projecting $Y$ onto $\phi(D)\otimes \phi(X)\otimes \hat{\mu}_w(D,X,Z)$.

\begin{algorithm}[End-to-end: Principles]\label{algorithm:abstract}
Given $n$ observations of outcome $Y$, treatment $D$, covariates $X$, negative control outcome $W$, and negative control treatment $Z$,
\begin{enumerate}
    \item Specify the kernels $k_{\mathcal{D}},k_{\mathcal{X}},k_{\mathcal{W}},k_{\mathcal{Z}}$.
    \item Specify the regularization hyperparameters $(\lambda,\xi)$.
    \item Estimate the confounding bridge $\hat{h}$ in two stages, using $(\lambda,\xi)$.
    \begin{enumerate}
        \item Estimate the distribution in the integral equation $\hat{\mathbb{P}}(w|d,x,z)$ via its mean embedding $\hat{\mu}_w(d,x,z)$ with regularization $\lambda$.
        \item Regress $Y$ onto $\phi(D)\otimes \phi(X)\otimes \hat{\mu}_w(D,X,Z)$ with regularization $\xi$.
    \end{enumerate}
    \item Estimate the counterfactual distribution $\hat{\mathbb{P}}(x,w)$ via its mean embedding $\hat{\mu}$.
    \item Estimate the dose response $\hat{\theta}^{ATE}(d)$ by combining $\hat{h}$ and $\hat{\mu}$ according to $
\hat{\theta}^{ATE}(d)=\langle \hat{h}, \phi(d)\otimes \hat{\mu} \rangle_{\mathcal{H}}.
$
\end{enumerate}
\end{algorithm}

With these principles in mind, I now fill in the details. Objects in an RKHS are infinite dimensional, and I have reasoned about them abstractly in Algorithm~\ref{algorithm:abstract}. To actually compute the estimator, I must express the procedure exclusively as inner products of RKHS objects, i.e. as scalar evaluations of kernels. I provide such details in Algorithm~\ref{algorithm:concrete_main} in the main text.

\subsection{Representation}

\begin{proof}[Proof of Theorem~\ref{theorem:representation_treatment}]
In Assumption~\ref{assumption:RKHS}, I impose that the scalar kernels are bounded. This assumption has several implications. First, the feature maps are Bochner integrable \cite[Definition A.5.20]{steinwart2008support}. Bochner integrability permits the exchange of expectation and inner product. Second, the mean embeddings exist. Third, the product kernel is also bounded and hence the tensor product RKHS inherits these favorable properties. Since $h_0\in\mathcal{H}$,
\begin{align*}
    \gamma_0(d,x,z)&=
    \int h_0(d,x,w) \mathrm{d}\mathbb{P}(w|d,x,z) \\ 
    &=\int \langle h_0,\phi(d)\otimes \phi(x) \otimes \phi(w) \rangle_{\mathcal{H}} \mathrm{d}\mathbb{P}(w|d,x,z) \\
    &=\left\langle h_0,\phi(d)\otimes \phi(x) \otimes \int \phi(w)\mathrm{d}\mathbb{P}(w|d,x,z) \right\rangle_{\mathcal{H}} \\
    &=\langle h_0,\phi(d)\otimes \phi(x) \otimes \mu_w(d,x,z) \rangle_{\mathcal{H}}.
\end{align*}
Next, I generalize \cite[Theorem 3.2]{singh2020kernel}, replacing the prediction function with the confounding bridge $h_0$. By Theorem~\ref{theorem:id_treatment} and linearity of expectation,
  \begin{align*}
    \theta_0^{ATE}(d)&= \int h_0(d,x,w)\mathrm{d}\mathbb{P}(x,w)\\
    &=\int \langle h_0, \phi(d)\otimes \phi(x)\otimes \phi(w)\rangle_{\mathcal{H}}  \mathrm{d}\mathbb{P}(x,w) \\
    &= \left\langle h_0, \phi(d)\otimes \int[\phi(x)\otimes \phi(w)] \mathrm{d}\mathbb{P}(x,w) \right\rangle_{\mathcal{H}} \\
    &= \langle h_0, \phi(d)\otimes \mu \rangle_{\mathcal{H}};
\end{align*}
   \begin{align*}
    \theta_0^{DS}(d,\tilde{\mathbb{P}})&= \int h_0(d,x,w)\mathrm{d}\tilde{\mathbb{P}}(x,w)\\
    &=\int \langle h_0, \phi(d)\otimes \phi(x)\otimes \phi(w)\rangle_{\mathcal{H}}  \mathrm{d}\tilde{\mathbb{P}}(x,w) \\
    &= \left\langle h_0, \phi(d)\otimes \int[\phi(x)\otimes \phi(w)] \mathrm{d}\tilde{\mathbb{P}}(x,w) \right\rangle_{\mathcal{H}} \\
    &= \langle h_0, \phi(d)\otimes \nu \rangle_{\mathcal{H}};
\end{align*}
 \begin{align*}
    \theta_0^{ATT}(d,d')&= \int h_0(d',x,w)\mathrm{d}\mathbb{P}(x,w|d)\\
    &=\int \langle h_0, \phi(d')\otimes \phi(x)\otimes \phi(w)\rangle_{\mathcal{H}}  \mathrm{d}\mathbb{P}(x,w|d) \\
    &= \left\langle h_0, \phi(d')\otimes \int[\phi(x)\otimes \phi(w)] \mathrm{d}\mathbb{P}(x,w|d) \right\rangle_{\mathcal{H}} \\
    &= \langle h_0, \phi(d')\otimes \mu(d) \rangle_{\mathcal{H}};
\end{align*}
\begin{align*}
    \theta_0^{CATE}(d,v)&= \int h_0(d,v,x,w)\mathrm{d}\mathbb{P}(x,w|v)\\
    &=\int \langle h_0, \phi(d)\otimes\phi(v)\otimes \phi(x)\otimes \phi(w)\rangle_{\mathcal{H}}  \mathrm{d}\mathbb{P}(x,w|v) \\
    &= \left\langle h_0, \phi(d)\otimes \phi(v)\otimes \int[\phi(x)\otimes \phi(w)] \mathrm{d}\mathbb{P}(x,w|v) \right\rangle_{\mathcal{H}} \\
    &= \langle h_0, \phi(d)\otimes \phi(v)\otimes \mu(v) \rangle_{\mathcal{H}}.
\end{align*}
\end{proof}

\subsection{Confounding bridge}

Let $n$ be the number of observations of $(d_i,x_i,w_i,z_i)$ used to estimate the conditional mean embedding $\mu_w(d,x,z)$ by kernel ridge regression with regularization parameter $\lambda$. Let $m$ be the number of observations of $(\dot{y}_i,\dot{d}_i,\dot{x}_i,\dot{z}_i)$ used to estimate the confounding bridge $h_0$ by kernel ridge regression with regularization parameter $\xi$. This notation allows the analyst to use different quantities of observations $(n,m)$ to estimate $(E_0,h_0)$, or to reuse the same observations.

\begin{proof}[Derivation of Algorithm~\ref{algorithm:bridge}]
I proceed in steps, generalizing the derivation of \cite[Algorithm 1]{singh2019kernel}.
\begin{enumerate}
    \item Closed form for stage 1.
    
    By \cite[Algorithm 1]{singh2019kernel}, the closed form solution for the stage 1 conditional mean embedding is
$$
\hat{\mu}_w(d,x,z)=\sum_{i=1}^n \beta_i(d,x,z)\phi(w_i)
$$
where
$$
\beta(d,x,z)=(K_{DD}\odot K_{XX}\odot K_{ZZ}+n\lambda I)^{-1}[K_{Dd}\odot K_{Xx} \odot K_{Zz}]\in\mathbb{R}^n.
$$
Slightly abusing notation, one may write
$$
\hat{\mu}_w(d,x,z)=K_{\cdot W} \beta(d,x,z).
$$
    
    \item Closed form for stage 2.
    
    Next, I argue that $\hat{h}=\sum_{i=1}^m \alpha_i [\phi(\dot{d}_i)\otimes \phi(\dot{x}_i)\otimes \hat{\mu}_w(\dot{d}_i,\dot{x}_i,\dot{z}_i)]$ for some $\alpha\in\mathbb{R}^m$. Write the objective as
    $$
    \hat{h}=\argmin_{h\in\mathcal{H}} \mathcal{E}_{\xi}^m(h),\quad \mathcal{E}_{\xi}^m(h)=\frac{1}{m}\sum_{i=1}^m \|\dot{y}_i-\langle h,\phi(\dot{d}_i)\otimes \phi(\dot{x}_i)\otimes \hat{\mu}_w(\dot{d}_i,\dot{x}_i,\dot{z}_i)\rangle_{\mathcal{H}}\|^2_{\mathcal{Y}}+\xi\|h\|^2_{\mathcal{H}}.
    $$
    Due to the ridge penalty, the stated objective is coercive and strongly convex with respect to $h$. Hence it has a unique minimizer that obtains the minimum.
    
    Write $\hat{h}=\hat{h}_m+\hat{h}_m^{\perp}$ where $\hat{h}_m\in span\{\phi(\dot{d}_i)\otimes \phi(\dot{x}_i)\otimes \hat{\mu}_w(\dot{d}_i,\dot{x}_i,\dot{z}_i)\}$ and $\hat{h}_m^{\perp}$ is an element of the orthogonal complement. Therefore
    $$
    \mathcal{E}_{\xi}^m(\hat{h})=\mathcal{E}_{\xi}^m(\hat{h}_m)+\xi\|\hat{h}_m^{\perp}\|^2_{\mathcal{H}}
    $$
    which implies $\mathcal{E}_{\xi}^m(\hat{h})\geq \mathcal{E}_{\xi}^m(\hat{h}_m)$. Since $\hat{h}$ is the unique minimizer, $\hat{h}=\hat{h}_m$.
    
    \item Substitution.
    
    I substitute the functional form of $\hat{h}$ into its objective. Note that
    \begin{align*}
        &\|\hat{h}\|^2_{\mathcal{H}}\\
        &=\langle \hat{h}, \hat{h}\rangle_{\mathcal{H}} \\
        &=\left\langle \sum_{i=1}^m \alpha_i [\phi(\dot{d}_i)\otimes \phi(\dot{x}_i)\otimes \hat{\mu}_w(\dot{d}_i,\dot{x}_i,\dot{z}_i)], \sum_{j=1}^m \alpha_j [\phi(\dot{d}_j)\otimes \phi(\dot{x}_j)\otimes \hat{\mu}_w(\dot{d}_j,\dot{x}_j,\dot{z}_j)]\right\rangle_{\mathcal{H}} \\
        &=\sum_{i=1}^m\sum_{j=1}^m \alpha_i \alpha_j k(\dot{d}_i,\dot{d}_j) k(\dot{x}_i,\dot{x}_j) \beta(\dot{d}_i,\dot{x}_i,\dot{z}_i)^\top K_{WW}\beta(\dot{d}_j,\dot{x}_j,\dot{z}_j) \\
        &=\alpha^{\top} M \alpha
    \end{align*}
    where in the last line I define the matrix
     $$
    M=K_{\dot{D}\dot{D}}\odot K_{\dot{X}\dot{X}}\odot \{\dot{A}^{\top}(A+n\lambda I)^{-1}K_{WW}(A+n\lambda I)^{-1} \dot{A}\}\in \mathbb{R}^{m\times m}
    $$
    with 
    $$
    A=K_{DD}\odot K_{XX}\odot K_{ZZ}\in\mathbb{R}^{n\times n},\quad \dot{A}=K_{D\dot{D}}\odot K_{X\dot{X}}\odot K_{Z\dot{Z}} \in \mathbb{R}^{n\times m}.
    $$
    Further define
    $$
    B=(A+n\lambda I)^{-1} \dot{A} \in\mathbb{R}^{n\times m}
    $$
    whose $j$-th column is 
    $
    \beta(\dot{d}_j,\dot{x}_j,\dot{z}_j).
    $ In terms of $B$, 
    $$
        M=K_{\dot{D}\dot{D}}\odot K_{\dot{X}\dot{X}}\odot \{B^{\top}K_{WW}B\}.
   $$

    Next note that
    \begin{align*}
        &\langle \hat{h},\phi(d)\otimes \phi(x) \otimes \hat{\mu}_w(d,x,z) \rangle_{\mathcal{H}} \\
        &=\left\langle \sum_{j=1}^m \alpha_j [\phi(\dot{d}_j)\otimes \phi(\dot{x}_j)\otimes \hat{\mu}_w(\dot{d}_j,\dot{x}_j,\dot{z}_j)],\phi(d)\otimes \phi(x) \otimes \hat{\mu}_w(d,x,z) \right\rangle_{\mathcal{H}}  \\
        &= \sum_{j=1}^m \alpha_j \left\langle \phi(\dot{d}_j)\otimes \phi(\dot{x}_j)\otimes \hat{\mu}_w(\dot{d}_j,\dot{x}_j,\dot{z}_j),\phi(d)\otimes \phi(x) \otimes \hat{\mu}_w(d,x,z) \right\rangle_{\mathcal{H}} \\
        &=\sum_{j=1}^m \alpha_j k(\dot{d}_j,d)k(\dot{x}_j,x)\beta(\dot{d}_j,\dot{x}_j,\dot{z}_j)^{\top} K_{WW}\beta(d,x,z).
    \end{align*}
    Hence
        \begin{align*}
        &\langle \hat{h},\phi(\dot{d}_i)\otimes \phi(\dot{x}_i) \otimes \hat{\mu}_w(\dot{d}_i,\dot{x}_i,\dot{z}_i) \rangle_{\mathcal{H}} \\
        &=\sum_{j=1}^n \alpha_j k(\dot{d}_j,\dot{d}_i)k(\dot{x}_j,\dot{x}_i)\beta(\dot{d}_j,\dot{x}_j,\dot{z}_j)^{\top} K_{WW}\beta(\dot{d}_i,\dot{x}_i,\dot{z}_i) \\
        &=\sum_{j=1}^m \alpha_j k(\dot{d}_i,\dot{d}_j)k(\dot{x}_i,\dot{x}_j)\beta(\dot{d}_i,\dot{x}_i,\dot{z}_i)^{\top} K_{WW}\beta(\dot{d}_j,\dot{x}_j,\dot{z}_j).
    \end{align*}
    In terms of $M$,
    \begin{align*}
    [M^{\top}\alpha ]_i
    &=\left[\left(K_{\dot{D}\dot{D}}\odot K_{\dot{X}\dot{X}}\odot \{B^{\top}K_{WW}B\}\right)\alpha \right]_i\\
    &=\left[\sum_{j=1}^m \alpha_j \left( K_{\dot{D}\dot{d}_j}\odot K_{\dot{X}\dot{x}_j} \odot B^{\top} K_{WW}\beta(\dot{d}_j,\dot{x}_j,\dot{z}_j)\right)\right]_i \\
    &=\sum_{j=1}^m \alpha_j k(\dot{d}_i,\dot{d}_j)k(\dot{x}_i,\dot{x}_j)\beta(\dot{d}_i,\dot{x}_i,\dot{z}_i)^{\top}K_{WW}\beta(\dot{d}_j,\dot{x}_j,\dot{z}_j) \\
         &=\langle \hat{h},\phi(\dot{d}_i)\otimes \phi(\dot{x}_i) \otimes \hat{\mu}_w(\dot{d}_i,\dot{x}_i,\dot{z}_i) \rangle_{\mathcal{H}}.
    \end{align*}

    In summary,
    $$
    \mathcal{E}_{\xi}^m(\hat{h})=\frac{1}{m}\|\dot{Y}-M^{\top}\alpha\|_2^2+\xi \alpha^{\top} M\alpha.
    $$
        \item Optimization. 
        
        Express the objective, scaled by $m$, as
        \begin{align*}
            m\mathcal{E}_{\xi}^m(\hat{h})
            &=\|\dot{Y}-M^{\top}\alpha\|_2^2+m\xi \cdot \alpha^{\top} M\alpha \\
            &= \dot{Y}^{\top}\dot{Y}-2\alpha^{\top}M \dot{Y}+\alpha^{\top}MM^{\top}\alpha+m\xi \cdot \alpha^{\top} M\alpha \\
            &=\dot{Y}^{\top}\dot{Y}-2\alpha^{\top}M \dot{Y}+\alpha^{\top}\left(MM^{\top}+m\xi M\right)\alpha
        \end{align*}
        Solving the first order condition with respect to $\alpha$,
        $$
        \hat{\alpha}=(MM^{\top}+m\xi M)^{-1}M \dot{Y}.
        $$
        Therefore
        \begin{align*}
             \hat{h}(d,x,w)
             &=\langle \hat{h} ,\phi(d)\otimes \phi(x)\otimes \phi(w)\rangle_{\mathcal{H}}\\
             &=\left\langle \sum_{i=1}^m \hat{\alpha}_i[\phi(\dot{d}_i)\otimes \phi(\dot{x}_i)\otimes \hat{\mu}_w(\dot{d}_i,\dot{x}_i,\dot{z}_i)] ,\phi(d)\otimes \phi(x)\otimes \phi(w)\right\rangle_{\mathcal{H}} \\
             &=\sum_{i=1}^m \hat{\alpha}_i \left\langle \phi(\dot{d}_i)\otimes \phi(\dot{x}_i)\otimes \hat{\mu}_w(\dot{d}_i,\dot{x}_i,\dot{z}_i) ,\phi(d)\otimes \phi(x)\otimes \phi(w)\right\rangle_{\mathcal{H}} \\
             &=\sum_{i=1}^m \hat{\alpha}_i k(\dot{d}_i,d) k(\dot{x}_i,x) \beta(\dot{d}_i,\dot{x}_i,\dot{z}_i)^{\top}K_{Ww} \\
             &=\hat{\alpha}^{\top} [K_{\dot{D}d}\odot K_{\dot{X}x} \odot \{B^{\top}K_{Ww}\}].
        \end{align*}
\end{enumerate}

\end{proof}

\subsection{Treatment effects}

\begin{proof}[Derivation of Algorithm~\ref{algorithm:treatment}]
By Algorithm~\ref{algorithm:bridge},
$$
\langle \hat{h}, \phi(d)\otimes \phi(\cdot)\otimes \phi(\cdot) \rangle_{\mathcal{H}} =\hat{\alpha}^{\top}[K_{\dot{D}d}\odot K_{\dot{X}\cdot} \odot \{B^{\top}K_{W\cdot}\}]
$$
where I leave the arguments $(x,w)$ blank in order to showcase how the confounding bridge estimator will be combined with the kernel mean embedding estimators.
By Theorem~\ref{theorem:representation_treatment}, it is sufficient to obtain expressions for kernel mean embedding estimators and to substitute them in. In particular, 
\begin{enumerate}
    \item $\theta_0^{ATE}$:
    $
    \hat{\mu}=n^{-1}\sum_{i=1}^n [\phi(x_i)\otimes \phi(w_i)]=n^{-1}\sum_{i=1}^n k(\cdot,x_i) k(\cdot,w_i);
    $
    \item $\theta_0^{DS}$:
        $
    \hat{\nu}=\tilde{n}^{-1}\sum_{i=1}^{\tilde{n}} [\phi(\tilde{x}_i)\otimes \phi(\tilde{w}_i)]=\tilde{n}^{-1}\sum_{i=1}^{\tilde{n}} k(\cdot,\tilde{x}_i)k(\cdot,\tilde{w}_i);
    $
    \item $\theta_0^{ATT}$:
    $
    \hat{\mu}(d)=[K_{\cdot X}\odot K_{\cdot W}](K_{DD}+n\lambda_1 I)^{-1}K_{Dd};
    $
    \item $\theta_0^{CATE}$: 
       $
    \hat{\mu}(v)=[K_{\cdot X}\odot K_{\cdot W}](K_{VV}+n\lambda_2 I)^{-1}K_{Vv}.
    $
\end{enumerate}
I use the $n$ observations of $(d_i,x_i,w_i,z_i)$ to estimate the kernel mean embeddings, and I do not use the $m$ observations of $(\dot{y}_i,\dot{d}_i,\dot{x}_i,\dot{z}_i)$, because the former contain the negative control outcome while the latter do not. The conditional mean embedding expressions follow from \cite[Algorithm 1]{singh2019kernel}. Matching the blank arguments yields the desired result.
\end{proof}
\section{Tuning}\label{section:tuning}

\subsection{Simplified setting}

The dose response and heterogeneous treatment effect estimators I propose are composed of kernel ridge regressions. The hyperparameters of the new estimators are therefore the same hyperparameters of kernel ridge regression. I present practical tuning procedures for the hyperparameters of (i) ridge regression penalties and (ii) the kernel itself. This appendix is an elaboration of \cite[Appendix F]{singh2020kernel}.

\subsection{Ridge penalty}

To begin, I quote a tuning procedure for kernel ridge regression. For simplicity, I focus on the regression of $Y$ on $A$. It is convenient to tune $\lambda$ by leave-one-out cross validation (LOOCV), since the validation loss has a closed form solution. 

\begin{algorithm}[Tuning kernel ridge regression; Algorithm F.1 of \cite{singh2020kernel}]\label{algorithm:tuning}
Construct the matrices
$$
H_{\lambda}:=I-K_{AA}(K_{AA}+n\lambda I)^{-1}\in\mathbb{R}^{n\times n},\quad \tilde{H}_{\lambda}:=diag(H_{\lambda})\in\mathbb{R}^{n\times n}
$$
where $\tilde{H}_{\lambda}$ has the same diagonal entries as $H_{\lambda}$ and off diagonal entries of 0. Then set
$$
\lambda^*=\argmin_{\lambda \in\Lambda} \frac{1}{n}\|\tilde{H}_{\lambda}^{-1}H_{\lambda} Y\|_2^2,\quad \Lambda\subset\mathbb{R}.
$$
\end{algorithm}
The same principles govern the tuning of conditional mean embeddings, which can be viewed as vector valued regressions. Moreover, the same principles govern the tuning of kernel instrumental variable regression, which consists of two kernel ridge regressions. The extension of these tuning procedures to kernel instrumental variable regression is an innovation, and it differs from the tuning procedure in \cite[Appendix A.5.2]{singh2019kernel}. To simplify the discussion, I focus on the conditional mean embedding $\mu(a)=\int \phi(y)\mathrm{d}\mathbb{P}(y|a)$. Recall that the closed form solution of the conditional mean embedding estimator using all observations is
$$
\hat{\mu}(a)=K_{\cdot Y}(K_{AA}+n\lambda I)^{-1}K_{Aa}.
$$

\begin{algorithm}[Tuning conditional mean embedding]\label{algorithm:tuning2}
Construct the matrices
\begin{align*}
    &R:=K_{AA}(K_{AA}+n\lambda I)^{-1}\in\mathbb{R}^{n\times n}\\
    &S\in\mathbb{R}^{n\times n} \text{ s.t. } S_{ij}=1\{i=j\} \left\{\frac{1}{1-R_{ii}}\right\}^2 \\
    &T:=S(K_{YY}-2K_{YY}R^{\top}+RK_{YY}R^{\top})\in\mathbb{R}^{n\times n}
\end{align*}
where $R_{ii}$ is the $i$-th diagonal element of $R$. Then set
$$
\lambda^*=\argmin_{\lambda \in\Lambda} \frac{1}{n}tr(T),\quad \Lambda\subset\mathbb{R}.
$$
\end{algorithm}

\begin{proof}[Derivation]
I prove that $n^{-1}tr(T)$ is the LOOCV loss. By definition, the LOOCV loss is
$$
\mathcal{E}(\lambda):=\frac{1}{n}\sum_{i=1}^n \|\phi(y_i)-\hat{\mu}_{-i}(a_i)\|_{\mathcal{H}_{\mathcal{Y}}}^2
$$
where $\hat{\mu}_{-i}$ is the conditional mean embedding estimator using all observations except the $i$-th observation.

Informally, let $\Phi$ be the matrix of features for $\{a_i\}$, with $i$-th row $\phi(a_i)^{\top}$, and let $Q:=\Phi^{\top}\Phi+n\lambda$. Let $\Psi$ be the matrix of features for $\{y_i\}$, with $i$-th row $\phi(y_i)^{\top}$. By the regression first order condition
\begin{align*}
    \hat{\mu}(a)^{\top}&=\phi(a)^{\top}Q^{-1}\Phi^{\top}\Psi  \\
    \hat{\mu}_{-i}(a)^{\top}&=\phi(a)^{\top}\{Q-\phi(a_i)\phi(a_i)^{\top}\}^{-1}\{\Phi^{\top}\Psi-\phi(a_i)\phi(y_i)^{\top}\}.
\end{align*}
The Sherman-Morrison formula for rank one updates gives
$$
\{Q-\phi(a_i)\phi(a_i)^{\top}\}^{-1}=Q^{-1}+\frac{Q^{-1}\phi(a_i)\phi(a_i)^{\top}Q^{-1}}{1-\phi(a_i)^{\top}Q^{-1}\phi(a_i)}.
$$
Let $\beta_i:=\phi(a_i)^{\top} Q^{-1} \phi(a_i)$. Then
\begin{align*}
     \hat{\mu}_{-i}(a)^{\top}&=\phi(a_i)^{\top} \left\{Q^{-1}+\frac{Q^{-1}\phi(a_i)\phi(a_i)^{\top}Q^{-1}}{1-\beta_i}\right\}\{\Phi^{\top}\Psi-\phi(a_i)\phi(y_i)^{\top}\} \\
    &=\phi(a_i)^{\top} \left\{I+\frac{Q^{-1}\phi(a_i)\phi(a_i)^{\top}}{1-\beta_i}\right\}\{\hat{\mu}^{\top}-Q^{-1}\phi(a_i)\phi(y_i)^{\top}\} \\
    &=\left\{1 +\frac{\beta_i}{1-\beta_i}\right\}\phi(a_i)^{\top}\{\hat{\mu}^{\top}-Q^{-1}\phi(a_i)\phi(y_i)^{\top}\}\\
    &=\left\{1 +\frac{\beta_i}{1-\beta_i}\right\}\{\hat{\mu}(a_i)^{\top}-\beta_i\phi(y_i)^{\top}\} \\
    &=\frac{1}{1-\beta_i}\{\hat{\mu}(a_i)^{\top}-\beta_i\phi(y_i)^{\top}\}
\end{align*}
i.e. $\hat{\mu}_{-i}$ can be expressed in terms of $\hat{\mu}$. Note that
\begin{align*}
    \phi(y_i)-\hat{\mu}_{-i}(a_i)&=\phi(y_i)-\frac{1}{1-\beta_i}\{\hat{\mu}(a_i)-\beta_i\phi(y_i)\} \\
    &=\phi(y_i)+\frac{1}{1-\beta_i}\{\beta_i\phi(y_i)-\hat{\mu}(a_i)\} \\
    &=\frac{1}{1-\beta_i}\{\phi(y_i)-\hat{\mu}(a_i)\}.
\end{align*}
Substituting back into the LOOCV loss
\begin{align*}
    \frac{1}{n}\sum_{i=1}^n \left\|\phi(y)_i-\hat{\mu}_{-i}(a_i)\right\|_{\mathcal{H}_{\mathcal{Y}}}^2 
    &=\frac{1}{n}\sum_{i=1}^n \left\|\{\phi(y_i)-\hat{\mu}(a_i)\}\left\{\frac{1}{1-\beta_i}\right\}\right\|_{\mathcal{H}_{\mathcal{Y}}}^2 \\
    &= \frac{1}{n}\sum_{i=1}^n \left\{\frac{1}{1-\beta_i}\right\}^2 \left\|\phi(y_i)-\hat{\mu}(a_i)\right\|_{\mathcal{H}_{\mathcal{Y}}}^2.
\end{align*}
By arguments in \cite[Appendix F]{singh2020kernel}, 
$$
\beta_i=[K_{AA}(K_{AA}+n\lambda I)^{-1}]_{ii}
$$
i.e. $\beta_i$ can be calculated as the $i$-th diagonal element of $K_{AA}(K_{AA}+n\lambda I)^{-1}$. Moreover
\begin{align*}
    &\|\phi(y_i)-\hat{\mu}(a_i)\|_{\mathcal{H}}^2 \\
    &=k(y_i,y_i)-2 K_{y_iY}(K_{AA}+n\lambda I)^{-1}K_{Aa_i}+K_{a_iA} (K_{AA}+n\lambda I)^{-1}K_{YY}(K_{AA}+n\lambda I)^{-1}K_{Aa_i}\\
    &=[K_{YY}-2K_{YY}(K_{AA}+n\lambda I)^{-1}K_{AA} +K_{AA} 
(K_{AA}+n\lambda I)^{-1}K_{YY}(K_{AA}+n\lambda I)^{-1}K_{AA}]_{ii}
\end{align*}
i.e. $\|\phi(y_i)-\hat{\mu}(a_i)\|_{\mathcal{H}}^2$ can be calculated as the $i$-th diagonal element of a matrix as well. Substituting these results back into the LOOCV loss gives the final expression.
\end{proof}

\subsection{Kernel}

The Gaussian kernel satisfies the requirements of Assumption~\ref{assumption:RKHS}. Formally, the kernel
$$
k(a,a')=\exp\left\{-\frac{1}{2}\frac{\|a-a'\|^2_{\mathcal{A}}}{\sigma^2}\right\}
$$ is continuous, bounded, and characteristic. The kernel hyperparameter $\sigma$ is called the lengthscale. A simple heuristic is to set $\sigma$ as the median interpoint distance of $\{a_i\}^n_{i=1}$, where the interpoint distance between observations $i$ and $j$ is $\|a_i-a_j\|_{\mathcal{A}}$.

I use the Gaussian kernel in experiments. When the input $a$ is a vector rather than a scalar, I use the kernel obtained as the product of scalar kernels for each input dimension, following \cite{singh2019kernel,singh2020kernel}. For example, if $\mathcal{A}\subset \mathbb{R}^d$ then
$$
k(a,a')=\prod_{j=1}^d \exp\left\{-\frac{1}{2}\frac{[a_j-a_j']^2}{\sigma_j^2}\right\}.
$$
Each lengthscale $\sigma_j$ is set as the median interpoint distance for that input dimension. In principle, one could instead use LOOCV to tune kernel hyperparameters as above. The LOOCV approach to tuning lengthscales $\{\sigma_j\}$ is impractical in high dimensions, since there is a lengthscale $\sigma_j$ for each input dimension.

\subsection{Time complexity}

As in classic kernel ridge regression, the time consuming step is tuning. We see in Algorithms~\ref{algorithm:tuning} and~\ref{algorithm:tuning2} that to choose the ridge penalty hyperparameter $\lambda^*$, one must invert the matrix
$$
K_{AA}+n\lambda I \in\mathbb{R}^{n\times n}
$$
for each value $\lambda$ in the grid $\Lambda$. Inversion of such a matrix has complexity $O(n^3)$; the sample size $n$ is the limiting factor. The same is true for the two ridge penalty hyperparameters in the confounding bridge estimated by Algorithm~\ref{algorithm:bridge}, and for the additional ridge penalty hyperparameter that appears when estimating the heterogeneous treatment effect in Algorithm~\ref{algorithm:treatment}. Therefore tuning of the heterogeneous treatment effect in Algorithm~\ref{algorithm:treatment} takes roughly three times as long as tuning of a kernel ridge regression, whose runtime scales as $O(n^3)$.

In the simulations of Section~\ref{section:experiments} and Appendix~\ref{section:simulation_details}, I implement kernel methods with and without negative controls, across various designs, with several iterations, which compounds the time complexity of the tuning step. It is therefore feasible to implement only 100 iterations when the sample size is $n=10,000$. 

In practice, an analyst would implement the method once for one sample size, which is feasible on a personal laptop. The dose response curve can be estimated in a matter of seconds for $n\in \{100,500\}$, in a matter of minutes for $n\in \{1000,5000\}$, and a matter of hours for $n=10,000$. 
A vast literature considers how to speed up kernel methods by replacing the kernel matrix with a low rank approximation. Appendix E.3 of \cite{dikkala2020minimax} discusses popular techniques and their implementation in NPIV. I pose as a question for future work how to extend the main results of this paper to accommodate kernel matrix approximations.

\section{Confounding bridge consistency proof}\label{section:consistency_proof1}

In this appendix, I (i) state a probability lemma, (ii) explicitly specialize the smoothness assumptions, (iii) provide regression lemmas, (iv) prove technical bounds, and (v) prove uniform consistency of the confounding bridge. This is the most technically demanding appendix.

\subsection{Probability lemma}

\begin{lemma}[Lemma 2 of \cite{smale2007learning}]\label{lemma:prob}
Let $\xi$ be a random variable taking values in a real separable Hilbert space $\mathcal{K}$. Suppose there exists $ \tilde{M}$ and $\sigma^2$ such that
\begin{align*}
    \|\xi\|_{\mathcal{K}} &\leq \tilde{M}<\infty \quad \text{ almost surely},\quad \mathbb{E}\|\xi\|_{\mathcal{K}}^2\leq \sigma^2.
\end{align*}
Then $\forall n\in\mathbb{N}, \forall \eta\in(0,1)$,
$$
\mathbb{P}\bigg[\bigg\|\dfrac{1}{n}\sum_{i=1}^n\xi_i-\mathbb{E}\xi\bigg\|_{\mathcal{K}}\leq\dfrac{2\tilde{M}\ln(2/\eta)}{n}+\sqrt{\dfrac{2\sigma^2\ln(2/\eta)}{n}}\bigg]\geq 1-\eta.
$$
\end{lemma}

\subsection{Smoothness assumptions}

Let the symbol $\circ$ mean composition. I use it to emphasize the composition of operators. To lighten notation, let $\mathcal{H}_{RF}=\mathcal{H}_{\mathcal{D}}\otimes \mathcal{H}_{\mathcal{X}}\otimes \mathcal{H}_{\mathcal{Z}}$, which is named for the reduced form.

\begin{assumption}[Smoothness of conditional expectation]\label{assumption:smooth_E}
Assume
\begin{enumerate}
    \item The conditional expectation operator $E_0$ is well specified as a Hilbert-Schmidt operator between RKHSs, i.e. $E_0\in \mathcal{L}_2(\mathcal{H}_{\mathcal{W}},\mathcal{H}_{RF})$, where
    $$
    E_0:\mathcal{H}_{\mathcal{W}} \rightarrow \mathcal{H}_{RF},\quad h(\cdot)\mapsto \mathbb{E}[h(W)|D=\cdot,X=\cdot,Z=\cdot].
    $$
    \item The conditional expectation operator is a particularly smooth element of $\mathcal{L}_2(\mathcal{H}_{\mathcal{W}},\mathcal{H}_{RF})$. Formally, define the covariance operator $T_0:=\mathbb{E}[\phi(D,X,Z)\otimes \phi(D,X,Z)]$ for $\mathcal{L}_2(\mathcal{H}_{\mathcal{W}},\mathcal{H}_{RF})$.
    I assume there exists $G_0\in \mathcal{L}_2(\mathcal{H}_{\mathcal{W}},\mathcal{H}_{RF})$ such that $E_0=(T_0)^{\frac{c_0-1}{2}}\circ G_0$, $c_0\in(1,2]$, and $\|G_0\|^2_{\mathcal{L}_2(\mathcal{H}_{\mathcal{W}},\mathcal{H}_{RF})}\leq\zeta_0$.
\end{enumerate}
\end{assumption}

\begin{assumption}[Smoothness of confounding bridge]\label{assumption:smooth_bridge_long}
Assume
\begin{enumerate}
    \item The confounding bridge $h_0$ is well specified, i.e. $h_0\in \mathcal{H}_{\mu}\subset \mathcal{H}$.
    \item The confounding bridge is a particularly smooth element of $\mathcal{H}_{\mu}$. Formally, define the covariance operator $T:=\mathbb{E}[\mu(D,X,Z) \otimes\mu(D,X,Z)]$, where $\mu(d,x,z)=\phi(d)\otimes \phi(x)\otimes \mu_w(d,x,z)$, for $\mathcal{H}_{\mu}$.
    I assume there exists $g\in \mathcal{H}$ such that $h_0=T^{\frac{c-1}{2}}\circ g$, $c\in(1,2]$, and $\|g\|^2_{\mathcal{H}}\leq\zeta$.
\end{enumerate}
\end{assumption}

\begin{proposition}
The following assumptions are equivalent
\begin{enumerate}
    \item Assumption~\ref{assumption:smooth_op} with $\mathcal{A}_0=\mathcal{W}$ and $\mathcal{B}_0=\mathcal{D}\times \mathcal{X}\times \mathcal{Z}$ is equivalent to Assumption~\ref{assumption:smooth_E}
    \item Assumption~\ref{assumption:smooth_bridge} is equivalent to Assumption~\ref{assumption:smooth_bridge_long}
\end{enumerate}
\end{proposition}

\begin{proof}
     The result is immediate from \cite[Remark 2]{caponnetto2007optimal}. The expanded expressions are more convenient for analysis. 
\end{proof}

\subsection{Regression lemmas}

Let $n$ be the number of observations of $(d_i,x_i,w_i,z_i)$ used to estimate the stage 1 conditional mean embedding $\mu(d,x,z)=\phi(d)\otimes \phi(x)\otimes \mu_w(d,x,z)$ by kernel ridge regression with regularization parameter $\lambda$. Let $m$ be the number of observations of $(\dot{y}_i,\dot{d}_i,\dot{x}_i,\dot{z}_i)$ used to estimate the stage 2 confounding bridge operator $h_0$ by kernel ridge regression with regularization parameter $\xi$.

\begin{proposition}\label{op2_sup}
Suppose Assumptions~\ref{assumption:solution} and~\ref{assumption:RKHS} hold, and $h_0\in \mathcal{H}$. Then
$$
\mathbb{E}[\mu(D,X,Z)Y]=Th_0.
$$
\end{proposition}

\begin{proof}
Appealing to the definition of $T$, the argument in the proof of Theorem~\ref{theorem:representation_treatment}, and the law of iterated expectations,
\begin{align*}
    Th_0
    &=\mathbb{E}[\mu(D,X,Z) \otimes\mu(D,X,Z)] h_0 \\
    &=\mathbb{E}[\mu(D,X,Z) \langle\mu(D,X,Z),h_0\rangle_{\mathcal{H}}] \\
    &=\mathbb{E}[\mu(D,X,Z) \gamma_0(D,X,Z)] \\
    &=\mathbb{E}[\mu(D,X,Z)Y]. 
\end{align*}
\end{proof}

To facilitate analysis, define the following quantities. 

\begin{definition}[Confounding bridge risk]
Define
\begin{enumerate}
    \item Target bridge 
    $$
    h_0\in\argmin_{h\in\mathcal{H}}\mathcal{E}(h),\quad \mathcal{E}(h) = \mathbb{E}[\{Y-\langle h,\mu(D,X,Z) \rangle_{\mathcal{H}}\}^2].
    $$
    \item Regularized bridge
    $$
    h_{\xi}=\argmin_{h\in \mathcal{H}}\mathcal{E}_{\xi}(h),\quad \mathcal{E}_{\xi}(h)=\mathcal{E}(h)+\xi\|h\|^2_{\mathcal{H}}.
    $$
    \item Empirical regularized bridge
    $$
    h^{m}_{\xi}=\argmin_{h\in \mathcal{H}}\mathcal{E}^{m}_{\xi}(h),\quad \mathcal{E}^{m}_{\xi}(h)=\dfrac{1}{m}\sum_{i=1}^{m}\{\dot{y}_i-\langle h,\mu(\dot{d}_i,\dot{x}_i,\dot{z}_i)\rangle_{\mathcal{H}}\}^2+\xi\|h\|^2_{\mathcal{H}}.
    $$
    \item Estimated bridge
    $$
    \hat{h}^{m}_{\xi}=\argmin_{h\in \mathcal{H}}\hat{\mathcal{E}}^{m}_{\xi}(h),\quad \hat{\mathcal{E}}^{m}_{\xi}(h)=\dfrac{1}{m}\sum_{i=1}^{m}\{\dot{y}_i-\langle h,\mu^n_{\lambda}(\dot{d}_i,\dot{x}_i,\dot{z}_i)\rangle_{\mathcal{H}} \}^2+\xi\|h\|^2_{\mathcal{H}}
    $$
    where $\mu^n_{\lambda}(d,x,z)$ is an estimator of the conditional mean embedding $\mu(d,x,z)$.
\end{enumerate}
\end{definition}

\begin{proposition}[Closed form]\label{sol_2}
$\forall \xi>0$, the solution $h^{m}_{\xi}$ to $\mathcal{E}_{\xi}^m$ 
and the solution $\hat{h}^{m}_{\xi}$ to $\hat{\mathcal{E}}_{\xi}^m$ both exist, are unique, and
\begin{align*}
    h_{\xi}^{m}&=(\mathbf{T}+\xi)^{-1}\mathbf{g},\quad \mathbf{T}=\dfrac{1}{m}\sum_{i=1}^m \mu(\dot{d}_i,\dot{x}_i,\dot{z}_i)\otimes \mu(\dot{d}_i,\dot{x}_i,\dot{z}_i),\quad \mathbf{g}=\dfrac{1}{m}\sum_{i=1}^m \mu(\dot{d}_i,\dot{x}_i,\dot{z}_i) \dot{y}_i, \\
    \hat{h}_{\xi}^{m}&=(\hat{\mathbf{T}}+\xi)^{-1}\hat{\mathbf{g}},\quad \hat{\mathbf{T}}=\dfrac{1}{m}\sum_{i=1}^m \mu^n_{\lambda}(\dot{d}_i,\dot{x}_i,\dot{z}_i) \otimes \mu^n_{\lambda}(\dot{d}_i,\dot{x}_i,\dot{z}_i),\quad \hat{\mathbf{g}}=\dfrac{1}{m}\sum_{i=1}^m \mu^n_{\lambda}(\dot{d}_i,\dot{x}_i,\dot{z}_i) \dot{y}_i.
\end{align*}
\end{proposition}

\begin{proof}
The result is a simplification of \cite[Theorem 3]{singh2019kernel}.
\end{proof}

\subsection{Bias and variance of second stage}

\begin{proposition}[Bias]\label{approx_sup}
Suppose Assumptions~\ref{assumption:solution}, \ref{assumption:RKHS}, \ref{assumption:original},  and \ref{assumption:smooth_bridge_long} hold. Then
$$
\|h_{\xi}-h_0\|_{\mathcal{H}}\leq \xi^{\frac{c-1}{2}} \sqrt{\zeta}.
$$
\end{proposition}

\begin{proof}
I generalize \cite[Theorem 4]{smale2005shannon}. By Assumption~\ref{assumption:smooth_bridge_long}, there exists a function $g\in\mathcal{H}$ such that
$$
    g=T^{\frac{1-c}{2}} h_0   
     =\sum_k \eta_k^{\frac{1-c}{2}} e_k\langle e_k, h_0 \rangle_{\mathcal{H}}
$$
where $\{\eta_k\}$ are the eigenvalues and $\{e_k\}$ are the eigenfunctions of $T$. By Proposition~\ref{op2_sup}, write
$$
    h_{\xi}-h_0=[(T+\xi)^{-1}T-I]  h_0 
     =\sum_k \bigg(\dfrac{\eta_k}{\eta_k+\xi}-1\bigg) e_k\langle e_k, h_0 \rangle_{\mathcal{H}}.
$$
Therefore
\begin{align*}
     \|h_{\xi}-h_0\|^2_{\mathcal{H}}&=\sum_k \bigg(\dfrac{\eta_k}{\eta_k+\xi}-1\bigg)^2 \langle e_k, h_0 \rangle^2_{\mathcal{H}} \\
     &=\sum_k \bigg(\dfrac{\xi}{\eta_k+\xi}\bigg)^2 \langle e_k, h_0 \rangle^2_{\mathcal{H}} \\
     &=\sum_k \bigg(\dfrac{\xi}{\eta_k+\xi}\bigg)^2 \langle e_k, h_0 \rangle^2_{\mathcal{H}} \bigg(\dfrac{\xi}{\xi}\cdot \dfrac{\eta_k}{\eta_k}\cdot \dfrac{\eta_k+\xi}{\eta_k+\xi} \bigg)^{c-1}\\ 
     &=\xi^{c-1} \sum_k  \eta_k^{1-c} \langle e_k, h_0 \rangle^2_{\mathcal{H}}\bigg(\dfrac{\xi}{\eta_k+\xi}\bigg)^{3-c} \bigg(\dfrac{\eta_k}{\eta_k+\xi}\bigg)^{c-1} \\
     &\leq \xi^{c-1} \sum_k  \eta_k^{1-c} \langle e_k, h_0 \rangle^2_{\mathcal{H}} \\
     &=\xi^{c-1}\|g\|^2_{\mathcal{H}} \\
     &\leq \xi^{c-1}  \zeta.
\end{align*}
\end{proof}

\begin{lemma}[Helpful bounds]\label{lemma:bounds}
Suppose Assumptions~\ref{assumption:RKHS},~\ref{assumption:original}, and~\ref{assumption:smooth_bridge} hold. I adopt the language of \cite{caponnetto2007optimal}.
\begin{enumerate}
   \item The generalized reconstruction error is $\mathcal{B}(\xi)=\|h_{\xi}-h_0\|^2_{\mathcal{H}} \leq \zeta \cdot \xi^{c-1}$.
    \item The generalized effective dimension is $\mathcal{N}(\xi)=\text{\normalfont tr}\{(T+\xi)^{-1}T\}\leq C (\pi/b) \{\sin(\pi/b)\}^{-1}\xi^{-1/b}$.
\end{enumerate}
\end{lemma}

\begin{proof}
     The first result is a corollary of Proposition~\ref{approx_sup}. The second result follows from \cite[eq. f]{sutherland2017fixing}, appealing to the effective dimension condition in Assumption~\ref{assumption:smooth_bridge}.
\end{proof}

\begin{lemma}[Decomposition of variance]\label{lemma:decomp}
The following bound holds:
\begin{align*}
    \|h^m_{\xi}-h_{\xi}\|_{\mathcal{H}}
    &\leq \|(T+\xi)^{-1/2} \{\mathbf{g}-(\mathbf{T}+\xi)h_{\xi}\}\|_{\mathcal{H}} \\
    &\quad \cdot \|(T+\xi)^{1/2}(\mathbf{T}+\xi)^{-1}(T+\xi)^{1/2}\|_{op} \\
    &\quad \cdot \|(T+\xi)^{-1/2}\|_{op}.
\end{align*}
Moreover, in the first factor,
\begin{align*}
    (T+\xi)^{-1/2}\{\mathbf{g}-(\mathbf{T}+\xi)h_{\xi}\} &=\frac{1}{m}\sum_{i=1}^m \dot{\xi}_i-\mathbb{E}[\dot{\xi}]
\end{align*}
where 
\begin{align*}
    \dot{\xi}_i&=(T+\xi)^{-1/2}\{\mu(\dot{d}_i,\dot{x}_i,\dot{z}_i)\dot{y}_i-\mu(\dot{d}_i,\dot{x}_i,\dot{z}_i)\otimes \mu(\dot{d}_i,\dot{x}_i,\dot{z}_i)h_{\xi}\}\\
    &=(T+\xi)^{-1/2}\{\mu(\dot{d}_i,\dot{x}_i,\dot{z}_i)[\dot{y}_i-\langle h_{\xi},\mu(\dot{d}_i,\dot{x}_i,\dot{z}_i)\rangle_{\mathcal{H}}]\}.
\end{align*}

\end{lemma}

\begin{proof}
The result mirrors \cite[eq. 44]{fischer2017sobolev}. For the decomposition of the first factor, the definitions in Proposition~\ref{sol_2} give
$$
\frac{1}{m}\sum_{i=1}^m \dot{\xi}_i=(T+\xi)^{-1/2}(\mathbf{g}-\mathbf{T}  h_{\xi}).
$$
Meanwhile
\begin{align*}
    \mathbb{E}[\dot{\xi}]&=(T+\xi)^{-1/2}\{\mathbb{E}[\mu(D,X,Z)Y]-Th_{\xi}\} \\
    &=(T+\xi)^{-1/2}\{\mathbb{E}[\mu(D,X,Z)Y]-Th_{\xi}-\xi h_{\xi} + \xi h_{\xi}\} \\
    &=(T+\xi)^{-1/2}\{\mathbb{E}[\mu(D,X,Z)Y]-(T+\xi)h_{\xi} + \xi h_{\xi}\} \\
    &=(T+\xi)^{-1/2}\{\mathbb{E}[\mu(D,X,Z)Y]-\mathbb{E}[\mu(D,X,Z)Y]+\xi h_{\xi}\}  \\
    &=(T+\xi)^{-1/2}\{\xi h_{\xi}\}
\end{align*}
as desired. 
\end{proof}

\begin{lemma}[Bounding the first factor]\label{lemma:1}
Suppose Assumptions~\ref{assumption:RKHS} and~\ref{assumption:original} hold. Then with probability $1-\delta/2$, the first factor in Lemma~\ref{lemma:decomp} is bounded as 
\begin{align*}
  &\|(T+\xi)^{-1/2} \{\mathbf{g}-(\mathbf{T}+\xi)h_{\xi}\}\|_{\mathcal{H}} \\
  &\leq 
    4\log(4/\delta) \left\{\frac{\kappa C+\kappa^2 \|h_0\|_{\mathcal{H}}}{m\xi^{1/2}}+\frac{\kappa^2\mathcal{B}(\xi)^{1/2}}{m\xi^{1/2}}+\frac{(C+\kappa\|h_0\|_{\mathcal{H}}) \mathcal{N}(\xi)^{1/2}}{m^{1/2}}+\frac{\kappa \mathcal{B}(\xi)^{1/2}\mathcal{N}(\xi)^{1/2}}{m^{1/2}}\right\}.
\end{align*}
\end{lemma}

\begin{proof}
I verify the conditions of Lemma~\ref{lemma:prob}. Let
$$
\dot{\xi}_i=(T+\xi)^{-1/2}\{\mu(\dot{d}_i,\dot{x}_i,\dot{z}_i)[\dot{y}_i-\langle h_{\xi},\mu(\dot{d}_i,\dot{x}_i,\dot{z}_i)\rangle_{\mathcal{H}}]\}.
$$
I proceed in steps.
\begin{enumerate}
    \item First moment.
    
    Observe that 
\begin{align*}
    \|\dot{\xi}_i\|_{\mathcal{H}}
    &=\|(T+\xi)^{-1/2}\mu(\dot{d}_i,\dot{x}_i,\dot{z}_i)[\dot{y}_i-\langle h_{\xi},\mu(\dot{d}_i,\dot{x}_i,\dot{z}_i)\rangle_{\mathcal{H}}]\|_{\mathcal{H}} \\
    &\leq \| (T+\xi)^{-1/2}\mu(\dot{d}_i,\dot{x}_i,\dot{z}_i) \|_{\mathcal{H}} 
    \cdot |\dot{y}_i-\langle h_{\xi},\mu(\dot{d}_i,\dot{x}_i,\dot{z}_i)\rangle_{\mathcal{H}}|.
\end{align*}
Moreover
$$
\| (T+\xi)^{-1/2}\mu(\dot{d}_i,\dot{x}_i,\dot{z}_i) \|_{\mathcal{H}} \leq \|(T+\xi)^{-1/2}\|_{op} \| \mu(\dot{d}_i,\dot{x}_i,\dot{z}_i)  \|_{\mathcal{H}}\leq \frac{\kappa}{\xi^{1/2}}
$$
and
\begin{align*}
 |\dot{y}_i-\langle h_{\xi},\mu(\dot{d}_i,\dot{x}_i,\dot{z}_i)\rangle_{\mathcal{H}}|
 &\leq |\dot{y}_i-\langle h_{0},\mu(\dot{d}_i,\dot{x}_i,\dot{z}_i)\rangle_{\mathcal{H}}|
 + |\langle h_0-h_{\xi},\mu(\dot{d}_i,\dot{x}_i,\dot{z}_i)\rangle_{\mathcal{H}}| \\
 &\leq C+\kappa\|h_0\|_{\mathcal{H}} +\kappa\|h_0-h_{\xi}\|_{\mathcal{H}} \\
 &\leq C+\kappa\{\|h_0\|_{\mathcal{H}} +\mathcal{B}(\xi)^{1/2}\}.
\end{align*}
In summary,
$$
\|\dot{\xi}_i\|_{\mathcal{H}} \leq \frac{\kappa}{\xi^{1/2}} [C+\kappa\{\|h_0\|_{\mathcal{H}} +\mathcal{B}(\xi)^{1/2}\}].
$$
    \item Second moment.
    
    Next, write
    \begin{align*}
   &\mathbb{E}(\|\dot{\xi}_i\|^2_{\mathcal{H}}) \\
   &= \int [\dot{y}_i-\langle h_{\xi},\mu(\dot{d}_i,\dot{x}_i,\dot{z}_i)\rangle_{\mathcal{H}}]^2 \langle \mu(\dot{d}_i,\dot{x}_i,\dot{z}_i) , (T+\xi)^{-1} \mu(\dot{d}_i,\dot{x}_i,\dot{z}_i)\rangle_{\mathcal{H}}  \mathrm{d}\mathbb{P}(\dot{y}_i,\dot{d}_i,\dot{x}_i,\dot{z}_i) \\
   &\leq \sup_{y,d,x,z} [y-\langle h_{\xi},\mu(d,x,z)\rangle_{\mathcal{H}}]^2 
   \int \langle \mu(\dot{d}_i,\dot{x}_i,\dot{z}_i) , (T+\xi)^{-1} \mu(\dot{d}_i,\dot{x}_i,\dot{z}_i)\rangle_{\mathcal{H}}  \mathrm{d}\mathbb{P}(\dot{d}_i,\dot{x}_i,\dot{z}_i).
    \end{align*}
    Focusing on the former factor, as argued above,
    $$
    \sup_{y,d,x,z} [y-\langle h_{\xi},\mu(d,x,z)\rangle_{\mathcal{H}}]^2  \leq \left[C+\kappa\{\|h_0\|_{\mathcal{H}} +\mathcal{B}(\xi)^{1/2}\} \right]^2.
    $$
    Focusing on the latter factor, 
    \begin{align*}
       &\int \langle \mu(\dot{d}_i,\dot{x}_i,\dot{z}_i) , (T+\xi)^{-1} \mu(\dot{d}_i,\dot{x}_i,\dot{z}_i)\rangle_{\mathcal{H}}  \mathrm{d}\mathbb{P}(\dot{d}_i,\dot{x}_i,\dot{z}_i) \\
        &=\int \text{\normalfont tr}[(T+\xi)^{-1} \{\mu(\dot{d}_i,\dot{x}_i,\dot{z}_i) \otimes\mu(\dot{d}_i,\dot{x}_i,\dot{z}_i) \}] \mathrm{d}\mathbb{P}(\dot{d}_i,\dot{x}_i,\dot{z}_i) \\
        &= \text{\normalfont tr}\{(T+\xi)^{-1} T\} \\
        &=\mathcal{N}(\xi).
    \end{align*}
    In summary,
    $$
    \mathbb{E}(\|\dot{\xi}_i\|^2_{\mathcal{H}}) \leq \mathcal{N}(\xi)\left[C+\kappa\{\|h_0\|_{\mathcal{H}} +\mathcal{B}(\xi)^{1/2}\} \right]^2.
    $$
    \item Concentration.
    
    Therefore with probability $1-\delta/2$,
    \begin{align*}
        &\left\|\frac{1}{m}\sum_{i=1}^m \dot{\xi}_i-\mathbb{E}[\dot{\xi}] \right\|_{\mathcal{H}}  \\
        &\leq \frac{2 \log(4/\delta)}{m} \frac{\kappa}{\xi^{1/2}} [C+\kappa\{\|h_0\|_{\mathcal{H}} +\mathcal{B}(\xi)^{1/2}\}] + \left[\frac{2 \log(4/\delta)}{m} \mathcal{N}(\xi)\left[C+\kappa\{\|h_0\|_{\mathcal{H}} +\mathcal{B}(\xi)^{1/2}\} \right]^2 \right]^{1/2} \\
        &\leq 4\log(4/\delta) \left\{\frac{\kappa C+\kappa^2 \|h_0\|_{\mathcal{H}}}{m\xi^{1/2}}+\frac{\kappa^2\mathcal{B}(\xi)^{1/2}}{m\xi^{1/2}}+\frac{(C+\kappa\|h_0\|_{\mathcal{H}}) \mathcal{N}(\xi)^{1/2}}{m^{1/2}}+\frac{\kappa \mathcal{B}(\xi)^{1/2}\mathcal{N}(\xi)^{1/2}}{m^{1/2}}\right\}.
     \end{align*}

\end{enumerate}

\end{proof}

\begin{remark}[Sufficiently large $m$]\label{remark:big_n}
In the finite sample, I assume a certain inequality holds when bounding the second factor: 
\begin{equation}\label{eq:n_big}
  m\geq 8\kappa^2 \log(4/\delta) \cdot \xi \cdot 
\log
\left\{ 2e\cdot \mathcal{N}(\xi)\frac{\|T\|_{op}+\xi}{\|T\|_{op}}
\right\},\quad \kappa=\kappa_d\cdot \kappa_x\cdot \kappa_w.  
\end{equation}
Ultimately, I will choose $\xi=m^{-1/(c+1/b)}$ in Theorem~\ref{theorem:consistency_treatment}. This choice of $\xi$ together with the bound on generalized effective dimension $\mathcal{N}(\xi)$ in Lemma~\ref{lemma:bounds} imply that there exists an $m_0$ such that for all $m\geq m_0$,~\eqref{eq:n_big} holds, as argued by \cite[Proof of Theorem 1]{fischer2017sobolev}. I use the phrase ``$m$ sufficiently large'' when I appeal to this logic, and I summarize the final bound using $O(\cdot)$ notation.
\end{remark}

\begin{lemma}[Bounding the second factor]\label{lemma:2}
Suppose Assumptions~\ref{assumption:RKHS} and~\ref{assumption:original} hold. Further assume~\eqref{eq:n_big} holds. Then probability $1-\delta/2$, the second factor in Lemma~\ref{lemma:decomp} is bounded as 
$$
\|(T+\xi)^{1/2}(\mathbf{T}+\xi)^{-1}(T+\xi)^{1/2}\|_{op} \leq 3.
$$
\end{lemma}

\begin{proof}
The result follows from \cite[eq. 44b, 47]{fischer2017sobolev}. In particular, my assumptions suffice for the properties used in \cite[Lemma 17]{fischer2017sobolev} to hold, using the same argument as \cite[Lemma I.4]{singh2020kernel}. Separability of original spaces together with boundedness of kernels imply that $\mathcal{H}$ is separable \cite[Lemma 4.33]{steinwart2008support}. Next, I verify the assumptions called EMB, EVD, and SRC. Boundedness of the kernel implies EMB with $a=1$. EVD is the assumption I call effective dimension, parametrized by $b\geq 1$. SRC is the assumption I call the source condition, parametrized by $c\in(1,2]$. 
\end{proof}

\begin{lemma}[Bounding the third factor]\label{lemma:3}
With probability one, the third factor in Lemma~\ref{lemma:decomp} is bounded as
$$
    \|(T+\xi)^{-1/2}\|_{op} \leq \xi^{-1/2}.
$$
\end{lemma}

\begin{proof}
The result follows from the definition of operator norm.
\end{proof}

\begin{proposition}[Variance]\label{sampling_sup}
Suppose Assumptions~\ref{assumption:solution},~\ref{assumption:RKHS},~\ref{assumption:original}, and~\ref{assumption:smooth_op} hold. Then $\forall \delta\in(0,1)$, for $m$ sufficiently large, the following holds with probability $1-\delta$:
$$
\|h^m_{\xi}-h_{\xi}\|_{\mathcal{H}}\leq C \log(4/\delta) \left\{ \frac{1}{m\xi}+\frac{1}{m^{1/2} \xi^{\frac{1}{2b}+\frac{1}{2}}}\right\}.
$$
\end{proposition}

\begin{proof}
I combine the previous lemmas to generalize \cite[Theorem 16]{fischer2017sobolev}. By Lemmas~\ref{lemma:decomp},~\ref{lemma:1},~\ref{lemma:2}, and~\ref{lemma:3}, if~\eqref{eq:n_big} holds, then with probability $1-\delta$
\begin{align*}
    &\|h^m_{\xi}-h_{\xi}\|_{\mathcal{H}}  \\
    &\leq \frac{12\log(4/\delta) }{\xi^{1/2}} \left\{\frac{\kappa C+\kappa^2 \|h_0\|_{\mathcal{H}}}{m\xi^{1/2}}+\frac{\kappa^2\mathcal{B}(\xi)^{1/2}}{m\xi^{1/2}}+\frac{(C+\kappa\|h_0\|_{\mathcal{H}}) \mathcal{N}(\xi)^{1/2}}{m^{1/2}}+\frac{\kappa \mathcal{B}(\xi)^{1/2}\mathcal{N}(\xi)^{1/2}}{m^{1/2}}\right\}.
\end{align*}
Next, recall the bounds in Lemma~\ref{lemma:bounds}. When $\xi\leq 1$,
     $$
     \mathcal{B}(\xi)^{1/2} \leq \zeta^{1/2} \xi^{\frac{c-1}{2}} \leq \zeta^{1/2}.
     $$
     For brevity, write
     $$
     \mathcal{N}(\xi)^{1/2}\leq C'\xi^{-\frac{1}{2b}}.
     $$
     Therefore when $\xi\leq 1$ the bound simplifies as
     \begin{align*}
         \|h^m_{\xi}-h_{\xi}\|_{\mathcal{H}} 
    \leq C\log(4/\delta) \left\{\frac{1}{m\xi}+\frac{1}{m^{1/2}\xi^{1/(2b)+1/2}}\right\}.
     \end{align*}
\end{proof}

\subsection{Bounds}

\begin{proposition}\label{prop:bound_KIV}
Suppose Assumption~\ref{assumption:RKHS} holds. Assume that $\forall d\in\mathcal{D}, x\in\mathcal{X}, z\in\mathcal{Z}$, $\|\mu^n_{\lambda}(d,x,z)-\mu(d,x,z)\|_{\mathcal{H}}\leq r_{\mu}(n,\delta,b_0,c_0)$.
\begin{enumerate}
    \item Then $
 \|\hat{\mathbf{T}}-\mathbf{T}\|_{\mathcal{L}(\mathcal{H})}\leq \left\{2 \kappa+r_{\mu}(n,\delta,b_0,c_0) \right\}r_{\mu}(n,\delta,b_0,c_0).
$
    \item If in addition Assumption~\ref{assumption:original} holds then $\|\hat{\mathbf{g}}-\mathbf{g}\|_{\mathcal{H}}\leq  Cr_{\mu}(n,\delta,b_0,c_0)$.
\end{enumerate}
\end{proposition}

\begin{proof}
   For the former result, write
   \begin{align*}
       \hat{\mathbf{T}}-\mathbf{T}
       &=\dfrac{1}{m}\sum_{i=1}^m \{\mu^n_{\lambda}(\dot{d}_i,\dot{x}_i,\dot{z}_i) \otimes \mu^n_{\lambda}(\dot{d}_i,\dot{x}_i,\dot{z}_i) - \mu(\dot{d}_i,\dot{x}_i,\dot{z}_i)\otimes \mu(\dot{d}_i,\dot{x}_i,\dot{z}_i)\} \\
       &=\dfrac{1}{m}\sum_{i=1}^m \{\mu^n_{\lambda}(\dot{d}_i,\dot{x}_i,\dot{z}_i) \otimes \mu^n_{\lambda}(\dot{d}_i,\dot{x}_i,\dot{z}_i)
       -\mu^n_{\lambda}(\dot{d}_i,\dot{x}_i,\dot{z}_i) \otimes \mu(\dot{d}_i,\dot{x}_i,\dot{z}_i) \\
       &\quad +\mu^n_{\lambda}(\dot{d}_i,\dot{x}_i,\dot{z}_i) \otimes \mu(\dot{d}_i,\dot{x}_i,\dot{z}_i)
       - \mu(\dot{d}_i,\dot{x}_i,\dot{z}_i)\otimes \mu(\dot{d}_i,\dot{x}_i,\dot{z}_i)\}.
   \end{align*}
   Hence
   \begin{align*}
       \|\mathbf{T}-\hat{\mathbf{T}}\|_{\mathcal{L}(\mathcal{H})}
       &\leq \dfrac{1}{m}\sum_{i=1}^m \|\mu^n_{\lambda}(\dot{d}_i,\dot{x}_i,\dot{z}_i) \otimes \mu^n_{\lambda}(\dot{d}_i,\dot{x}_i,\dot{z}_i)
       -\mu^n_{\lambda}(\dot{d}_i,\dot{x}_i,\dot{z}_i) \otimes \mu(\dot{d}_i,\dot{x}_i,\dot{z}_i)\|_{\mathcal{L}(\mathcal{H})} \\
       &\quad +\|\mu^n_{\lambda}(\dot{d}_i,\dot{x}_i,\dot{z}_i) \otimes \mu(\dot{d}_i,\dot{x}_i,\dot{z}_i)
       - \mu(\dot{d}_i,\dot{x}_i,\dot{z}_i)\otimes \mu(\dot{d}_i,\dot{x}_i,\dot{z}_i)\|_{\mathcal{L}(\mathcal{H})}.
   \end{align*}
   In the first term, 
   \begin{align*}
       &\|\mu^n_{\lambda}(\dot{d}_i,\dot{x}_i,\dot{z}_i) \otimes \mu^n_{\lambda}(\dot{d}_i,\dot{x}_i,\dot{z}_i)
       -\mu^n_{\lambda}(\dot{d}_i,\dot{x}_i,\dot{z}_i) \otimes \mu(\dot{d}_i,\dot{x}_i,\dot{z}_i)\|_{\mathcal{L}(\mathcal{H})} \\
       &= \|\mu^n_{\lambda}(\dot{d}_i,\dot{x}_i,\dot{z}_i)\|_{\mathcal{H}} \cdot \|\mu^n_{\lambda}(\dot{d}_i,\dot{x}_i,\dot{z}_i)-\mu(\dot{d}_i,\dot{x}_i,\dot{z}_i)\|_{\mathcal{H}} \\
       &\leq \left(\|\mu(\dot{d}_i,\dot{x}_i,\dot{z}_i)\|_{\mathcal{H}}+\|\mu^n_{\lambda}(\dot{d}_i,\dot{x}_i,\dot{z}_i)-\mu(\dot{d}_i,\dot{x}_i,\dot{z}_i)\|_{\mathcal{H}} \right)\|\mu^n_{\lambda}(\dot{d}_i,\dot{x}_i,\dot{z}_i)-\mu(\dot{d}_i,\dot{x}_i,\dot{z}_i)\|_{\mathcal{H}} \\
       &\leq \left\{\kappa+r_{\mu}(n,\delta,b_0,c_0) \right\}r_{\mu}(n,\delta,b_0,c_0).
   \end{align*}
   In the second term
   \begin{align*}
       &\|\mu^n_{\lambda}(\dot{d}_i,\dot{x}_i,\dot{z}_i) \otimes \mu(\dot{d}_i,\dot{x}_i,\dot{z}_i)
       - \mu(\dot{d}_i,\dot{x}_i,\dot{z}_i)\otimes \mu(\dot{d}_i,\dot{x}_i,\dot{z}_i)\|_{\mathcal{L}(\mathcal{H})} \\
       &=\|\mu^n_{\lambda}(\dot{d}_i,\dot{x}_i,\dot{z}_i)-\mu(\dot{d}_i,\dot{x}_i,\dot{z}_i)\|_{\mathcal{H}}\cdot \|\mu(\dot{d}_i,\dot{x}_i,\dot{z}_i)\|_{\mathcal{H}} \\
       &\leq \kappa \cdot r_{\mu}(n,\delta,b_0,c_0).
   \end{align*}
   In summary,
   $$
   \|\mathbf{T}-\hat{\mathbf{T}}\|_{\mathcal{L}(\mathcal{H})}
       \leq \left\{2 \kappa+r_{\mu}(n,\delta,b_0,c_0) \right\}r_{\mu}(n,\delta,b_0,c_0).
   $$
   For the latter result, write
   $$
   \hat{\mathbf{g}}-\mathbf{g}=\dfrac{1}{m}\sum_{i=1}^m \{\mu^n_{\lambda}(\dot{d}_i,\dot{x}_i,\dot{z}_i)-\mu(\dot{d}_i,\dot{x}_i,\dot{z}_i)\} \dot{y}_i.
   $$
   Hence
   $$
   \|\hat{\mathbf{g}}-\mathbf{g}\|_{\mathcal{H}} \leq \dfrac{1}{m}\sum_{i=1}^m \|\{\mu^n_{\lambda}(\dot{d}_i,\dot{x}_i,\dot{z}_i)-\mu(\dot{d}_i,\dot{x}_i,\dot{z}_i)\}\|_{\mathcal{H}}\cdot |\dot{y}_i|\leq Cr_{\mu}(n,\delta,b_0,c_0).
   $$
\end{proof}

\begin{proposition}\label{bound_diff}
Suppose Assumption~\ref{assumption:RKHS} holds. If $\forall d\in\mathcal{D}, x\in\mathcal{X}, z\in\mathcal{Z}$, $\|\mu^n_{\lambda}(d,x,z)-\mu(d,x,z)\|_{\mathcal{H}}\leq r_{\mu}(n,\delta,b_0,c_0)$ then
$$
 \|(\hat{\mathbf{T}}+\xi)^{-1}
-(\mathbf{T}+\xi)^{-1}\|_{\mathcal{L}(\mathcal{H})} \leq \frac{\left\{2 \kappa+r_{\mu}(n,\delta,b_0,c_0) \right\}r_{\mu}(n,\delta,b_0,c_0)}{\xi^2}.
$$
\end{proposition}

\begin{proof}
Since $A^{-1}-B^{-1}=B^{-1}(B-A)A^{-1}$,
$$
(\hat{\mathbf{T}}+\xi)^{-1}
-(\mathbf{T}+\xi)^{-1}=(\mathbf{T}+\xi)^{-1}(\mathbf{T}-\hat{\mathbf{T}})(\hat{\mathbf{T}}+\xi)^{-1}.
$$
Therefore
$$
 \|(\hat{\mathbf{T}}+\xi)^{-1}
-(\mathbf{T}+\xi)^{-1}\|_{\mathcal{L}(\mathcal{H})} 
\leq \frac{1}{\xi^2}\|\mathbf{T}-\hat{\mathbf{T}}\|_{\mathcal{L}(\mathcal{H})}\leq \frac{\left\{2 \kappa+r_{\mu}(n,\delta,b_0,c_0) \right\}r_{\mu}(n,\delta,b_0,c_0)}{\xi^2}
$$
where the final inequality appeals to Proposition~\ref{prop:bound_KIV}.
\end{proof}

\begin{proposition}\label{bound_g}
Suppose Assumptions~\ref{assumption:RKHS} and~\ref{assumption:original} hold. Then
$$
\|\mathbf{g}\|_{\mathcal{H}}\leq \kappa C.
$$
\end{proposition}

\begin{proof}
$$
 \|\mathbf{g}\|_{\mathcal{H}}
    =\left\|\frac{1}{m}\sum_{i=1}^m\mu(\dot{d}_i,\dot{x}_i,\dot{z}_i)\dot{y}_i\right\|_{\mathcal{H}}
    \leq \frac{1}{m}\sum_{i=1}^m \left\|\mu(\dot{d}_i,\dot{x}_i,\dot{z}_i)\right\|_{\mathcal{H}}|\dot{y}_i|  \\
    \leq \kappa C.
$$
\end{proof}

\subsection{Main result}

\begin{theorem}[Collecting results]\label{sup_consistency}
Suppose Assumptions~\ref{assumption:solution}, \ref{assumption:RKHS}, \ref{assumption:original}, and \ref{assumption:smooth_bridge} hold. Suppose $\forall d\in\mathcal{D}, x\in\mathcal{X}, z\in\mathcal{Z}$,  $\|\mu^n_{\lambda}(d,x,z)-\mu(d,x,z)\|_{\mathcal{H}}\leq r_{\mu}(n,\delta,b_0,c_0)$ with probability $1-\delta$. Then $\forall \delta\in(0,1)$ and $\forall \eta\in (0,1)$, the following holds with probability $1-\eta-\delta$:
\begin{align*}
    &\|\hat{h}_{\xi}^m-h_0\|_{\mathcal{H}}\\
    &\leq r_h(n,\delta,c_0;m,\eta,c) \\
    &:=\frac{Cr_{\mu}(n,\delta,b_0,c_0)}{\xi} +\frac{\kappa C }{\xi^2}\left\{2 \kappa+r_{\mu}(n,\delta,b_0,c_0) \right\}r_{\mu}(n,\delta,b_0,c_0)\\
    &\quad +C \log(4/\eta) \left\{ \frac{1}{m\xi}+\frac{1}{m^{1/2} \xi^{\frac{1}{2b}+\frac{1}{2}}}\right\}+\xi^{\frac{c-1}{2}} \sqrt{\zeta}.
\end{align*}
\end{theorem}

\begin{proof}
I begin with a decomposition using Proposition~\ref{sol_2}.
$$
\hat{h}_{\xi}^m-h_0
=\left[(\hat{\mathbf{T}}+\xi)^{-1}\hat{\mathbf{g}}
-(\hat{\mathbf{T}}+\xi)^{-1}\mathbf{g}\right]
+\left[(\hat{\mathbf{T}}+\xi)^{-1}\mathbf{g}
-(\mathbf{T}+\xi)^{-1}\mathbf{g}\right]
+\left[(\mathbf{T}+\xi)^{-1}\mathbf{g}
-h_0\right].
$$

Consider the first term. 
$$
\|(\hat{\mathbf{T}}+\xi)^{-1}(\hat{\mathbf{g}}-\mathbf{g})\|_{\mathcal{H}}\leq \frac{1}{\xi}\|\hat{\mathbf{g}}-\mathbf{g}\|_{\mathcal{H}}\leq \frac{1}{\xi} Cr_{\mu}(n,\delta,b_0,c_0)
$$
by Proposition~\ref{prop:bound_KIV}.

Consider the second term.
\begin{align*}
    \left\|\left[(\hat{\mathbf{T}}+\xi)^{-1}
-(\mathbf{T}+\xi)^{-1}\right]\mathbf{g}\right\|_{\mathcal{H}}
&\leq \|(\hat{\mathbf{T}}+\xi)^{-1}
-(\mathbf{T}+\xi)^{-1}\|_{\mathcal{L}(\mathcal{H})}\|\mathbf{g}\|_{\mathcal{H}}\\
&\leq \frac{1}{\xi^2}\left\{2 \kappa+r_{\mu}(n,\delta,b_0,c_0) \right\}r_{\mu}(n,\delta,b_0,c_0) \cdot \kappa C
\end{align*}
with probability $1-\delta$ by Propositions~\ref{bound_diff} and ~\ref{bound_g}.

Consider the third term.
$$
\left\|h_{\xi}^m
-h_0\right\|_{\mathcal{H}}\leq \left\|h_{\xi}^m-h_{\xi}\right\|_{\mathcal{H}}+\left\|h_{\xi}
-h_0\right\|_{\mathcal{H}}
\leq C \log(4/\eta) \left\{ \frac{1}{m\xi}+\frac{1}{m^{1/2} \xi^{\frac{1}{2b}+\frac{1}{2}}}\right\}+\xi^{\frac{c-1}{2}} \sqrt{\zeta}
$$
with probability $1-\eta$, appealing to triangle inequality and Propositions~\ref{approx_sup} and~\ref{sampling_sup}.
\end{proof}

\begin{theorem}[Conditional mean embedding]\label{theorem:conditional}
Suppose Assumptions~\ref{assumption:RKHS},~\ref{assumption:original}, and~\ref{assumption:smooth_op} hold. Set $\lambda=n^{-\frac{1}{c_0+1/b_0}}$. Then with probability $1-\delta$, $\forall d\in\mathcal{D}, x\in\mathcal{X}, z\in\mathcal{Z}$
    $$
  \|\hat{\mu}(d,x,z)-\mu(d,x,z)\|_{\mathcal{H}}\leq r_{\mu}(n,\delta,b_0,c_0)
    $$
    where  $
   \kappa_{RF}:=\kappa_d\kappa_x\kappa_z
    $ and
    $$
    r_{\mu}(n,\delta,b_0,c_0):=\kappa_d\kappa_x\cdot \kappa_{RF}\cdot
  C \log(4/\delta) n^{-\frac{1}{2}\frac{c_0-1}{c_0+1/b_0}}.
    $$
\end{theorem}

\begin{proof}
Write
\begin{align*}
    \hat{\mu}(d,x,z)-\mu(d,x,z)&=[\phi(d)\otimes \phi(x)\otimes \hat{\mu}_w(d,x,z)]-[\phi(d)\otimes \phi(x)\otimes \mu_w(d,x,z)] \\
    &=\phi(d)\otimes \phi(x)\otimes [\hat{\mu}_w(d,x,z)-\mu_w(d,x,z)]
\end{align*}
so that
$$
\| \hat{\mu}(d,x,z)-\mu(d,x,z)\|_{\mathcal{H}}\leq \kappa_d\kappa_x\cdot \|\hat{\mu}_w(d,x,z)-\mu_w(d,x,z)\|_{\mathcal{H}_{\mathcal{W}}}.
$$
The bound on the final factor follows from \cite[Proposition H.3]{singh2020kernel}, observing that
$$
E_0:\mathcal{H}_{\mathcal{W}}\rightarrow \mathcal{H}_{RF},\quad \|\phi(d,x,z)\|_{\mathcal{H}_{RF}}\leq \kappa_{RF}.
$$
\end{proof}

\begin{proof}[Proof of Theorem~\ref{theorem:consistency_bridge}]
     Summarize the bound in Theorem~\ref{sup_consistency} as
\begin{align*}
    r_h&=O\left(\frac{r_{\mu}}{\xi}+\frac{r_{\mu}}{\xi^2}+\frac{r^2_{\mu}}{\xi^2}
    +\frac{1}{m\xi}+\frac{1}{m^{1/2} \xi^{\frac{1}{2b}+\frac{1}{2}}}+\xi^{\frac{c-1}{2}}\right) \\
    &=O\left(\frac{r_{\mu}}{\xi^2}+\frac{1}{m\xi}+\frac{1}{m^{1/2} \xi^{\frac{1}{2b}+\frac{1}{2}}}+\xi^{\frac{c-1}{2}}\right) \\
    &=O\left(\frac{r_{\mu}}{\xi^2}+\frac{1}{m^{1/2} \xi^{\frac{1}{2b}+\frac{1}{2}}}+\xi^{\frac{c-1}{2}}\right)
\end{align*}
where the last statement holds when $m\xi\geq 1$. Summarize the bound in Theorem~\ref{theorem:conditional} as
$$
    r_{\mu}=O\left(n^{-\frac{1}{2}\frac{c_0-1}{c_0+1/b_0}}\right)=O(m^{-\frac{a}{2}}),\quad 
    n=m^{\frac{a(c_0+1/b_0)}{(c_0-1)}}.
$$
Combining results,
$$
r_h=O\left(\frac{1}{m^{\frac{a}{2}}\xi^2}+\frac{1}{m^{1/2} \xi^{\frac{1}{2b}+\frac{1}{2}}}+\xi^{\frac{c-1}{2}}\right),\quad \text{such that}\quad \xi^2\geq r_{\mu},\quad 
    m\xi^{\frac{1}{b}+1}\geq1.
$$
This choice of $(n,m)$ ratio generalizes the parametrization of \cite[Theorem 4]{singh2019kernel} to allow $b_0>1$.

I have choice over $\xi$ as a function of $m$ to achieve the single stage rate of $m^{-\frac{1}{2}\frac{c-1}{c+1/b}}$. I choose $\xi$ to match the bias $\xi^{\frac{c-1}{2}}$ with the variance $\frac{1}{m^{\frac{a}{2}}\xi^2}+\frac{1}{m^{1/2} \xi^{\frac{1}{2b}+\frac{1}{2}}}$. I set bias equal to each term in the variance.
\begin{enumerate}
    \item $\xi^{\frac{c-1}{2}}=\frac{1}{m^{\frac{a}{2}}\xi^2}$. Rearranging, $\xi=m^{-\frac{a}{c+3}}$. The bias term becomes
    $$
    \xi^{\frac{c-1}{2}}=\left(m^{-\frac{a}{c+3}}\right)^{\frac{c-1}{2}}
    $$
    and the remaining term becomes
    $$
    \frac{1}{\sqrt{m}\xi^{\frac{1}{2b}+\frac{1}{2}}}=\frac{m^{\frac{1}{2}(1+1/b)\frac{a}{c+3}}}{\sqrt{m}}=m^{\frac{(1+1/b)a-(c+3)}{2(c+3)}}.
    $$
    Note that the former dominates the latter if and only if
    $$
    {-\frac{a}{c+3}}{\frac{c-1}{2}}\geq \frac{(1+1/b)a-(c+3)}{2(c+3)} \iff a\leq \frac{c+3}{c+1/b}.
    $$
    \item $\xi^{\frac{c-1}{2}}=\frac{1}{\sqrt{m}\xi^{\frac{1}{2b}+\frac{1}{2}}}$. Rearranging, $\xi=m^{-\frac{1}{c+1/b}}$. The bias term becomes
    $$
    \xi^{\frac{c-1}{2}}=\left(m^{-\frac{1}{c+1/b}}\right)^{\frac{c-1}{2}}
    $$
    and the remaining term becomes
    $$
    \frac{1}{m^{\frac{a}{2}}\xi^2}=m^{-\frac{a}{2}}\left(m^{-\frac{1}{c+1/b}}\right)^{-2}=m^{\frac{4-a(c+1/b)}{2(c+1/b)}}.
    $$
    Note that the former dominates the latter if and only if
    $$
    -\frac{1}{c+1/b} \frac{c-1}{2}\geq \frac{4-a(c+1/b)}{2(c+1/b)}\iff a\geq \frac{c+3}{c+1/b}.
    $$
\end{enumerate}
\end{proof}

\section{Treatment effect consistency proof}\label{section:consistency_proof2}

In this appendix, I (i) explicitly specialize the smoothness assumptions, (ii) provide rates for unconditional mean embeddings, (iv) provide rates for conditional mean embeddings, and (iv) prove uniform consistency of negative control treatment effects.

\subsection{Smoothness assumptions}

\begin{assumption}[Smoothness for $\theta_0^{ATT}$]\label{assumption:smooth_ATT}
Assume
\begin{enumerate}
\item The conditional expectation operator $E_1$ is well specified as a Hilbert-Schmidt operator between RKHSs, i.e. $E_1\in \mathcal{L}_2(\mathcal{H}_{\mathcal{X}}\otimes \mathcal{H}_{\mathcal{W}},\mathcal{H}_{\mathcal{D}})$, where
    $$
    E_1:\mathcal{H}_{\mathcal{X}}\otimes \mathcal{H}_{\mathcal{W}} \rightarrow \mathcal{H}_{\mathcal{D}},\quad f(\cdot,\cdot)\mapsto \mathbb{E}[f(X,W)|D=\cdot].
    $$
    \item The conditional expectation operator is a particularly smooth element of $\mathcal{L}_2(\mathcal{H}_{\mathcal{X}}\otimes \mathcal{H}_{\mathcal{W}},\mathcal{H}_{\mathcal{D}})$. Formally, define the covariance operator $T_1:=\mathbb{E}[\phi(D)\otimes \phi(D)]$ for $\mathcal{L}_2(\mathcal{H}_{\mathcal{X}}\otimes \mathcal{H}_{\mathcal{W}},\mathcal{H}_{\mathcal{D}})$.
    I assume there exists $G_1\in \mathcal{L}_2(\mathcal{H}_{\mathcal{X}}\otimes \mathcal{H}_{\mathcal{W}},\mathcal{H}_{\mathcal{D}})$ such that $E_1=(T_1)^{\frac{c_1-1}{2}}\circ G_1$, $c_1\in(1,2]$, and $\|G_1\|^2_{\mathcal{L}_2(\mathcal{H}_{\mathcal{X}}\otimes \mathcal{H}_{\mathcal{W}},\mathcal{H}_{\mathcal{D}})}\leq\zeta_1$.
    \end{enumerate}
\end{assumption}

\begin{assumption}[Smoothness for $\theta_0^{CATE}$]\label{assumption:smooth_CATE}
Assume
\begin{enumerate}
\item The conditional expectation operator $E_2$ is well specified as a Hilbert-Schmidt operator between RKHSs, i.e. $E_2\in \mathcal{L}_2(\mathcal{H}_{\mathcal{X}}\otimes \mathcal{H}_{\mathcal{W}},\mathcal{H}_{\mathcal{V}})$, where
    $$
    E_2:\mathcal{H}_{\mathcal{X}}\otimes \mathcal{H}_{\mathcal{W}} \rightarrow \mathcal{H}_{\mathcal{V}},\quad f(\cdot,\cdot)\mapsto \mathbb{E}[f(X,W)|V=\cdot].
    $$
    \item The conditional expectation operator is a particularly smooth element of $\mathcal{L}_2(\mathcal{H}_{\mathcal{X}}\otimes \mathcal{H}_{\mathcal{W}},\mathcal{H}_{\mathcal{V}})$. Formally, define the covariance operator $T_2:=\mathbb{E}[\phi(V)\otimes \phi(V)]$ for $\mathcal{L}_2(\mathcal{H}_{\mathcal{X}}\otimes \mathcal{H}_{\mathcal{W}},\mathcal{H}_{\mathcal{V}})$.
    I assume there exists $G_2\in \mathcal{L}_2(\mathcal{H}_{\mathcal{X}}\otimes \mathcal{H}_{\mathcal{W}},\mathcal{H}_{\mathcal{V}})$ such that $E_2=(T_2)^{\frac{c_2-1}{2}}\circ G_2$, $c_2\in(1,2]$, and $\|G_2\|^2_{\mathcal{L}_2(\mathcal{H}_{\mathcal{X}}\otimes \mathcal{H}_{\mathcal{W}},\mathcal{H}_{\mathcal{V}})}\leq\zeta_2$.
    \end{enumerate}
\end{assumption}

\begin{proposition}
The following assumptions are equivalent
\begin{enumerate}
    \item Assumption~\ref{assumption:smooth_op} with $\mathcal{A}_1=\mathcal{X}\times \mathcal{W}$ and $\mathcal{B}_1=\mathcal{D}$ is equivalent to Assumption~\ref{assumption:smooth_ATT}
     \item Assumption~\ref{assumption:smooth_op} with $\mathcal{A}_2=\mathcal{X}\times \mathcal{W}$ and $\mathcal{B}_2=\mathcal{V}$ is equivalent to Assumption~\ref{assumption:smooth_CATE}
\end{enumerate}
\end{proposition}

\begin{proof}
     The result is immediate from \cite[Remark 2]{caponnetto2007optimal}. The expanded expressions are more convenient for analysis.
\end{proof}

\subsection{Unconditional mean embedding}

\begin{theorem}[Unconditional mean embedding]\label{theorem:mean}
Suppose Assumptions~\ref{assumption:RKHS} and~\ref{assumption:original} hold. Then with probability $1-\delta$,
$$
\|\hat{\mu}-\mu\|_{\mathcal{H}_{\mathcal{X}}\otimes \mathcal{H}_{\mathcal{W}}}\leq r_{\mu}(n,\delta):=\frac{4\kappa_x\kappa_w \ln(2/\delta)}{\sqrt{n}}.
$$
Likewise, with probability $1-\delta$
$$
\|\hat{\nu}-\nu\|_{\mathcal{H}_{\mathcal{X}}\otimes \mathcal{H}_{\mathcal{W}}}\leq r_{\nu}(\tilde{n},\delta):=\frac{4\kappa_x\kappa_w \ln(2/\delta)}{\sqrt{\tilde{n}}}.
$$
\end{theorem}

\begin{proof}
The first result follows from Lemma~\ref{lemma:prob} with $\xi_i=\phi(x_i)\otimes \phi(w_i)$, since
\begin{align*}
    \left\|\frac{1}{n}\sum_{i=1}^n [\phi(x_i)\otimes \phi(w_i)]-\mathbb{E}[\phi(X)\otimes \phi(W)]\right\|_{\mathcal{H}_{\mathcal{X}}\otimes \mathcal{H}_{\mathcal{W}}}
&\leq \frac{2 \kappa_x\kappa_w\ln(2/\delta)}{n}+\sqrt{\dfrac{2\kappa^2_x\kappa^2_w\ln(2/\delta)}{n}}\\
&\leq \frac{4\kappa_x\kappa_w \ln(2/\delta)}{\sqrt{n}}.
\end{align*}
The argument for $\nu$ is identical, using $\xi_i=\phi(\tilde{x}_i)\otimes \phi(\tilde{w}_i) $.
\end{proof}

\subsection{Conditional mean embedding}

\begin{theorem}[Conditional mean embedding rate]\label{theorem:conditional2}
Suppose Assumptions~\ref{assumption:RKHS} and~\ref{assumption:original} hold. Set $(\lambda_1,\lambda_2)=(n^{-\frac{1}{c_1+1/b_1}},n^{-\frac{1}{c_2+1/b_2}})$.
\begin{enumerate}
    \item If in addition Assumption~\ref{assumption:smooth_ATT} holds then with probability $1-\delta$, for $n$ sufficiently large, $\forall d\in\mathcal{D}$
    $$
  \|\hat{\mu}(d)-\mu(d)\|_{\mathcal{H}_{\mathcal{X}}\otimes\mathcal{H}_{\mathcal{W}}}\leq r^{ATT}_{\mu}(n,\delta,b_1,c_1)
    $$
    where
    $$
    r^{ATT}_{\mu}(n,\delta,b_1,c_1):=\kappa_{d}\cdot
  C \log(4/\delta) n^{-\frac{1}{2}\frac{c_1-1}{c_1+1/b_1}}.
    $$
     \item If in addition Assumption~\ref{assumption:smooth_CATE} holds then with probability $1-\delta$, for $n$ sufficiently large, $\forall v\in\mathcal{V}$
    $$
   \| \hat{\mu}(v)-\mu(v)\|_{\mathcal{H}_{\mathcal{X}}\otimes\mathcal{H}_{\mathcal{W}}}\leq r^{CATE}_{\mu}(n,\delta,b_2,c_2)
    $$
    where
    $$
    r^{CATE}_{\mu}(n,\delta,b_2,c_2):=\kappa_{v}\cdot  C \log(4/\delta) n^{-\frac{1}{2}\frac{c_2-1}{c_2+1/b_2}}.
    $$
\end{enumerate}
\end{theorem}

\begin{proof}
The proof immediately follows from \cite[Proposition H.3]{singh2020kernel}, observing that
$$
E_1:\mathcal{H}_{\mathcal{X}}\otimes\mathcal{H}_{\mathcal{W}}\rightarrow \mathcal{H}_{\mathcal{D}},\quad \|\phi(d)\|_{\mathcal{H}_{\mathcal{D}}}\leq \kappa_{d}
$$
and
$$
E_2:\mathcal{H}_{\mathcal{X}}\otimes\mathcal{H}_{\mathcal{W}}\rightarrow \mathcal{H}_{\mathcal{V}},\quad \|\phi(v)\|_{\mathcal{H}_{\mathcal{V}}}\leq \kappa_{v}.
$$
\end{proof}

\subsection{Main result}

In summary, the rates are
\begin{align*}
    \|\hat{h}-h_0\|_{\mathcal{H}}&=O_p\left(m^{-\frac{1}{2}\frac{c-1}{c+1/b}}\right) \\
    \|\hat{\mu}-\mu\|_{\mathcal{H}_{\mathcal{X}}\otimes \mathcal{H}_{\mathcal{W}} }&=O_p\left(n^{-\frac{1}{2}}\right) \\
    \|\hat{\nu}-\nu\|_{\mathcal{H}_{\mathcal{X}}\otimes \mathcal{H}_{\mathcal{W}} }&=O_p\left(\tilde{n}^{-\frac{1}{2}}\right) \\
    \|\hat{\mu}(d)-\mu(d)\|_{\mathcal{H}_{\mathcal{X}}\otimes \mathcal{H}_{\mathcal{W}} }&=O_p\left(n^{-\frac{1}{2}\frac{c_1-1}{c_1+1/b_1}}\right) \\
     \|\hat{\mu}(v)-\mu(v)\|_{\mathcal{H}_{\mathcal{X}}\otimes \mathcal{H}_{\mathcal{W}} }&=O_p\left(n^{-\frac{1}{2}\frac{c_2-1}{c_2+1/b_2}}\right).
\end{align*}

\begin{proof}[Proof of Theorem~\ref{theorem:consistency_treatment}]
I generalize the argument in \cite[Theorem 3.3]{singh2020kernel}. I write out the finite sample bounds. Consider the decomposition
\begin{align*}
        &\hat{\theta}^{ATE}(d)-\theta_0^{ATE}(d)\\
        &=\langle \hat{h} , \phi(d)\otimes \hat{\mu} \rangle_{\mathcal{H}} - \langle h_0 , \phi(d)\otimes \mu \rangle_{\mathcal{H}} \\
        &=\langle \hat{h} , \phi(d)\otimes[\hat{\mu}-\mu] \rangle_{\mathcal{H}} + \langle [\hat{h}-h_0], \phi(d) \otimes \mu \rangle_{\mathcal{H}} \\
        &=\langle [\hat{h}-h_0], \phi(d)\otimes[\hat{\mu}-\mu] \rangle_{\mathcal{H}} + \langle h_0, \phi(d)\otimes[\hat{\mu}-\mu] \rangle_{\mathcal{H}}+\langle [\hat{h}-h_0], \phi(d) \otimes \mu \rangle_{\mathcal{H}}.
    \end{align*}
    Therefore with probability $1-2\delta-\eta$
   \begin{align*}
       &|\hat{\theta}^{ATE}(d)-\theta_0^{ATE}(d)|\\
       &\leq 
       \|\hat{h}-h_0\|_{\mathcal{H}}\|\phi(d)\|_{\mathcal{H}_{\mathcal{D}}} \|\hat{\mu}-\mu\|_{\mathcal{H}_{\mathcal{X}}\otimes\mathcal{H}_{\mathcal{W}}}
       +
       \|h_0\|_{\mathcal{H}}\|\phi(d)\|_{\mathcal{H}_{\mathcal{D}}}\|\hat{\mu}-\mu\|_{\mathcal{H}_{\mathcal{X}}\otimes\mathcal{H}_{\mathcal{W}}} \\
       &\quad +
       \|\hat{h}-h_0\|_{\mathcal{H}}\|\phi(d)\|_{\mathcal{H}_{\mathcal{D}}} \|\mu\|_{\mathcal{H}_{\mathcal{X}}\otimes\mathcal{H}_{\mathcal{W}}}
      \\
      &\leq \kappa_d \cdot r_h(n,\delta,b_0,c_0;m,\eta,b,c) \cdot r_{\mu}(n,\delta)+\kappa_d\cdot\|h_0\|_{\mathcal{H}} \cdot r_{\mu}(n,\delta)\\
      &\quad +\kappa_d\kappa_x\kappa_w \cdot r_h(n,\delta,b_0,c_0;m,\eta,b,c)\\
      &=O\left(m^{-\frac{1}{2}\frac{c-1}{c+1/b}}+n^{-\frac{1}{2}}\right).
   \end{align*}
By the same argument as for $\theta_0^{ATE}$, with probability $1-2\delta-\eta$
    \begin{align*}
    &|\hat{\theta}^{DS}(d,\tilde{\mathbb{P}})-\theta_0^{DS}(d,\tilde{\mathbb{P}})| \\
    &\leq \kappa_d \cdot r_h(n,\delta,b_0,c_0;m,\eta,b,c) \cdot r_{\nu}(\tilde{n},\delta)+\kappa_d\cdot\|h_0\|_{\mathcal{H}} \cdot r_{\nu}(\tilde{n},\delta) \\
      &\quad +\kappa_d\kappa_x\kappa_w \cdot r_h(n,\delta,b_0,c_0;m,\eta,b,c)\\
      &=O\left( m^{-\frac{1}{2}\frac{c-1}{c+1/b}}+\tilde{n}^{-\frac{1}{2}}\right).
    \end{align*}
Next, I turn to the nonparametric treatment effects with conditional mean embeddings. Consider the decomposition
    \begin{align*}
        &\hat{\theta}^{ATT}(d,d')-\theta_0^{ATT}(d,d')\\
        &=\langle \hat{h} , \phi(d')\otimes \hat{\mu}(d) \rangle_{\mathcal{H}} - \langle h_0 , \phi(d')\otimes \mu(d) \rangle_{\mathcal{H}} \\
        &=\langle \hat{h} , \phi(d')\otimes[\hat{\mu}(d)-\mu(d)] \rangle_{\mathcal{H}} + \langle [\hat{h}-h_0], \phi(d') \otimes \mu(d) \rangle_{\mathcal{H}} \\
        &=\langle [\hat{h}-h_0], \phi(d')\otimes[\hat{\mu}(d)-\mu(d)] \rangle_{\mathcal{H}} + \langle h_0, \phi(d')\otimes[\hat{\mu}(d)-\mu(d)] \rangle_{\mathcal{H}}\\
        &\quad +\langle [\hat{h}-h_0], \phi(d') \otimes \mu(d) \rangle_{\mathcal{H}}.
    \end{align*}
    Therefore with probability $1-2\delta-\eta$
   \begin{align*}
       &|\hat{\theta}^{ATT}(d,d')-\theta_0^{ATT}(d,d')|\\
       &\leq 
       \|\hat{h}-h_0\|_{\mathcal{H}}\|\phi(d')\|_{\mathcal{H}_{\mathcal{D}}} \|\hat{\mu}(d)-\mu(d)\|_{\mathcal{H}_{\mathcal{X}}\otimes\mathcal{H}_{\mathcal{W}}} \\
       &\quad +
       \|h_0\|_{\mathcal{H}}\|\phi(d')\|_{\mathcal{H}_{\mathcal{D}}}\|\hat{\mu}(d)-\mu(d)\|_{\mathcal{H}_{\mathcal{X}}\otimes\mathcal{H}_{\mathcal{W}}} \\
       &\quad+
       \|\hat{h}-h_0\|_{\mathcal{H}}\|\phi(d')\|_{\mathcal{H}_{\mathcal{D}}} \|\mu(d)\|_{\mathcal{H}_{\mathcal{X}}\otimes\mathcal{H}_{\mathcal{W}}}
      \\
      &\leq \kappa_d \cdot r_h(n,\delta,b_0,c_0;m,\eta,b,c) \cdot r_{\mu}^{ATT}(n,\delta,b_1,c_1)+\kappa_d\cdot\|h_0\|_{\mathcal{H}} \cdot r_{\mu}^{ATT}(n,\delta,b_1,c_1)\\
      &\quad +\kappa_d\kappa_x\kappa_w \cdot r_h(n,\delta,b_0,c_0;m,\eta,b,c)
      \\
      &=O\left(m^{-\frac{1}{2}\frac{c-1}{c+1/b}}+n^{-\frac{1}{2}\frac{c_1-1}{c_1+1/b_1}}\right).
   \end{align*}
   Similarly, consider the decomposition
   \begin{align*}
        &\hat{\theta}^{CATE}(d,v)-\theta_0^{CATE}(d,v)\\
        &=\langle \hat{h} , \phi(d)\otimes \phi(v)\otimes \hat{\mu}(v) \rangle_{\mathcal{H}} - \langle h_0 , \phi(d )\otimes \phi(v) \otimes \mu(v) \rangle_{\mathcal{H}} \\
        &=\langle \hat{h} , \phi(d)\otimes \phi(v)\otimes[\hat{\mu}(v)-\mu(v)] \rangle_{\mathcal{H}} + \langle [\hat{h}-h_0], \phi(d)\otimes \phi(v) \otimes \mu(v) \rangle_{\mathcal{H}} \\
        &=\langle [\hat{h}-h_0], \phi(d)\otimes \phi(v)\otimes[\hat{\mu}(v)-\mu(v)] \rangle_{\mathcal{H}} \\
        &\quad + \langle h_0, \phi(d)\otimes \phi(v)\otimes[\hat{\mu}(v)-\mu(v)] \rangle_{\mathcal{H}}\\
        &\quad +\langle [\hat{h}-h_0], \phi(d)\otimes \phi(v) \otimes \mu(v) \rangle_{\mathcal{H}}.
    \end{align*}
    Therefore with probability $1-2\delta-\eta$
   \begin{align*}
       &|\hat{\theta}^{CATE}(d,v)-\theta_0^{CATE}(d,v)|\\
       &\leq 
       \|\hat{h}-h_0\|_{\mathcal{H}}\|\phi(d)\|_{\mathcal{H}_{\mathcal{D}}}\|\phi(v)\|_{\mathcal{H}_{\mathcal{V}}} \|\hat{\mu}(v)-\mu(v)\|_{\mathcal{H}_{\mathcal{X}}\otimes\mathcal{H}_{\mathcal{W}}}\\
      &\quad+
       \|h_0\|_{\mathcal{H}}\|\phi(d)\|_{\mathcal{H}_{\mathcal{D}}}\|\phi(v)\|_{\mathcal{H}_{\mathcal{V}}}\|\hat{\mu}(v)-\mu(v)\|_{\mathcal{H}_{\mathcal{X}}\otimes\mathcal{H}_{\mathcal{W}}} \\
       &\quad+
       \|\hat{h}-h_0\|_{\mathcal{H}}\|\phi(d)\|_{\mathcal{H}_{\mathcal{D}}}\|\phi(v)\|_{\mathcal{H}_{\mathcal{V}}} \|\mu(v)\|_{\mathcal{H}_{\mathcal{X}}\otimes\mathcal{H}_{\mathcal{W}}}
      \\
      &\leq \kappa_d\kappa_{v} \cdot r_h(n,\delta,b_0,c_0;m,\eta,b,c) \cdot r_{\mu}^{CATE}(n,\delta,b_2,c_2)+\kappa_d\kappa_{v}\cdot\|h_0\|_{\mathcal{H}} \cdot r_{\mu}^{CATE}(n,\delta,b_2,c_2)
      \\
      &\quad+\kappa_d\kappa_{v} \kappa_{x}\kappa_w \cdot r_h(n,\delta,b_0,c_0;m,\eta,b,c)
      \\
      &=O\left(m^{-\frac{1}{2}\frac{c-1}{c+1/b}}+n^{-\frac{1}{2}\frac{c_2-1}{c_2+1/b_2}}\right).
   \end{align*}
\end{proof}
\section{Simulation details}\label{section:simulation_details}

\subsection{General simulation design}

A single observation is generated as follows. Recall that $(Y,D)\in\mathbb{R}$, $X\in\mathbb{R}^{dim(X)}$, $Z\in\mathbb{R}^{dim(Z)}$, and $W\in\mathbb{R}^{dim(W)}$.
\begin{enumerate}
    \item Draw unobserved noise as
    $$
    \{\epsilon_i\}_{i\in[3]}\overset{i.i.d}{\sim}\mathcal{N}(0,1),\quad \nu_z\sim \mathcal{U}[-1,1]^{dim(Z)},\quad \nu_w \sim\mathcal{U}[-1,1]^{dim(W)}.
    $$
    \item Set unobserved confounders as
$
u_z=\epsilon_1+\epsilon_3$ and $u_w=\epsilon_2+\epsilon_3
$. 
    \item Set the negative controls as
$$
Z=\nu_z+0.25\cdot u_z\cdot  1_{dim(Z)},\quad W=\nu_w+0.25\cdot u_w \cdot  1_{dim(W)},
$$where $1_p\in\mathbb{R}^p$ is the vector of ones of length $p$.
    \item Draw covariates $X\sim\mathcal{N}(0,\Sigma)$ where the covariance matrix $\Sigma\in\mathbb{R}^{dim(X)\times dim(X)}$ is such that $\Sigma_{ii}=1$ and $\Sigma_{ij}=\frac{1}{2}\cdot 1\{|i-j|=1\}$ for $i\neq j$. 
    \item Then set treatment as
$$
D=\Lambda(3X^{\top}\beta_x+3Z^{\top}\beta_z)+0.25\cdot u_w.
$$
$\beta_x\in\mathbb{R}^{dim(X)}$ and $\beta_z\in\mathbb{R}^{dim(Z)}$ are quadratically decaying coefficients, e.g. $[\beta_x]_j=j^{-2}$. $\Lambda$ is the truncated logistic link function $\Lambda(t)=(0.9-0.1)\frac{\exp(t)}{1+\exp(t)}+0.1$.
\item Finally set the outcome as
$$
Y=\theta_0^{ATE}(D)+ 1.2(X^{\top}\beta_x+W^{\top}\beta_w)+DX_1+0.25\cdot u_z,$$
where $\beta_w\in\mathbb{R}^{dim(W)}$ is a quadratically decaying coefficient, i.e. $[\beta_w]_j=j^{-2}$.
\end{enumerate}

For the quadratic design, $\theta_0^{ATE}(d)=d^2+1.2d$ as in \cite{colangelo2020double}. For the sigmoid design, $\theta_0^{ATE}(d)=\ln(|16d-8|+1)\cdot sign(d-0.5)+1.2d$ similar to \cite{singh2019kernel}. Finally, for the peaked design, $\theta_0^{ATE}(d)=2\{d^4/600+\exp(-4d^2)+d/10-2\}+1.2d$ similar to \cite{singh2019kernel}.

I implement the estimator $\hat{\theta}^{ATE}(d)$ (\verb|N.C.|) described in Section~\ref{section:algorithm}, with the tuning procedure described in Appendix~\ref{section:tuning}. Specifically, I use ridge penalties determined by leave-one-out cross validation, and product Gaussian kernel with lengthscales set by the median heuristic. I implement the continuous treatment effect estimator of \cite{singh2020kernel} (\verb|T.E.|) using the same principles. The latter tuning procedure is simpler; whereas the new estimator involves reweighting a confounding bridge (with two ridge penalty hyperparameters), the previous estimator involves reweighting a regression (with one ridge penalty hyperparameter).

\subsection{Robustness to tuning}

As explained in Appendix~\ref{section:tuning}, the tuning procedure for the ridge penalties involves leave-one-out cross validation. I now confirm that the proposed estimator's performance is robust to improper tuning. Figures~\ref{fig:sim_tuning}(a) and~\ref{fig:sim_tuning}(b) summarize results. 

\begin{figure}[ht]
\begin{centering}
     \begin{subfigure}[b]{0.3\textwidth}
         \centering
         \includegraphics[width=\textwidth]{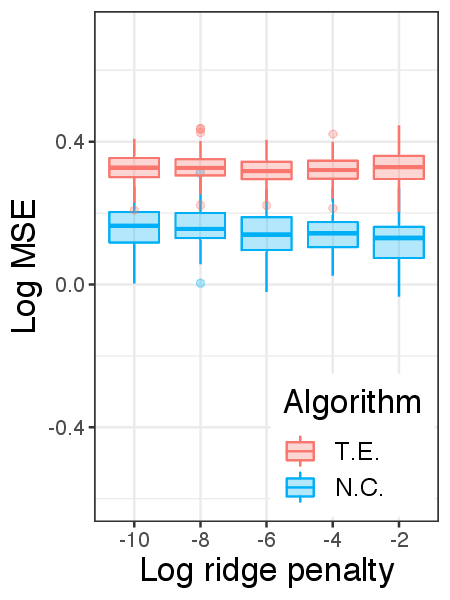}
         \vspace{-15pt}
         \caption{Forcing $\lambda$}
     \end{subfigure}
     \hfill
     \begin{subfigure}[b]{0.3\textwidth}
         \centering
         \includegraphics[width=\textwidth]{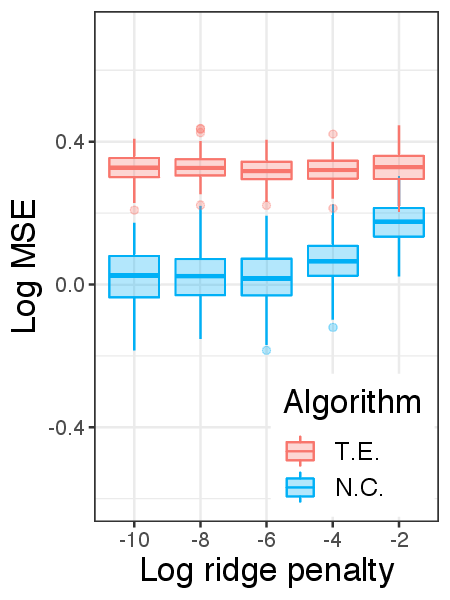}
         \vspace{-15pt}
         \caption{Forcing $\xi$}
     \end{subfigure}\hfill
     \begin{subfigure}[b]{0.3\textwidth}
         \centering
         \includegraphics[width=\textwidth]{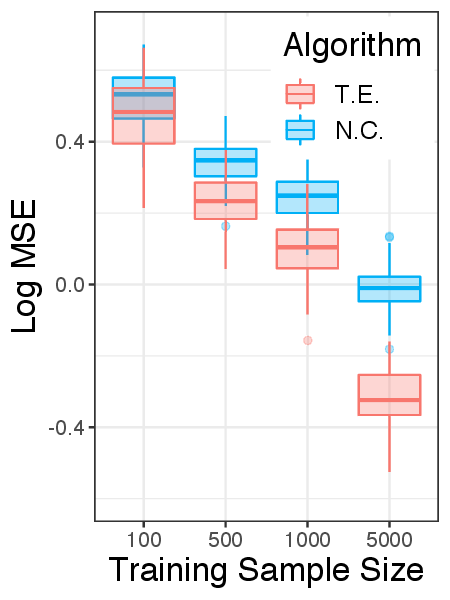}
         \vspace{-15pt}
         \caption{Unfounded design}
     \end{subfigure}
\par
\vspace{-5pt}
\caption{\label{fig:sim_tuning}
Robustness studies}
\end{centering}
\end{figure}

I conduct two robustness studies, each corresponding to a ridge penalty hyperparameter. In the first robustness study (Figure~\ref{fig:sim_tuning}(a)), I force $\lambda$ to take a particular value in a grid, then tune $\xi$ by leave-one-out cross validation: $\xi=\xi^*(\lambda)$. In the second robustness study (Figure~\ref{fig:sim_tuning}(b)), I tune $\lambda=\lambda^*$ by leave-one-out cross validation, then force $\xi$ to take a particular value in a grid. Across all levels of improper tuning of \verb|N.C.|, it continues to outperform the properly tuned \verb|T.E.| estimator.
 
 \subsection{No unobserved confounding}

Next, I study a setting where there is no unobserved confounding. I modify the data generating process as follows.
\begin{enumerate}
    \item Draw unobserved noise as
    $$
    \{\epsilon_i\}_{i\in[4]}\overset{i.i.d}{\sim}\mathcal{N}(0,1),\quad \nu_z\sim \mathcal{U}[-1,1]^{dim(Z)},\quad \nu_w \sim\mathcal{U}[-1,1]^{dim(W)}.
    $$
    \item Set unobserved confounders as
$
u_z=\epsilon_1+\epsilon_2$ and $u_w=\epsilon_3+\epsilon_4
$. 
    \item Set the negative controls as
$$
Z=\nu_z+0.25\cdot u_z\cdot  1_{dim(Z)},\quad W=\nu_w+0.25\cdot u_w \cdot  1_{dim(W)}.
$$
    \item Draw covariates $X\sim\mathcal{N}(0,\Sigma)$.
    \item Then set treatment as
$$
D=\Lambda(3X^{\top}\beta_x)+0.25\cdot u_w.
$$
\item Finally set the outcome as
$$
Y=\theta_0^{ATE}(D)+ 1.2(X^{\top}\beta_x)+DX_1+0.25\cdot u_z.$$
\end{enumerate}

Figure~\ref{fig:sim_tuning}(c) visualizes results. As expected, the previously existing method, which assumes no unobserved confounding, outperforms the proposed method. The proposed method solves an ill posed inverse problem. One pays a cost in terms of statistical efficiency when the ill posed inverse problem is unnecessary. From a practical perspective, an analyst should only use negative controls when the analyst firmly believes that unobserved confounding is present and that the negatives controls satisfy Assumptions~\ref{assumption:negative} and~\ref{assumption:solution}.

\subsection{Discrete treatment}

So far, I have focused on the case with continuous treatment. I now study empirical performance when treatment is discrete. I modify the data generating process as follows.
\begin{enumerate}
    \item Draw unobserved noise as
    $$
    \{\epsilon_i\}_{i\in[3]}\overset{i.i.d}{\sim}\mathcal{N}(0,1),\quad \nu_z\sim \mathcal{U}[-1,1]^{dim(Z)},\quad \nu_w \sim\mathcal{U}[-1,1]^{dim(W)}.
    $$
    \item Set unobserved confounders as
$
u_z=\epsilon_1+\epsilon_3$ and $u_w=\epsilon_2+\epsilon_3
$. 
    \item Set the negative controls as
$$
Z=\nu_z+0.5\cdot u_z\cdot  1_{dim(Z)},\quad W=\nu_w+0.5\cdot u_w \cdot  1_{dim(W)}.
$$
    \item Draw covariates $X\sim\mathcal{N}(0,\Sigma)$.
    \item Then set treatment as
$$
D~\sim Bernoulli(\Lambda(X^{\top}\beta_x+Z^{\top}\beta_z+u_w)).
$$
\item Finally set the outcome as
$$
Y=2.2+ 1.2(X^{\top}\beta_x+W^{\top}\beta_w)+DX_1+0.5\cdot u_z.$$
\end{enumerate}
By construction, $\theta_0^{ATE}(1)-\theta_0^{ATE}(0)=2.2-0=2.2$. 
Table~\ref{tab:discrete} summarizes results. The proposed method (\texttt{N.C.}) outperforms the previous approach (\texttt{T.E.}) that ignores unobserved confounding, when the sample size is sufficiently large. Importantly, the sample size $n$ must be at least 1000 for the estimator to detect and correct for the unobserved confounding. Substantial bias remains, suggesting that the debiased semiparametric estimators subsequently proposed by \cite{kallus2021causal,ghassami2021minimax} may be more appropriate when treatment is discrete.

\begin{table}[ht]
  \centering
    \subfloat[Treatment effect (\texttt{T.E.})]{
  \begin{tabular}{cccc}
       \toprule
    Sample size & Mean & S.D. & M.S.E. \\
    \midrule
     100 &  2.61   &      0.23    &             0.05\\
    500 &  2.59   &      0.11    &             0.01\\
    1000 & 2.55  &       0.06    &            0.00 \\
    5000 & 2.42   &      0.03   &              0.00\\
    \bottomrule
    \end{tabular}
  }\quad
  \subfloat[Negative control (\texttt{N.C.})]{\begin{tabular}{cccc}
       \toprule
  Sample size & Mean & S.D. & M.S.E. \\
    \midrule
     100 &  3.07   &     0.27   &             0.07\\
500 & 2.62    &    0.11   &              0.01\\
    1000 & 2.42& 0.09 &   0.01 \\
    5000 & 1.99& 0.05 &  0.00\\
    \bottomrule
    \end{tabular}}
   
  \caption{Discrete design}%
  \label{tab:discrete}%
\end{table}
\section{Application details}\label{section:application_details}

\subsection{Dose response of cigarette smoking}

As described in the main text, I estimate the dose response curves for the subpopulations of white, black, and Hispanic mothers who smoke. Formally, I estimate $\theta_0^{CATE}(d,v)$ where $D\in\mathbb{R}$ is the number of cigarettes smoked per day, and $V\in\mathbb{R}^2$ concatenates mother's race $V_1$ and mother's smoking status $V_2$. Observe that race is, for our purposes, a discrete variable with three values while smoking status is a binary variable. Consider the subpopulation of white mothers who smoke, i.e. $v=(\text{white},1)$. I implement the product of indicator kernels 
$$
k_{\mathcal{V}}(v,v')=k_{\mathcal{V}_1}(v_1,v_1')k_{\mathcal{V}_2}(v_2,v_2')=\mathbbm{1}\{v_1=v_1'\}\mathbbm{1}\{v_2=v_2'\}.
$$
Therefore the estimator of the confounding bridge $h_0(d,v,x,w)$ at the value $v=(\text{white},1)$ simply zeroes out observations with $V \neq (\text{white},1)$. Since the covariate of interest is discrete rather than continuous, a natural choice of the estimator $\hat{\mu}(v)$ that encodes $\mathbb{P}(x,w|v)$ simply takes the average for the corresponding subpopulation, e.g.
$$
\hat{\mu}(\text{white},1)=\frac{1}{|\{i:V_i=(\text{white},1)\}|}\sum_{i:V_i=(\text{white},1)}\phi(x_i)\otimes \phi(w_i).
$$
Observe that this estimator also zeroes out observations with $V\neq (\text{white},1)$.

In summary, when $V$ is discrete, the described estimator of $\theta_0^{CATE}(d,v)$ is equivalent to the following procedure: (i) subset to the observations such that $V=v$, then (ii) implement $\hat{\theta}^{ATE}(d)$. I implement the estimator $\hat{\theta}^{ATE}(d)$ (\verb|N.C.|) described in Section~\ref{section:algorithm}, with the tuning procedure described in Appendix~\ref{section:tuning}, on the appropriate subset of observations. I implement the continuous treatment effect estimator of \cite{singh2020kernel} (\verb|T.E.|) using the same principles. 

To speed up computation, I incorporate the \verb|kernlab| package in \verb|R|. Due to the large sample size for nonhispanic white women $(n=73,834)$, I split the observations according to the year (1989, 1990, 1991) then take the average.

\subsection{Variable classification}

Table~\ref{table:vars} summarizes the variable classification. The reason why this choice of negative controls $(Z,W)$ is appropriate when $U$ is income, but inappropriate when $U$ is stress, is the relevance condition detailed in Proposition~\ref{prop:relevance}. Based on clinical expertise and previously published results such as \cite{hobel2008psychosocial}, we posit that parental education and family size are relevant for income in the sense that variation in income $U$ can be recovered from variation in parental education $Z$ and family size $W$. The same cannot be said for stress; it is less justified by the existing literature that variation in stress $U$ can be recovered from variation in parental education $Z$ and family size $W$. Figure~\ref{dag:nc_smoking2} visualizes a representative DAG.

\begin{table}[H]
\centering
\caption{Variable classification}
\label{table:vars}
\begin{tabular}{lll}
  \toprule
Symbol & Definition & Empirical application \\
\midrule
 $Y$ & outcome & infant birth weight (grams) \\
\midrule
     $D$ & treatment & \# cigarettes smoked per day during pregnancy\\
\midrule
    $U$ & unobserved & income \\
    & confounding & stress \\
\midrule
    $Z$ & n.c. treatment & mother's educational attainment (years)\\
    &  & father's educational attainment (years)\\
\midrule
    $W$ & n.c. outcome & infant birth order  \\
    &  & infant sex  \\
    &  & Rh sensitization \\
\midrule
   $X$& covariates & mother's demographics: age, marriage status, foreign born \\
   && father's demographics: age, race \\
   && alcohol consumption during pregnancy \\
   && \# prenatal visits during pregnancy \\ 
    && trimester of first prenatal care visit \\ 
   && hypertension: chronic, gestational \\
   && other medical: diabetes, herpes, eclampsia, weight gain \\ 
   && county \\ 
   && year \\ 
\midrule
   $V$& covariate of interest & mother's race \\
   && mother's smoking status \\
 \bottomrule
\end{tabular}
\end{table}

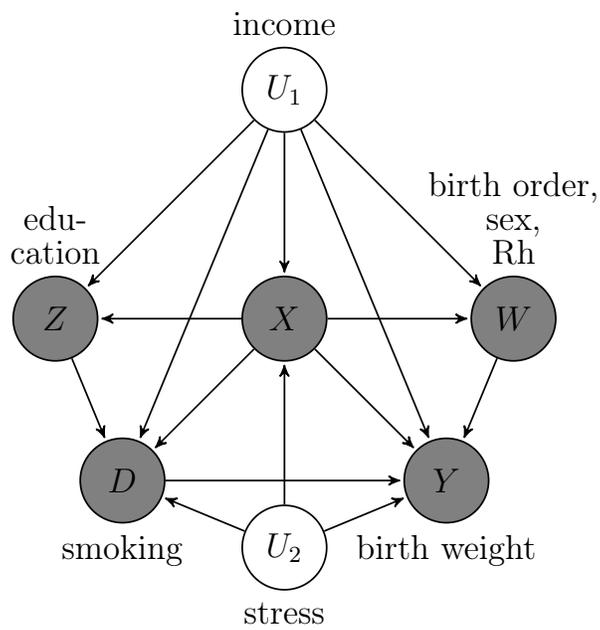
\begin{figure}[H]
    \centering
    \begin{adjustbox}{width=.5\textwidth}
\begin{tikzpicture}[->,>=stealth',shorten >=1pt,auto,node distance=2.8cm,
                    semithick]
  \tikzstyle{every state}=[draw=black,text=black]

  \node[state]         (x) [fill=gray]                   {$X$};
  \node[state]         (w) [right of=x, fill=gray, label={[align=center]above: birth order,\\[-10pt] sex,\\[-10pt] Rh}]       {$W$};
  \node[state]         (z) [left of=x, fill=gray, label={[align=center]above:edu-\\[-10pt]cation}]       {$Z$};
   \node[state]         (d) [below left of=x, fill=gray, label={below:smoking}]       {$D$};
    \node[state]         (y) [below right of=x, fill=gray, label={below:birth weight}]       {$Y$};
   \node[state]         (u1) [above of=x, label={income}]                  {$U_1$};
    \node[state]         (u2) [below of=x, label={below:stress}]                  {$U_2$};

  \path (u1) edge              node {$ $} (z)
             edge           node {$ $} (x)
             edge           node {$ $} (w)
              edge           node {$ $} (d)
             edge           node {$ $} (y)
        (u2) edge              node {$ $} (x)
             edge           node {$ $} (d)
             edge           node {$ $} (y)
        (x) edge              node {$ $} (w)
            edge            node {$ $} (z)
             edge           node {$ $} (d)
             edge           node {$ $} (y)
             (z) edge              node {$ $} (d)
            (w) edge              node {$ $} (y)
            (d) edge              node {$ $} (y);;
\end{tikzpicture}
\end{adjustbox}
\caption{Smoking DAG with stress}
\label{dag:nc_smoking2}
\end{figure}

\end{document}